\documentclass[10pt,twocolumn,letterpaper]{article}

\usepackage[pagenumbers]{cvpr}              %

\usepackage[dvipsnames]{xcolor}

\usepackage{pifont}%
\newcommand{\cmark}{{\color{ForestGreen} \ding{51}}}%
\newcommand{\xmark}{{\color{red} \ding{55}}}%

\newcommand{\rom}[1]{\uppercase\expandafter{\romannumeral #1\relax}}

\usepackage{xcolor}

\usepackage{calc}

\usepackage{amsthm}
\usepackage{amssymb}

\usepackage[pagebackref,breaklinks,colorlinks]{hyperref}
\usepackage{colortbl}
\usepackage{tabulary}
\usepackage{tabularx}
\usepackage{etoolbox}
\usepackage{multirow}
\usepackage{cuted}
\usepackage[percentage]{overpic}
\usepackage{booktabs}
\usepackage{pgfplots}
\pgfplotsset{compat=newest}%

\def\paperID{54} %
\def\confName{3DV\xspace}
\def\confYear{2024\xspace}

\title{Revisiting Map Relations for Unsupervised Non-Rigid Shape Matching}

\author{Dongliang Cao $\qquad$ Paul Roetzer $\qquad$ Florian Bernard\\
University of Bonn
}

\newtheorem{theorem}{Theorem}[section]

\newtheorem{lemma}[theorem]{Lemma}

\begin{document}
\maketitle

\setkeys{Gin}{keepaspectratio}

\def\pathOurs{figs/ours/}
\def\pathDiscOp{figs/dscrtopt/}
\def\pathAttFMap{figs/afmaps/}
\def\pathAttFMapFast{figs/afmaps_fst/}
\def\pathGeoFMap{figs/gmaps/}
\def\pathDpfm{figs/dpfm/}
\def\pathDpfmUn{figs/dpfm_un/}
\def\pathURSMM{figs/urssm/}
\def\srcEnd{_M}
\def\trgtEnd{_N}

\begin{strip}
  \centerline{
  \footnotesize
  \begin{tabular}{c}
    \setlength{\tabcolsep}{0pt}
    \includegraphics[width=\textwidth]{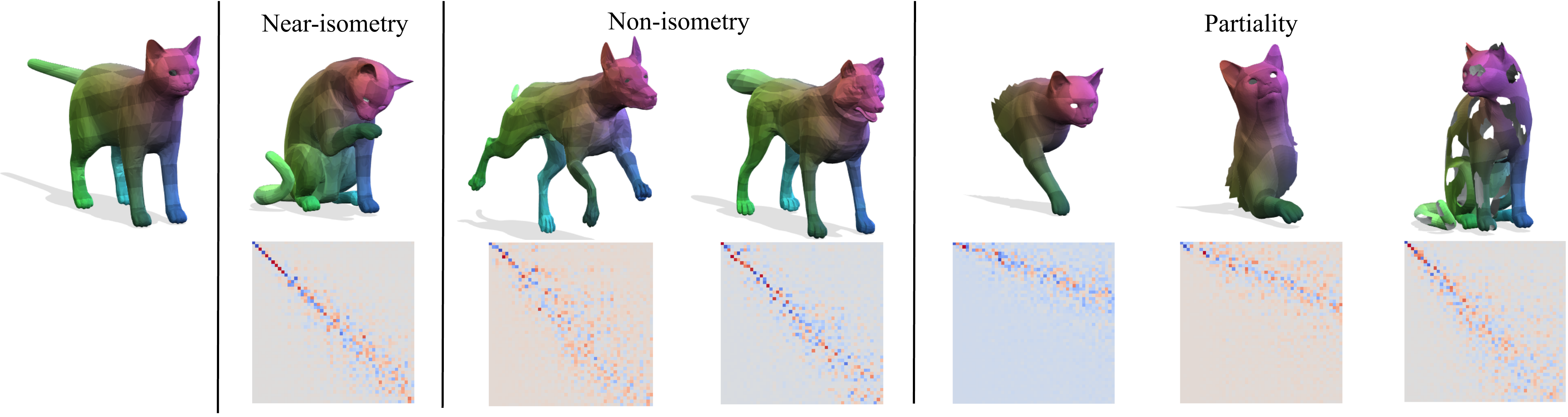}
  \end{tabular}
  }
\captionof{figure}{\textbf{Qualitative 3D shape matching results of our method.} 
The leftmost reference shape is matched to the other shapes in the first row. The second row visualises the corresponding functional maps. We observe that the diagonal structure of the functional maps changes significantly, especially under non-isometry or partiality. To better account for the varying structure of the functional map, we propose a self-adaptive functional map solver that can adjust the functional map regularisation strength and structure depending on the input shapes.
}
\label{fig:teaser}
\end{strip}

\begin{abstract}
We propose a novel unsupervised learning approach for non-rigid 3D shape matching. Our approach improves upon recent state-of-the art deep functional map methods and can be applied to a broad range of different challenging scenarios. Previous deep functional map methods mainly focus on feature extraction and aim exclusively at obtaining more expressive features for functional map computation. However, the importance of the functional map computation itself is often neglected and the relationship between the functional map and point-wise map is underexplored. In this paper, we systematically investigate the coupling relationship between the functional map from the functional map solver and the point-wise map based on feature similarity. To this end, we propose a self-adaptive functional map solver to adjust the functional map regularisation for different shape matching scenarios, together with a vertex-wise contrastive loss to obtain more discriminative features. Using different challenging datasets (including non-isometry, topological noise and partiality), we demonstrate that our method substantially outperforms previous state-of-the-art methods.  
\end{abstract}    
\section{Introduction}
\label{sec:intro}

3D shape matching is a fundamental problem in shape analysis, computer vision and computer graphics with a broad range of applications, including texture transfer~\cite{dinh2005texture}, deformation transfer~\cite{sumner2004deformation} and statistical shape analysis~\cite{loper2015smpl,li2017flame,egger20203d}. Even though 3D shape matching is a long-standing problem and has been studied for decades~\cite{van2011survey,tam2012registration}, finding correspondences between two non-rigidly deformed 3D shapes is still a challenging problem, especially for shapes with large non-isometric deformation, topological noise, or partiality. 

Notably, in the case of 3D shapes represented by triangle meshes, the functional map framework~\cite{ovsjanikov2012functional} is one of the most dominant pipelines in this area and has been extended by many follow-up works due to its efficiency and well-justified theoretical properties~\cite{nogneng2017informative,ren2018continuous,rodola2017partial,donati2022complex}. 

Meanwhile, with the recent rapid development in deep learning, many learning-based methods for non-rigid 3D shape matching are also based on the functional map framework, including both supervised~\cite{litany2017deep,donati2020deep,attaiki2021dpfm} and unsupervised~\cite{halimi2019unsupervised,roufosse2019unsupervised,sharma2020weakly,eisenberger2020deep,cao2022unsupervised,donati2022deep,li2022learning,cao2023unsupervised} approaches. Most of them mainly focus on training the feature extraction module to obtain functional maps based on the extracted features and then rely on off-the-shelf post-processing~\cite{melzi2019zoomout} to obtain final point-wise correspondences. In contrast, the recent work by~\citet{cao2023unsupervised} explicitly models the relationship between functional maps and pointwise maps and thus leads to more robust matching in a broad range of challenging scenarios. However, the method only focuses on extracting more expressive features and ignores the importance of the functional map computation itself. Further, it lacks a discussion about insights between the relationship between the functional map and point-wise map.   

In this paper, we improve upon the recent work by~\citet{cao2023unsupervised} by proposing a novel functional map solver that is self-adaptive to different shape matching scenarios. Moreover, we systematically analyse the relationship between the functional map and the point-wise map and introduce a vertex-wise contrastive loss to obtain more discriminative features leading to  more accurate correspondences. 
We summarise our main contributions as follows:
\begin{itemize}
    \item For the first time we propose a  functional map solver that is self-adaptive for different challenging matching scenarios.
    \item We introduce a vertex-wise contrastive loss to obtain more discriminative features that can be used directly for matching via nearest neighbour search.
    \item We set the new state-of-the-art performance on numerous challenging benchmarks in diverse settings, including non-isometric, topologically noisy and partial shape matching, even compared to recent supervised methods.
\end{itemize}

\section{Related work}
\label{sec:related_work}
3D shape matching is a long-standing problem that has been studied for decades. In the following we focus on reviewing those methods that are
most relevant to our work. A more comprehensive overview can be found in~\cite{tam2012registration,van2011survey,sahilliouglu2020recent}.

\subsection{Axiomatic shape matching methods}
Shape matching can be formulated as establishing point-wise correspondences between a given pair of shapes. A simple formulation for doing so is the linear assignment
problem (LAP)~\cite{munkres1957algorithms}. However, the LAP cannot take  geometric relations into account and thus leads to spatially non-smooth matchings. To compensate for this, several shape matching approaches~\cite{windheuser2011geometrically,holzschuh2020simulated,roetzer2022scalable} establish correspondences by explicitly incorporating geometric constraints. Some methods~\cite{huang2008non,ezuz2019elastic,eisenberger2019divergence,bernard2020mina} attempt to solve the problem based on non-rigid shape registration. Overall, directly establishing point-wise correspondences often leads to complex optimisation problems that are difficult to solve.

In contrast, the functional map framework finds correspondences in the functional domain~\cite{ovsjanikov2012functional}. Here, the correspondence relationship can be encoded with a small matrix, namely the functional map. Due to its simple yet efficient formulation, the functional map framework has been extended by many follow-up works, e.g.\ in terms of improving the matching accuracy and robustness~\cite{eynard2016coupled,ren2019structured}, extending it to more challenging scenarios (e.g.\ non-isometry~\cite{kovnatsky2013coupled,ren2018continuous,eisenberger2020smooth,ren2021discrete,magnet2022smooth}, partiality~\cite{rodola2017partial,litany2017fully}), considering multi-shape matching~\cite{huang2014functional,cohen2020robust,huang2020consistent,gao2021isometric}, and matching with non-unique solutions~\cite{ren2020maptree}. Nevertheless,  axiomatic functional map methods rely on handcrafted features~(e.g.\ HKS~\cite{bronstein2010scale}, WKS~\cite{aubry2011wave}, SHOT~\cite{salti2014shot}), which limits their performance. In contrast, our method (among others) directly learns discriminative features from training data and achieves more accurate and robust matching performance on challenging settings.

\subsection{Deep functional map methods}
In contrast to axiomatic approaches, deep functional map methods aim to learn features directly from training data. The supervised FMNet~\cite{litany2017deep} is the pioneer work that learns a non-linear transformation of SHOT feature~\cite{salti2014shot} based on a point-wise MLP. Later works~\cite{halimi2019unsupervised,roufosse2019unsupervised} enable unsupervised training of FMNet by introducing isometry regularisation in the spatial and spectral domain, respectively. Instead of using simple point-wise MLPs, follow-up works~\cite{donati2020deep,sharma2020weakly} replace FMNet by point-based networks~\cite{qi2017pointnet++,thomas2019kpconv} and lead to better matching performance. More recently,~\citet{sharp2020diffusionnet} introduces DiffusionNet with a learnable diffusion process and has set the new state-of-the-art matching performance for a broad range of shape matching scenarios, including near-isometry~\cite{cao2022unsupervised,attaiki2023understanding}, non-isometry~\cite{donati2022deep,li2022learning,attaiki2023ncp}, partiality~\cite{attaiki2021dpfm,cao2023unsupervised}, as well as shapes represented as point clouds~\cite{cao2023self}. Despite the rapid progress of deep functional map methods, existing approaches mostly focus on learning more expressive features for functional map computation, while ignoring the importance of the functional map computation itself. In this work, we systematically investigate the functional map computation process and introduce a self-adaptive functional map solver to better regularise the functional map structure for different kinds of input shapes.
\section{Background}
\label{sec:background}
In this section we explain the background and introduce the notation used throughout the rest of the paper in~\cref{table:notation}.

\begin{table}[!ht]
\small\centering
    \begin{tabularx}{\columnwidth}{lp{6.2cm}}
        \toprule
        \textbf{Symbol} &\textbf{Description} \\
        \toprule
        $\mathcal{X}, \mathcal{Y}$ & 3D shapes with $n_{\mathcal{X}}$ , $n_{\mathcal{Y}}$ vertices\\
        $L_{\mathcal{X}}$ & $\mathbb{R}^{n_{\mathcal{X}}\times n_{\mathcal{X}}}$ Laplacian matrix of shape $\mathcal{X}$\\
        $\Lambda_{\mathcal{X}}$ & $\mathbb{R}^{k\times k}$ eigenvalue matrix of Laplacian $L_{\mathcal{X}}$\\
        $\Phi_{\mathcal{X}}$ & $\mathbb{R}^{n_{\mathcal{X}}\times k}$  LBO eigenfunctions of shape $\mathcal{X}$\\
        $\Phi_{\mathcal{X}}^{\dagger}$ & $\mathbb{R}^{k\times n_{\mathcal{X}}}$  Moore-Penrose inverse of $\Phi_{\mathcal{X}}$\\
        $F_{\mathcal{X}}$ & $\mathbb{R}^{n_{\mathcal{X}}\times c}$ vertex-wise features of shape $\mathcal{X}$\\
        $A_{\mathcal{X}}$ & $\Phi_{\mathcal{X}}^{\dagger}F_{\mathcal{X}}$ projected feature coefficients of shape $\mathcal{X}$\\
        $C_{\mathcal{XY}}$ & $\mathbb{R}^{k\times k}$ functional map between shapes $\mathcal{X}$ and $\mathcal{Y}$\\
        $\Pi_{\mathcal{YX}}$ & point-wise map between shapes $\mathcal{Y}$ and $\mathcal{X}$\\
        \bottomrule
    \end{tabularx}
    \caption{Summary of the notation used in this paper.}
    \vspace{-0.5cm}
    \label{table:notation}
\end{table}

\subsection{Functional map framework}
\label{subsec:fmap_pipeline}
We consider a pair of 3D shapes $\mathcal{X}$ and $\mathcal{Y}$ represented as triangle meshes, with $n_{\mathcal{X}}$ and $n_{\mathcal{Y}}$ vertices, respectively. Here we summarise the common pipeline of the functional map framework.
\begin{enumerate}
    \item Compute the associated positive semi-definite Laplacian matrices $L_{\mathcal{X}}, L_{\mathcal{Y}}$~\cite{pinkall1993computing}. The Laplacian matrix can be computed as $L_{\mathcal{X}}= M_{\mathcal{X}}^{-1}W_{\mathcal{X}}$, where $M_{\mathcal{X}}$ is the diagonal lumped mass matrix and $W_{\mathcal{X}}$ is the cotangent weight matrix. 
    \item Compute the first $k$ eigenfunctions $\Phi_{\mathcal{X}}, \Phi_{\mathcal{Y}}$ and the corresponding eigenvalues $\Lambda_{\mathcal{X}}, \Lambda_{\mathcal{Y}}$ of the respective Laplacian matrices (i.e.\ LBO eigenfunctions/eigenvalues).
    \item Compute $c$-dimensional features $F_{\mathcal{X}}, F_{\mathcal{Y}}$ defined on each shape either from handcrafted feature descriptors or from a learnable feature extractor.
    \item Compute the functional map $C_{\mathcal{XY}}$ associated with the LBO eigenfunctions by solving (variants of) the  least squares problem
    \begin{equation}
    \label{eq:fmap}
        C_{\mathcal{XY}}=\mathrm{argmin}_{C}~ E_{\mathrm{data}}\left(C\right)+\lambda E_{\mathrm{reg}}\left(C\right).
    \end{equation}
    Here, minimising $E_{\mathrm{data}}=\left\|CA_{\mathcal{X}}-A_{\mathcal{Y}}\right\|^{2}_{F}$ enforces descriptor preservation, while minimising the regularisation term $E_{\mathrm{reg}}$ imposes some form of structural properties (e.g.\ Laplacian commutativity~\cite{ovsjanikov2012functional}).
    \item Recover the point-wise map $\Pi_{\mathcal{YX}}$ based on the relationship $C_{\mathcal{XY}} = \Phi_{\mathcal{Y}}^{\dagger}\Pi_{\mathcal{YX}}\Phi_{\mathcal{X}}$, e.g.\ either by nearest neighbour search in the spectral domain or by other post-processing techniques~\cite{vestner2017product,melzi2019zoomout,ezuz2019reversible}.
\end{enumerate} 
We emphasise that most deep functional methods mainly focus on the third step that aims to extract more expressive features for functional map computation while ignoring the importance of the other steps (i.e.\ the functional map computation and point-wise map conversion). However, we argue that this may lead to sub-optimal performance, since the  three interrelated aspects feature learning, functional map computation, and point-wise map conversion are considered in an isolated rather than a joint manner. Therefore, in this paper we systematically investigate the functional map computation step and the relationship between the functional map and the associated point-wise map. 
\subsection{Deep functional maps}
Instead of relying on handcrafted features~\cite{bronstein2010scale,aubry2011wave,salti2014shot} to compute functional maps, many deep functional map methods~\cite{roufosse2019unsupervised,sharma2020weakly} have been proposed. The common pipeline of those methods is shown in Fig.~\ref{fig:3dv_pipeline} (left). 

The common deep functional map framework mainly consists of two modules: a feature extractor and a functional map solver. The feature extractor is used to extract vertex-wise features and the functional map solver is used to compute functional maps based on the extracted features. To train the feature extractor, structural regularisation (e.g.\ orthogonality, bijectivity~\cite{roufosse2019unsupervised}) is imposed on the computed functional maps, i.e.
\begin{equation}
    \label{eq:l_fmap}
    L_{\mathrm{fmap}} = \lambda_{\mathrm{bij}}L_{\mathrm{bij}} + \lambda_{\mathrm{orth}}L_{\mathrm{orth}},
\end{equation}
where
 \begin{equation}
    \label{eq:bij}
    L_{\mathrm{bij}}=\left\|C_{\mathcal{XY}}C_{\mathcal{YX}}-I\right\|^{2}_{F}+\left\|C_{\mathcal{YX}}C_{\mathcal{XY}}-I\right\|^{2}_{F},
\end{equation}
\begin{equation}
    \label{eq:orth}
    L_{\mathrm{orth}}=\left\|C_{\mathcal{XY}}^{\top}C_{\mathcal{YX}}-I\right\|^{2}_{F}+\left\|C_{\mathcal{XY}}^{\top}C_{\mathcal{YX}}-I\right\|^{2}_{F}.
\end{equation} 
After training, off-the-shelf post-processing techniques~\cite{vestner2017product,melzi2019zoomout} are used to convert functional maps to point-wise maps.

As pointed out by recent works~\cite{cao2023unsupervised,ren2021discrete,attaiki2023understanding}, a major downside of this common pipeline is that the relation between the functional maps and associated point-wise maps is ignored, so that the performance is often sub-optimal, especially in the presence of large non-isometry, topological noise or partiality. To compensate for this,~\citet{cao2023unsupervised} proposed to directly obtain point-wise maps based on the extracted features and introduced a coupling loss $L_{\mathrm{couple}}$ to explicitly regularise the relation between the point-wise map $\Pi_{\mathcal{YX}}$ and the corresponding functional map $C_{\mathcal{XY}}$, i.e.
\begin{equation}
    \label{eq:l_couple}
    L_{\mathrm{couple}} = \left\|C_{\mathcal{XY}} - C_{\mathcal{XY}}^{\Pi}\right\|^{2}_{F},
\end{equation}
where $C_{\mathcal{XY}}^{\Pi}=\Phi_{\mathcal{Y}}^{\dagger}\Pi_{\mathcal{YX}}\Phi_{\mathcal{X}}$.

\begin{figure*}[!ht]
    \begin{center}  
    \begin{tabular}{c|c}
    \includegraphics[width=\columnwidth]{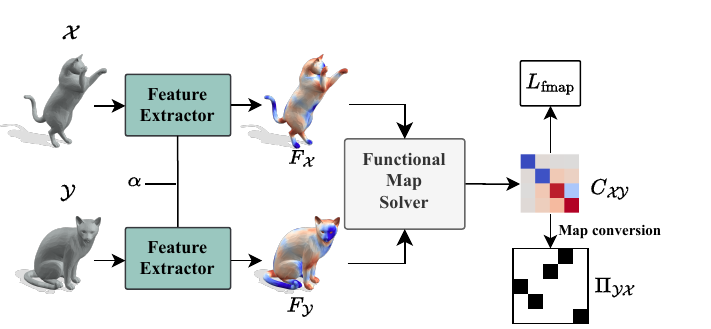}  &  \includegraphics[width=\columnwidth]{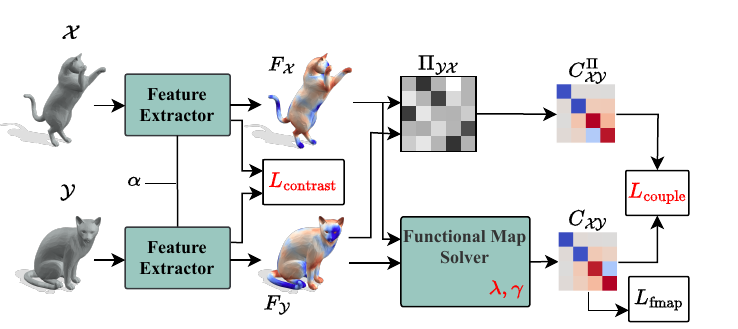}\\
    \footnotesize{Common deep functional map shape matching pipeline} & \footnotesize{Our proposed shape matching pipeline}\\
    \end{tabular}
    \caption{\textbf{Left: Common pipeline of deep functional map methods}. A Siamese feature extractor computes vertex-wise features for each shape. The extracted features are used for functional map computation. During training, structural regularisation $L_{\mathrm{fmap}}$ is imposed on the functional maps. During inference, the computed functional maps are typically converted to  point-wise maps via map conversion. \\
    \textbf{Right: Our proposed shape matching pipeline.} Point-wise correspondences are obtained based on feature similarity. A coupling loss \textcolor{red}{$L_\mathrm{couple}$} regularises the relation between the point-wise map $\Pi_{\mathcal{YX}}$ and the functional map $C_{\mathcal{XY}}$~\cite{cao2023unsupervised}. {To better balance the functional map regularisation and the coupling relationship between the functional map and the point-wise map}, we introduce a self-adaptive functional map solver (with learnable parameters \textcolor{red}{$\lambda$} and \textcolor{red}{$\gamma$}) to adjust the regularisation strength and structure, respectively. Additionally, a vertex-wise contrastive loss \textcolor{red}{$L_{\mathrm{contrast}}$} is introduced to improve the discriminative power of the features.
    }
    \vspace{-0.6cm}
    \label{fig:3dv_pipeline}
    \end{center}
\end{figure*}

By explicitly modelling the relationship between functional maps and point-wise maps, the method proposed by~\citet{cao2023unsupervised} substantially outperforms existing methods and is robust in different challenging scenarios. However, there are two major limitations:
\begin{itemize}
    \item In their approach, one functional map $C_{\mathcal{XY}}$ is computed from the functional map solver, while the other one $C_{\mathcal{XY}}^{\Pi}$ is converted from the point-wise map based on the feature similarity. Yet, the underlying relationship between them is  not well-understood.
    \item The coupling loss $L_{\mathrm{couple}}$ regularises the functional maps computed from the functional map solver based on the extracted features, while the functional map computation itself is {not optimised in a data-driven manner}. 
\end{itemize}
In the following, we theoretically analyse the relationship between $C_{\mathcal{XY}}$ and $C_{\mathcal{XY}}^{\Pi}$, and revisit the map relations by introducing a self-adaptive functional map solver and a vertex-wise contrastive loss.

\section{Theoretical analysis of map relations}
\label{sec:theory}
In this section, we  analyse the underlying relation between the functional map computed from the functional map solver and the functional map converted by the point-wise map based on the deep feature similarity.

W.l.o.g.\ we assume that $n_{\mathcal{Y}} \leq n_{\mathcal{X}}$. With that, a (partial or full) shape $\mathcal{Y}$ is matched to a full shape $\mathcal{X}$ and thereby the point-wise map $\Pi_{\mathcal{YX}}$ should be a (partial) permutation matrix, i.e.
    \begin{equation}
        \label{eq:permutation_mat}
        \mathcal{P} := \left\{\Pi \in\{0,1\}^{n_{\mathcal{Y}} \times n_{\mathcal{X}}}: \Pi \mathbf{1}_{n_{\mathcal{X}}} = \mathbf{1}_{n_{\mathcal{Y}}}, \mathbf{1}_{n_{\mathcal{Y}}}^{\top} \Pi \leq \mathbf{1}_{n_{\mathcal{X}}}^{\top}\right\},
    \end{equation}
where $\Pi_{\mathcal{YX}}(i,j)$ indicates whether the $i$-th point in $\mathcal{Y}$ corresponds to the $j$-th point in $\mathcal{X}$.

Firstly, we analyse the point-wise map computed based on the feature similarity. We note that
\begin{equation}
    \label{eq:pi}
    \Pi_{\mathcal{YX}} = \mathrm{argmin}_{\Pi \in \mathcal{P}}\left\|\Pi {F}_{\mathcal{X}} - F_{\mathcal{Y}}\right\|_{F}^{2}.
\end{equation}

\begin{lemma}
\label{lemma:pmap}
 If there exists a unique solution to~\cref{eq:pi}, then the rows of ${F}_{\mathcal{X}}$ and ${F}_{\mathcal{Y}}$ have non-repeated rows.
\end{lemma}

\begin{proof}
 If $F_{\mathcal{X}}$ has repeated rows, we can find  a (full) permutation matrix $\Pi_{\mathcal{XX}} \neq I$ that satisfies $\Pi_{\mathcal{XX}}F_{\mathcal{X}}=F_{\mathcal{X}}$. Therefore, any solution $\Pi_{\mathcal{YX}}$ has an equivalent solution $\Pi_{\mathcal{YX}}' := \Pi_{\mathcal{YX}}\Pi_{\mathcal{XX}}$ and is thus not unique. Due to the orthogonal invariance of the Frobenius norm, an analogous statement can be made for $F_{\mathcal{Y}}$.
\end{proof}

\noindent \textbf{Discussion.} {To obtain a valid point-wise map based on the feature similarity, the features $F_{\mathcal{X}}, F_{\mathcal{Y}}$ should have non-repeated rows. To this end, based on Lemma~\ref{lemma:pmap} we propose a vertex-wise contrastive loss to encourage more discriminative features.} 

\begin{theorem}
\label{theorem:fmap}
Consider the following conditions:
\begin{enumerate}[label=(\roman*)]
    \item $\Pi_{\mathcal{YX}}F_{\mathcal{X}}=F_{\mathcal{Y}}, \Pi_{\mathcal{YX}} \in \mathcal{P}$, where {$n_{\mathcal{Y}} \leq n_{\mathcal{X}}$}. 
    \item {$F_{\mathcal{X}}$ is in the span of $\Phi_{\mathcal{X}}$ and $ F_{\mathcal{Y}}$ is in the span of $\Phi_{\mathcal{Y}}$.}
    \item $\lambda=0$ in~\cref{eq:fmap} and $A_{\mathcal{X}}$ {$\in 
    \mathbb{R}^{k \times c}$} ($k \leq c$) is full rank.
\end{enumerate}
If conditions (i)-(iii) hold, then we have $C_{\mathcal{XY}} = C_{\mathcal{XY}}^{\Pi}$, and $\left\|C_{\mathcal{XY}}A_{\mathcal{X}}-A_{\mathcal{Y}}\right\|^{2}_{F}=0$.
\end{theorem}

\begin{proof}
By condition (i), we have $\Pi_{\mathcal{YX}}F_{\mathcal{X}}=F_{\mathcal{Y}}$ and {from condition (ii) we know that $F_{\mathcal{X}} = \Phi_{\mathcal{X}} A_{\mathcal{X}}$ (since $A_{\mathcal{X}}$ is the matrix of projected feature coefficients). The same holds for $\mathcal{Y}$}. Putting these together,
\begin{equation}
\label{eq:theorem3.2.1}
\Pi_{\mathcal{YX}}\Phi_{\mathcal{X}}A_{\mathcal{X}} = \Phi_{\mathcal{Y}}A_{\mathcal{Y}}.
\end{equation}
Pre-multiplying~\cref{eq:theorem3.2.1} by $\Phi_{\mathcal{Y}}^{\dagger}$  we obtain
\begin{equation}
\Phi_{\mathcal{Y}}^{\dagger}\Pi_{\mathcal{YX}}\Phi_{\mathcal{X}}A_{\mathcal{X}} = C_{\mathcal{XY}}^{\Pi}A_{\mathcal{X}} =  A_{\mathcal{Y}},
\end{equation}
where the definition $C_{\mathcal{XY}}^{\Pi}=\Phi_{\mathcal{Y}}^{\dagger}\Pi_{\mathcal{YX}}\Phi_{\mathcal{X}}$ is used.
Thus $\left\|C_{\mathcal{XY}}^{\Pi}A_{\mathcal{X}}-A_{\mathcal{Y}}\right\|^{2}_{F}=0$ and $C_{\mathcal{XY}} = C_{\mathcal{XY}}^{\Pi}$ 
($C_{\mathcal{XY}}=\mathrm{argmin}_{C}\left\|CA_{\mathcal{X}}-A_{\mathcal{Y}}\right\|^{2}_{F}$  achieves 0, and $A_{\mathcal{X}}$ is full rank so that the solution is unique, implying  $C_{\mathcal{XY}} = C_{\mathcal{XY}}^{\Pi}$).    
\end{proof}

\noindent \textbf{Discussion.} \cref{theorem:fmap} builds a connection between the functional map $C_{\mathcal{XY}}$ computed from the functional map solver, i.e.\ \cref{eq:fmap}, and the functional map $C_{\mathcal{XY}}^{\Pi}$ converted from the point-wise map $\Pi_{\mathcal{YX}}$. It explicitly shows that $C_{\mathcal{XY}}$ and $C_{\mathcal{XY}}^{\Pi}$ are equal under certain conditions. However, in practical situations, the assumptions are too restrictive and often \textbf{not} satisfied. For example, when computing functional maps using~\cref{eq:fmap}, structural regularisation $E_{\mathrm{reg}}$ is typically needed to preserve the structure of the functional map (e.g.\ Laplacian commutativity for isometry). Furthermore, we often do not want to constrain the feature $F_{\mathcal{X}}$ to lie in the span of the corresponding LBO eigenfunctions $\Phi_{\mathcal{X}}$, which limits its discriminative power and expressiveness, since the first $k$ LBO eigenfunctions correspond to the $k$ smoothest orthonormal functions defined on the surface w.r.t.\ the Dirichlet energy~\cite{bobenko2007discrete}. 
Even though the conditions of the theoretical results are not strictly met, the results give insights about the relations between variables, which we transfer into soft constraints that approximate the conditions.
For instance, we note that the functional map solver plays a crucial role to balance the functional map regularisation and the coupling relation between the functional map and the point-wise map. On the one hand, the regularisation term $E_{\mathrm{reg}}$ in~\cref{eq:fmap} preserves the functional map structure. On the other hand, it may result an invalid functional map (i.e.\ a functional map without an associated point-wise map). Therefore, it is important to adjust the functional map regularisation in a data-driven manner.

\section{Revisiting the map relations}
\label{sec:method}
In the previous section, we theoretically analyse the relationship between the functional map computed from the functional map solver and the functional map converted from the point-wise map based on deep feature similarity. Motivated by our analysis, we propose two simple yet efficient extensions from the existing framework, which we introduce in the following. We highlight the different parts in~\cref{fig:3dv_pipeline} (right) with \textcolor{red}{red colour}, compared to the common deep functional map pipeline shown in~\cref{fig:3dv_pipeline} (left).

\subsection{Self-adaptive functional map solver}
As discussed in~\cref{theorem:fmap}., we only consider $E_{\mathrm{data}}$ and ignore $E_{\mathrm{reg}}$ in~\cref{eq:fmap} with some additional assumptions (i.e.\ $F_{\mathcal{X}}, F_{\mathcal{Y}}$ in the span of $\Phi_{\mathcal{X}}, \Phi_{\mathcal{Y}}$, and $A_{\mathcal{X}}$ is full rank). With that, the functional map computed by the functional map solver $C_{\mathcal{XY}}$ is equal to the functional map $C_{\mathcal{XY}}^{\Pi}$ converted by the point-wise map. However, as shown in~\cref{fig:teaser}, the valid functional maps often exhibit certain structures that needs to be imposed by the regularisation term $E_{\mathrm{reg}}$. To this end, we propose a self-adaptive functional map solver that can optimise the regularisation strength and the regularisation structure based on the training data. Specifically, we use the regularisation term proposed by~\citet{ren2019structured}, which is an extension of the standard Laplacian commutativity. The standard Laplacian commutativity can be formulated as 
\begin{equation}
    \label{eq:lap}
    \begin{split}
        E_{\mathrm{lap}} &= \left\|C_{\mathcal{XY}}\Lambda_{\mathcal{X}}-\Lambda_{\mathcal{Y}}C_{\mathcal{XY}}\right\|_{F}^{2} \\
        &= \sum_{ij}(\Lambda_{\mathcal{Y}}(i,i) - \Lambda_{\mathcal{X}}(j,j))^{2}\left[C_{\mathcal{XY}}\right]_{ij}^{2} \\
        &= \sum_{ij}\left[M_{\mathrm{lap}}\right]_{ij}\left[C_{\mathcal{XY}}\right]_{ij}^{2}. 
    \end{split}
\end{equation}
\citet{ren2019structured} extended the standard mask $M_{\mathrm{lap}}$ to a resolvent mask $M_{\mathrm{res}}^{\gamma}$ in the form
\begin{equation}
    \left[M_{\mathrm{res}}^{\gamma}\right]_{ij} = \left[M_{\mathrm{re}}^{\gamma}\right]_{ij} + \left[M_{\mathrm{im}}^{\gamma}\right]_{ij},
\end{equation}
where 
\begin{equation}
    \label{eq:re}
    \left[M_{\mathrm{re}}^{\gamma}\right]_{ij} = \left(\frac{\Lambda_{\mathcal{Y}}^{\gamma}(i,i)}{\Lambda_{\mathcal{Y}}^{2\gamma}(i,i) + 1} - \frac{\Lambda_{\mathcal{X}}^{\gamma}(j,j)}{\Lambda_{\mathcal{X}}^{2\gamma}(j,j) + 1}\right)^{2},
\end{equation}
\begin{equation}
    \label{eq:img}
    \left[M_{\mathrm{im}}^{\gamma}\right]_{ij} = \left(\frac{1}{\Lambda_{\mathcal{Y}}^{2\gamma}(i,i) + 1} - \frac{1}{\Lambda_{\mathcal{X}}^{2\gamma}(j,j) + 1}\right)^{2}.
\end{equation}
The parameter $\gamma$ in the resolvent mask controls the regularisation structure of the functional map as shown in~\cref{fig:resolvent_mask}. In general, the $\gamma$ is chosen in the range $(0, 1]$ to keep the funnel-structure regularisation to be similar to the ground-truth functional map, i.e.\ more diagonal-dominant entries for smaller eigenvalues. Additionally, a larger $\gamma$ imposes larger penalisation on the non-zero off-diagonal entries and a smaller $\gamma$ provides more flexibility of the functional map~(see~\cref{fig:resolvent_mask}).

\begin{figure}[!ht]
    \begin{center}  \includegraphics[width=\columnwidth]{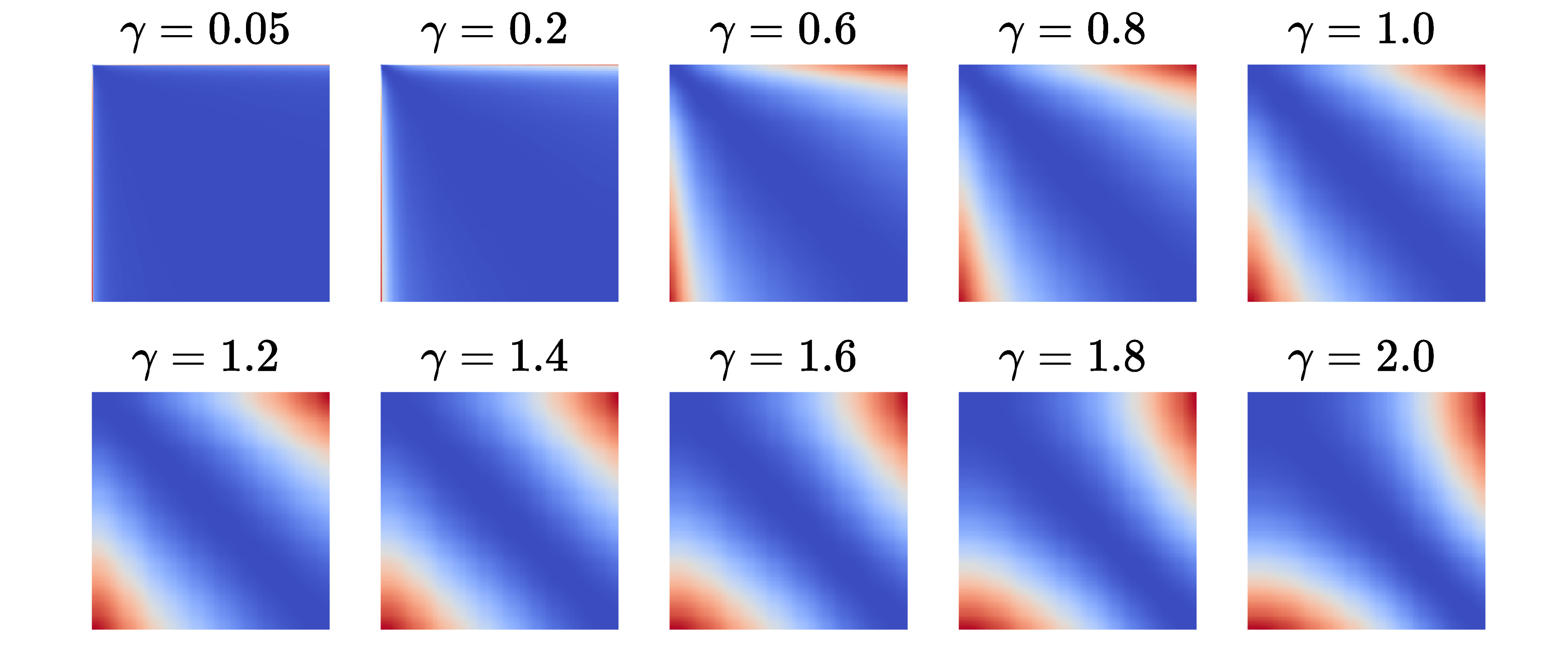}
    \caption{\textbf{The resolvent mask $M_{\mathrm{res}}^{\gamma}$ for different $\gamma$.} The red region indicates large penalty, while the blue region indicates small penalty. We notice the funnel-like structure changes w.r.t.\ the change of $\gamma$ and it reverses the direction for $\gamma > 1$. 
    }
    \vspace{-0.5cm}
    \label{fig:resolvent_mask}
    \end{center}
\end{figure}

Instead of manually choosing the regularisation strength (i.e.\ $\lambda$ in~\cref{eq:fmap}) and structure (i.e.\ $\gamma$ in~\cref{eq:re},~\cref{eq:img}), we propose to learn these parameters during training. To this end, the functional map solver is optimised from the input shapes to find a better balance between the data term $E_{\mathrm{data}}$ and regularisation term $E_{\mathrm{reg}}$ and thus to better couple $C_{\mathcal{XY}}$ and $C_{\mathcal{XY}}^{\Pi}$. In the experiment, we show the simple modification leads to better matching performance especially for the most challenging scenarios. We also visualise the regularisation for each evaluated datasets in~\cref{fig:fmap_reg}.

\subsection{Vertex-wise contrastive loss}
As discussed in~\cref{lemma:pmap}, a valid point-wise map based on the feature similarity requires $F_{\mathcal{X}}, F_{\mathcal{Y}}$ both have distinct rows. To this end, we propose a vertex-wise contrastive loss to encourage more discriminative features. We first compute a point-wise map $\Pi_{\mathcal{XX}}$ that maps shape $\mathcal{X}$ to itself. To make the computation differentiable, we use the softmax operator to approximate a soft point-wise map, i.e.
\begin{equation}
    \label{eq:soft_corr}
    \Pi_{\mathcal{XX}} = \mathrm{Softmax}\left( {F_{\mathcal{X}}F_{\mathcal{X}}^{T}} / \tau\right),
\end{equation}
where parameter $\tau$ is to determine the softness of the point-wise map. Similar to~\cite{cao2023unsupervised}, the computed point-wise map $\Pi_{\mathcal{XX}}$ is projected to the associated functional map $C_{\mathcal{XX}}$, i.e.
\begin{equation}
    C_{\mathcal{XX}} = \Phi_{\mathcal{X}}^{\dagger}\Pi_{\mathcal{XX}}\Phi_{\mathcal{X}}.
\end{equation}
Our vertex-wise contrastive loss regularises the functional map $C_{\mathcal{XX}}$ to be an identity matrix, i.e.
\begin{equation}
    L_{\mathrm{contrast}} = \left\|C_{\mathcal{XX}} - I\right\|_{F}^{2}.
\end{equation}
Similarly, we also apply the vertex-wise contrastive loss $L_{\mathrm{contrast}}$ for the functional map $C_{\mathcal{YY}}$. There are two main advantages of applying the regularisation on the functional map domain. The first advantage is to make $L_{\mathrm{contrast}}$ comparable to other loss terms (i.e.\ \cref{eq:l_fmap} and~\cref{eq:l_couple}). The second advantage is to make the $L_{\mathrm{contrast}}$ discretisation-agnostic. Overall, the total unsupervised loss can be expressed as
\begin{equation}
    \label{eq:l_total}
    L_{\mathrm{total}} = L_{\mathrm{fmap}} + \lambda_{\mathrm{couple}}L_{\mathrm{couple}} + \lambda_{\mathrm{contrast}}L_{\mathrm{contrast}}.
\end{equation}

\section{Experimental results}
\label{sec:experiment}
In this section we compare our method  to previous methods on diverse benchmark shape matching datasets with different settings (including near-isometric, non-isometric, topological noisy, and partial shape matching). 
\subsection{Near-isometric shape matching}
\noindent \textbf{Datasets.} We evaluate our method on three standard benchmark datasets, namely the FAUST~\cite{bogo2014faust}, SCAPE~\cite{anguelov2005scape} and SHREC'19~\cite{melzi2019shrec} datasets. Following prior works, we choose the more challenging remeshed versions from~\cite{ren2018continuous,donati2020deep}. The FAUST dataset consists of 100 shapes, where the train/test split is 80/20. The SCAPE dataset contains 71 shapes, where the last 20 shapes are used for evaluation. The SHREC'19 dataset is a more challenging dataset with significant variance in the mesh connectivity and shape geometry. It has a total of 430 pairs for evaluation.

\begin{table}[ht!]
\small
\setlength{\tabcolsep}{2pt}
\centering
\small
\begin{tabular}{@{}lccc@{}}
\toprule
\multicolumn{1}{l}{Train}  & \multicolumn{1}{c}{\textbf{FAUST}}   & \multicolumn{1}{c}{\textbf{SCAPE}}  & \multicolumn{1}{c}{\textbf{FAUST + SCAPE}} \\ \cmidrule(lr){2-2} \cmidrule(lr){3-3} \cmidrule(lr){4-4}
\multicolumn{1}{l}{Test} & \multicolumn{1}{c}{\textbf{FAUST}} &  \multicolumn{1}{c}{\textbf{SCAPE}} &  \multicolumn{1}{c}{\textbf{SHREC'19}}
\\ \midrule
\multicolumn{4}{c}{Axiomatic Methods} \\
\multicolumn{1}{l}{BCICP~\cite{ren2018continuous}} & \multicolumn{1}{c}{6.1}  & \multicolumn{1}{c}{11.0} & \multicolumn{1}{c}{-}\\
\multicolumn{1}{l}{ZoomOut~\cite{melzi2019zoomout}} & \multicolumn{1}{c}{6.1} & \multicolumn{1}{c}{7.5} &  \multicolumn{1}{c}{-}\\
\multicolumn{1}{l}{Smooth Shells~\cite{eisenberger2020smooth}} & \multicolumn{1}{c}{2.5}  & \multicolumn{1}{c}{4.7} & \multicolumn{1}{c}{-}\\ 
\multicolumn{1}{l}{DiscreteOp~\cite{ren2021discrete}} & \multicolumn{1}{c}{5.6}  & \multicolumn{1}{c}{13.1} & \multicolumn{1}{c}{-}\\ 
\midrule
\multicolumn{4}{c}{Supervised Methods} \\ 
\multicolumn{1}{l}{FMNet~\cite{litany2017deep}} & \multicolumn{1}{c}{11.0} & \multicolumn{1}{c}{17.0} & \multicolumn{1}{c}{-} \\

\multicolumn{1}{l}{3D-CODED~\cite{groueix20183d}} & \multicolumn{1}{c}{2.5}  & \multicolumn{1}{c}{31.0} & \multicolumn{1}{c}{-} \\
\multicolumn{1}{l}{GeomFMaps~\cite{donati2020deep}}& \multicolumn{1}{c}{2.6} & \multicolumn{1}{c}{3.0} & \multicolumn{1}{c}{7.9}\\

\midrule
\multicolumn{4}{c}{Unsupervised Methods} \\

\multicolumn{1}{l}{WSupFMNet~\cite{sharma2020weakly}} & \multicolumn{1}{c}{3.8}  & \multicolumn{1}{c}{4.4}  & \multicolumn{1}{c}{-} \\

\multicolumn{1}{l}{Deep Shells~\cite{eisenberger2020deep}} & \multicolumn{1}{c}{1.7}  & \multicolumn{1}{c}{2.5} &  \multicolumn{1}{c}{21.1} \\

\multicolumn{1}{l}{DUO-FMNet~\cite{donati2022deep}}  & \multicolumn{1}{c}{2.5}  & \multicolumn{1}{c}{2.6} & \multicolumn{1}{c}{6.4} \\
\multicolumn{1}{l}{AttentiveFMaps~\cite{li2022learning}}  & \multicolumn{1}{c}{1.9}  & \multicolumn{1}{c}{2.2} & \multicolumn{1}{c}{5.8}\\
\multicolumn{1}{l}{AttentiveFMaps-Fast~\cite{li2022learning}}  & \multicolumn{1}{c}{1.9}  & \multicolumn{1}{c}{2.1} &  \multicolumn{1}{c}{6.3}\\
\multicolumn{1}{l}{URSSM~\cite{cao2023unsupervised}}  & \multicolumn{1}{c}{1.6}  & \multicolumn{1}{c}{1.9} &  \multicolumn{1}{c}{4.6}\\
\multicolumn{1}{l}{Ours}  & \multicolumn{1}{c}{\textbf{1.5}}  & \multicolumn{1}{c}{\textbf{1.8}} & \multicolumn{1}{c}{\textbf{3.4}} \\\hline
\end{tabular}
\caption{\textbf{Near-isometric shape matching and cross-dataset generalisation on FAUST, SCAPE and SHREC'19.} The \textbf{best} results in each column are highlighted. Our method outperforms previous axiomatic, supervised and unsupervised methods and demonstrates better cross-dataset generalisation ability.}
\label{tab:near_isometric}
\end{table}

\noindent \textbf{Results.} The mean geodesic error~\cite{kim2011blended} is used as quantitative measure. We compare our method with state-of-the-art axiomatic, supervised and unsupervised methods. The results are summarised in~\cref{tab:near_isometric}. Our
method outperforms the previous state of the art,
even in comparison to the supervised methods. Meanwhile, our method
achieves substantially better cross-dataset generalisation ability
compared to existing learning-based methods.

\subsection{Non-isometric shape matching}
\textbf{Datasets.} In the context of non-isometric shape matching, we consider the SMAL~\cite{zuffi20173d} dataset and the DT4D-H~\cite{magnet2022smooth} dataset. The SMAL dataset contains 49 animal shapes of eight species. Following~\citet{donati2022deep}, five species are used for training and three different species are used for testing (i.e.\ 29/20 shapes for train/test split). The DT4D-H dataset based on DeformingThings4D~\cite{li20214dcomplete} is introduced by~\citet{magnet2022smooth}. Following \citet{li2022learning}, nine classes of humanoid shapes are used for evaluation, resulting in 198/95 shapes for train/test split. 

\begin{table}[ht!]
    \setlength{\tabcolsep}{4pt}
    \small
    \centering
    \begin{tabular}{@{}lccc@{}}
    \toprule
    \multicolumn{1}{l}{\multirow{2}{*}{\textbf{Geo. error ($\times$100)}}}  & \multicolumn{1}{c}{\multirow{2}{*}{\textbf{SMAL}}}   & \multicolumn{2}{c}{\textbf{DT4D-H}}\\ \cmidrule(lr){3-4}
    &  & \multicolumn{1}{c}{\textbf{intra-class}} & \multicolumn{1}{c}{\textbf{inter-class}}
    \\ \midrule
    \multicolumn{4}{c}{Axiomatic Methods} \\
    \multicolumn{1}{l}{ZoomOut~\cite{melzi2019zoomout}}  & \multicolumn{1}{c}{38.4} & \multicolumn{1}{c}{4.0} & \multicolumn{1}{c}{29.0} \\
    \multicolumn{1}{l}{Smooth Shells~\cite{eisenberger2020smooth}}  & \multicolumn{1}{c}{36.1} & \multicolumn{1}{c}{1.1} & \multicolumn{1}{c}{6.3} \\
    \multicolumn{1}{l}{DiscreteOp~\cite{ren2021discrete}}  & \multicolumn{1}{c}{38.1} & \multicolumn{1}{c}{3.6} & \multicolumn{1}{c}{27.6} \\
    \midrule
    \multicolumn{4}{c}{Supervised Methods} \\ 
    \multicolumn{1}{l}{FMNet~\cite{litany2017deep}}  & \multicolumn{1}{c}{42.0} & \multicolumn{1}{c}{9.6} & \multicolumn{1}{c}{38.0} \\
    \multicolumn{1}{l}{GeomFMaps~\cite{donati2020deep}}  & \multicolumn{1}{c}{8.4} & \multicolumn{1}{c}{2.1} & \multicolumn{1}{c}{4.1} \\
    \midrule
    \multicolumn{4}{c}{Unsupervised Methods} \\
    \multicolumn{1}{l}{WSupFMNet~\cite{sharma2020weakly}}  & \multicolumn{1}{c}{7.6} & \multicolumn{1}{c}{3.3} & \multicolumn{1}{c}{22.6} \\
    \multicolumn{1}{l}{Deep Shells~\cite{eisenberger2020deep}}  & \multicolumn{1}{c}{29.3} & \multicolumn{1}{c}{3.4} & \multicolumn{1}{c}{31.1} \\
    \multicolumn{1}{l}{DUO-FMNet~\cite{donati2022deep}}  & \multicolumn{1}{c}{6.7} & \multicolumn{1}{c}{2.6} & \multicolumn{1}{c}{15.8} \\
    \multicolumn{1}{l}{AttentiveFMaps~\cite{li2022learning}}  & \multicolumn{1}{c}{5.4} & \multicolumn{1}{c}{1.7} & \multicolumn{1}{c}{11.6} \\
    \multicolumn{1}{l}{AttentiveFMaps-Fast~\cite{li2022learning}}  & \multicolumn{1}{c}{5.8} & \multicolumn{1}{c}{1.2} & \multicolumn{1}{c}{14.6} \\
    \multicolumn{1}{l}{URSSM~\cite{cao2023unsupervised}}  & \multicolumn{1}{c}{3.9} & \multicolumn{1}{c}{\textbf{0.9}} & \multicolumn{1}{c}{4.1} \\
    \multicolumn{1}{l}{Ours}  & \multicolumn{1}{c}{\textbf{3.6}} & \multicolumn{1}{c}{1.0} & \multicolumn{1}{c}{\textbf{4.0}} \\
    \hline
    \end{tabular}
    \caption{\textbf{Non-isometric matching on SMAL and DT4D-H.} Our method outperforms all existing methods for challenging non-isometric inter-class shape matching on both SMAL and DT4D-H datasets and shows comparable performance on intra-class shape matching on DT4D-H dataset.}
    \label{tab:non-isometry}
\end{table}

\begin{figure}[ht!]
    \centering
    \begin{tabular}{cc}
    \hspace{-1.2cm}
       \newcommand{\pckLineWidth}{2pt}
\newcommand{\plotWidth}{\columnwidth}
\newcommand{\plotHeight}{0.8\columnwidth}
\newcommand{\pckTitle}{\textbf{SMAL}}
\definecolor{cPLOT0}{RGB}{28,213,227}
\definecolor{cPLOT1}{RGB}{80,150,80}
\definecolor{cPLOT2}{RGB}{90,130,213}
\definecolor{cPLOT3}{RGB}{247,179,43}
\definecolor{cPLOT4}{RGB}{124,42,43}
\definecolor{cPLOT5}{RGB}{242,64,0}

\pgfplotsset{%
    label style = {font=\large},
    tick label style = {font=\large},
    title style =  {font=\LARGE},
    legend style={  fill= gray!10,
                    fill opacity=0.6, 
                    font=\large,
                    draw=gray!20, %
                    text opacity=1}
}
\begin{tikzpicture}[scale=0.5, transform shape]
	\begin{axis}[
		width=\plotWidth,
		height=\plotHeight,
		grid=major,
		title=\pckTitle,
		legend style={
			at={(0.97,0.03)},
			anchor=south east,
			legend columns=1},
		legend cell align={left},
        xlabel={\LARGE Mean Geodesic Error},
		xmin=0,
        xmax=0.2,
        ylabel near ticks,
        xtick={0, 0.05, 0.1, 0.15, 0.2},
	ymin=0,
        ymax=1,
        ytick={0, 0.20, 0.40, 0.60, 0.80, 1.0}
	]

\addplot [color=cPLOT5, smooth, line width=\pckLineWidth]
table[row sep=crcr]{%
0.0 0.06488340934619913 \\
0.010526315789473684 0.13704104694753083 \\
0.021052631578947368 0.3482007280994979 \\
0.031578947368421054 0.5091134238269884 \\
0.042105263157894736 0.6313163984788405 \\
0.05263157894736842 0.7207291821737424 \\
0.06315789473684211 0.788068912316791 \\
0.07368421052631578 0.8386001001475146 \\
0.08421052631578947 0.877629210593983 \\
0.09473684210526316 0.9058721630509805 \\
0.10526315789473684 0.9262048152007687 \\
0.11578947368421053 0.9395975152589625 \\
0.12631578947368421 0.9487271792234507 \\
0.1368421052631579 0.9551447402254672 \\
0.14736842105263157 0.9603266974327049 \\
0.15789473684210525 0.9645085328389114 \\
0.16842105263157894 0.9678350543367934 \\
0.17894736842105263 0.9708245929815539 \\
0.18947368421052632 0.9734906822211095 \\
0.2 0.9757561814023359 \\
    };
\addlegendentry{\textcolor{black}{AttentiveFMaps: {0.78}}} 

\addplot [color=cPLOT3, smooth, line width=\pckLineWidth]
table[row sep=crcr]{%
0.0 0.11681666238107483 \\
0.010526315789473684 0.2009649010028285 \\
0.021052631578947368 0.4125514271020828 \\
0.031578947368421054 0.5778947030084854 \\
0.042105263157894736 0.6992298791463101 \\
0.05263157894736842 0.7843635896117254 \\
0.06315789473684211 0.8472923630753407 \\
0.07368421052631578 0.893088840318848 \\
0.08421052631578947 0.9245872975057855 \\
0.09473684210526316 0.944853432759064 \\
0.10526315789473684 0.9571213679609154 \\
0.11578947368421053 0.9645358704037027 \\
0.12631578947368421 0.9692600925687838 \\
0.1368421052631579 0.9726851375674981 \\
0.14736842105263157 0.9756203394188737 \\
0.15789473684210525 0.9781679094883003 \\
0.16842105263157894 0.9803272049370018 \\
0.17894736842105263 0.9822512213936745 \\
0.18947368421052632 0.9838518899460016 \\
0.2 0.98525777834919 \\
    };
\addlegendentry{\textcolor{black}{URSSM: {0.82}}}       

\addplot [color=cPLOT1, smooth, line width=\pckLineWidth]
table[row sep=crcr]{%
0.0  0.12317305219850862 \\
0.005128205128205128  0.13463293905888404 \\
0.010256410256410256  0.2076780663409617 \\
0.015384615384615385  0.32447672923630755 \\
0.020512820512820513  0.43185587554641297 \\
0.02564102564102564  0.5207431216250964 \\
0.03076923076923077  0.6035677552069941 \\
0.035897435897435895  0.6739926716379532 \\
0.041025641025641026  0.7303259192594498 \\
0.046153846153846156  0.776029827719208 \\
0.05128205128205128  0.8147878632039085 \\
0.05641025641025641  0.846931087683209 \\
0.06153846153846154  0.8739296734379017 \\
0.06666666666666667  0.8963679609154024 \\
0.07179487179487179  0.9148212908202623 \\
0.07692307692307693  0.9293783749035742 \\
0.08205128205128205  0.9411976086397531 \\
0.08717948717948718  0.9505110568269478 \\
0.09230769230769231  0.9579937001799949 \\
0.09743589743589744  0.9637490357418359 \\
0.10256410256410256  0.9681672666495243 \\
0.1076923076923077  0.9715678837747493 \\
0.11282051282051282  0.9743533041913088 \\
0.11794871794871795  0.9764862432501928 \\
0.12307692307692308  0.9781299820005143 \\
0.1282051282051282  0.9796650809976858 \\
0.13333333333333333  0.9809160452558499 \\
0.13846153846153847  0.98205901259964 \\
0.14358974358974358  0.9830367703779892 \\
0.14871794871794872  0.9839001028542042 \\
0.15384615384615385  0.984615582411931 \\
0.15897435897435896  0.9853445615839548 \\
0.1641025641025641  0.9859919002314219 \\
0.16923076923076924  0.9865460272563641 \\
0.17435897435897435  0.9870911545384418 \\
0.1794871794871795  0.9875674980714837 \\
0.18461538461538463  0.9880547698637182 \\
0.18974358974358974  0.9884764721007971 \\
0.19487179487179487  0.9888371046541528 \\
0.2  0.9892433787606069 \\
    };
\addlegendentry{\textcolor{black}{Ours: \textbf{0.84}}}  
\end{axis}
\end{tikzpicture}  & 
       \hspace{-1.3cm} \newcommand{\pckLineWidth}{2pt}
\newcommand{\plotWidth}{\columnwidth}
\newcommand{\plotHeight}{0.8\columnwidth}
\newcommand{\pckTitle}{\textbf{DT4D-H inter-class}}
\definecolor{cPLOT0}{RGB}{28,213,227}
\definecolor{cPLOT1}{RGB}{80,150,80}
\definecolor{cPLOT2}{RGB}{90,130,213}
\definecolor{cPLOT3}{RGB}{247,179,43}
\definecolor{cPLOT4}{RGB}{124,42,43}
\definecolor{cPLOT5}{RGB}{242,64,0}

\pgfplotsset{%
    label style = {font=\large},
    tick label style = {font=\large},
    title style =  {font=\LARGE},
    legend style={  fill= gray!10,
                    fill opacity=0.6, 
                    font=\large,
                    draw=gray!20, %
                    text opacity=1}
}
\begin{tikzpicture}[scale=0.5, transform shape]
	\begin{axis}[
		width=\plotWidth,
		height=\plotHeight,
		grid=major,
		title=\pckTitle,
		legend style={
			at={(0.97,0.03)},
			anchor=south east,
			legend columns=1},
		legend cell align={left},
        xlabel={\LARGE Mean Geodesic Error},
		xmin=0,
        xmax=0.2,
        ylabel near ticks,
        xtick={0, 0.05, 0.1, 0.15, 0.2},
	ymin=0,
        ymax=1,
        ytick={0, 0.20, 0.40, 0.60, 0.80, 1.0}
	]

\addplot [color=cPLOT5, smooth, line width=\pckLineWidth]
table[row sep=crcr]{%
0.0 0.03456135231465077 \\
0.010526315789473684 0.08081747504012339 \\
0.021052631578947368 0.26366498530545884 \\
0.031578947368421054 0.4222855743377035 \\
0.042105263157894736 0.5487909831794402 \\
0.05263157894736842 0.6365265648123059 \\
0.06315789473684211 0.6959819288409029 \\
0.07368421052631578 0.7381741251016112 \\
0.08421052631578947 0.7694500489818038 \\
0.09473684210526316 0.7939294870458762 \\
0.10526315789473684 0.8120624465889905 \\
0.11578947368421053 0.8260296600454384 \\
0.12631578947368421 0.8372674406486441 \\
0.1368421052631579 0.8469622944327491 \\
0.14736842105263157 0.8553124413781603 \\
0.15789473684210525 0.8624852533505638 \\
0.16842105263157894 0.8689187527356859 \\
0.17894736842105263 0.8752861787940055 \\
0.18947368421052632 0.8819866811180357 \\
0.2 0.8887820205515142 \\
    };
\addlegendentry{\textcolor{black}{AttentiveFMaps: {0.68}}} 

\addplot [color=cPLOT3, smooth, line width=\pckLineWidth]
table[row sep=crcr]{%
0.0 0.050371427975905124 \\
0.010526315789473684 0.12944817308293557 \\
0.021052631578947368 0.3864891510515455 \\
0.031578947368421054 0.5803202576234445 \\
0.042105263157894736 0.7183634866706964 \\
0.05263157894736842 0.8076292806970007 \\
0.06315789473684211 0.8629943514600746 \\
0.07368421052631578 0.897993205077433 \\
0.08421052631578947 0.9198287721199742 \\
0.09473684210526316 0.9337159472247119 \\
0.10526315789473684 0.942982262334035 \\
0.11578947368421053 0.9494709965191654 \\
0.12631578947368421 0.9543339516851825 \\
0.1368421052631579 0.95824395022615 \\
0.14736842105263157 0.9615212497655127 \\
0.15789473684210525 0.9646511870271172 \\
0.16842105263157894 0.9679835129332806 \\
0.17894736842105263 0.9716687162598745 \\
0.18947368421052632 0.9755318173291369 \\
0.2 0.9789688809221085 \\
    };
\addlegendentry{\textcolor{black}{URSSM: {0.81}}}       

\addplot [color=cPLOT1, smooth, line width=\pckLineWidth]
table[row sep=crcr]{%
0.0  0.05191133251349605 \\
0.005128205128205128  0.059836588365258354 \\
0.010256410256410256  0.12867290576734686 \\
0.015384615384615385  0.25443889780519835 \\
0.020512820512820513  0.38276476228192674 \\
0.02564102564102564  0.48862830106092503 \\
0.03076923076923077  0.5770309731746461 \\
0.035897435897435895  0.6524018800675324 \\
0.041025641025641026  0.7145776101048419 \\
0.046153846153846156  0.7647877107780812 \\
0.05128205128205128  0.8039087479417221 \\
0.05641025641025641  0.834578548054276 \\
0.06153846153846154  0.8592774662859287 \\
0.06666666666666667  0.8792195218542218 \\
0.07179487179487179  0.8950679492256707 \\
0.07692307692307693  0.9076370969422849 \\
0.08205128205128205  0.9177336015173938 \\
0.08717948717948718  0.9256923108989724 \\
0.09230769230769231  0.9320629676720095 \\
0.09743589743589744  0.9371590553807032 \\
0.10256410256410256  0.9414431915292745 \\
0.1076923076923077  0.9451615357358735 \\
0.11282051282051282  0.9483395168518248 \\
0.11794871794871795  0.9510300769118536 \\
0.12307692307692308  0.9533681764178669 \\
0.1282051282051282  0.9554263501261021 \\
0.13333333333333333  0.9572072034516539 \\
0.13846153846153847  0.958844446297184 \\
0.14358974358974358  0.9602656481230589 \\
0.14871794871794872  0.9615184359172103 \\
0.15384615384615385  0.9627219292577693 \\
0.15897435897435896  0.9639430352043687 \\
0.1641025641025641  0.9651800862913479 \\
0.16923076923076924  0.9665380286387227 \\
0.17435897435897435  0.9681233716155658 \\
0.1794871794871795  0.9698857786022469 \\
0.18461538461538463  0.9717841882568731 \\
0.18974358974358974  0.9738263126081247 \\
0.19487179487179487  0.9760231986160035 \\
0.2  0.9791147216374513 \\
    };
\addlegendentry{\textcolor{black}{Ours: \textbf{0.82}}}  
\end{axis}
\end{tikzpicture}
       \\
    \end{tabular}
    \caption{\textbf{Non-isometric matching on SMAL and DT4D-H inter-class datasets.} Proportion of correct keypoints (PCK) curves and corresponding area under curve (AUC) of our method compared to the existing state-of-the-art methods.
    }
    \label{fig:non_iso_pck}
\end{figure}
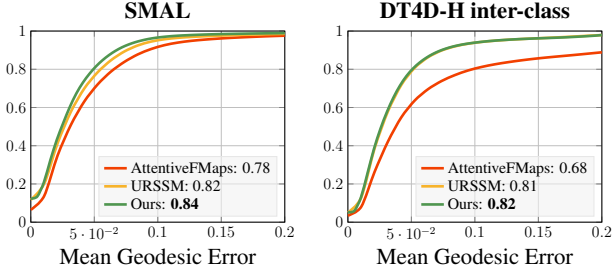

\noindent \textbf{Results.}
Tab.~\ref{tab:non-isometry} summarises the matching results on the SMAL and DT4D-H datasets. In the context of inter-class shape matching, our approach outperforms the existing state of the art on both challenging non-isometric datasets, even in comparison to supervised methods. Meanwhile, our method demonstrates comparable and near-perfect matching results for intra-class matching on the DT4D-H dataset.~\cref{fig:non_iso_pck} shows the PCK curves and the corresponding AUC of our method compared to  existing state-of-the-art methods.~\cref{fig:shrec20} demonstrates the qualitative results of our method applied on the challenging SHREC'20 dataset~\cite{dyke2020track}.

\begin{figure}[ht!]
    \centering
    \def\rowOnecolumnOne{dog-bison}
\def\rowOnecolumnTwo{dog-elephant_a}
\def\rowOnecolumnThree{dog-giraffe_b}
\def\rowOnecolumnFour{dog-pig}
\def\rowOnecolumnFive{dog-leopard}
\def\hspaceCols{-0.5cm}
\def\height{1.8cm}
\def\width{1.6cm}
\def\heightT{\height}
\def\widthT{\width}
\def\heightQ{\height}
\def\widthQ{\width}
\begin{tabular}{cccccc}%
        \setlength{\tabcolsep}{0pt} 
        \hspace{\hspaceCols}
        \includegraphics[height=\heightT, width=\widthT]{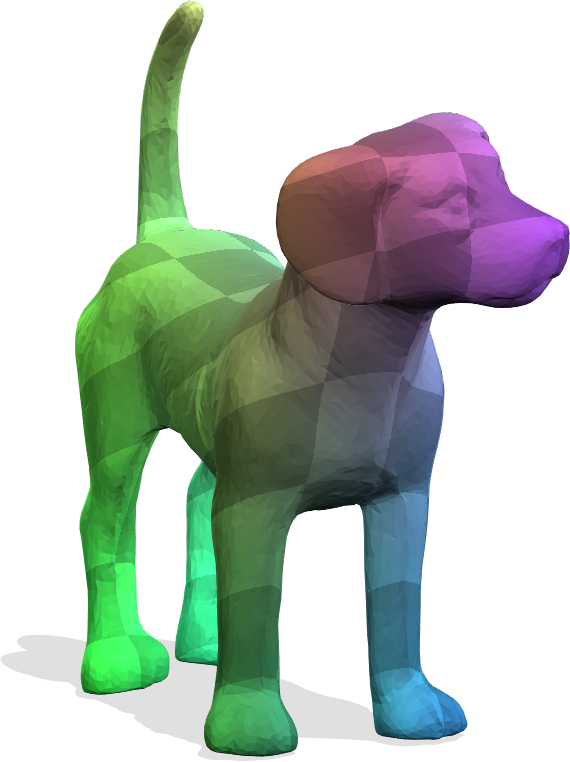}&
        \hspace{\hspaceCols}
        \includegraphics[height=\heightT, width=\widthT]{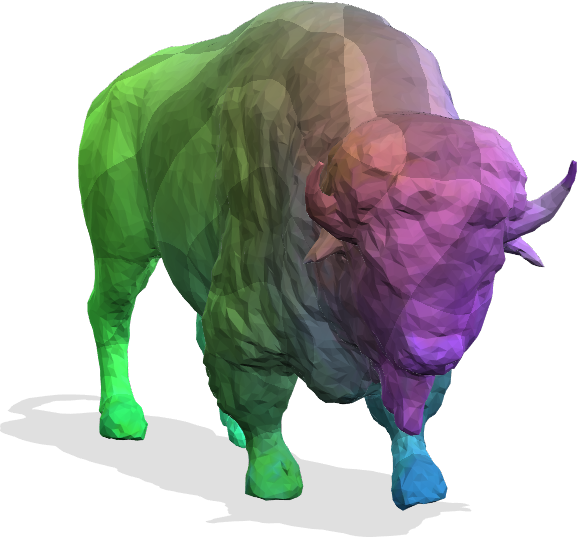}&
        \hspace{\hspaceCols}
        \includegraphics[height=\heightT, width=\widthT]{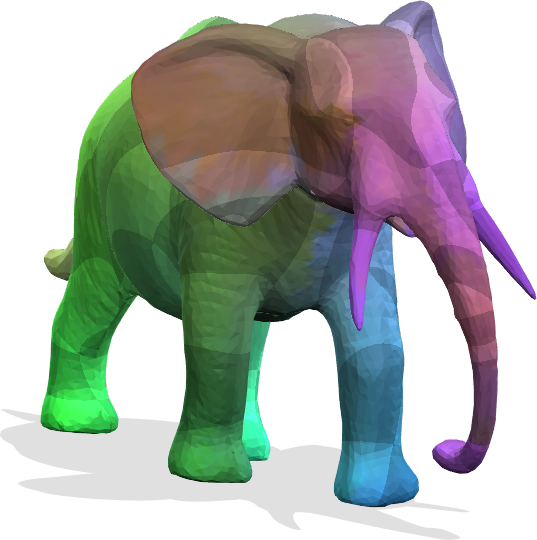}&
        \hspace{\hspaceCols}
        \includegraphics[height=\heightT, width=\widthT]{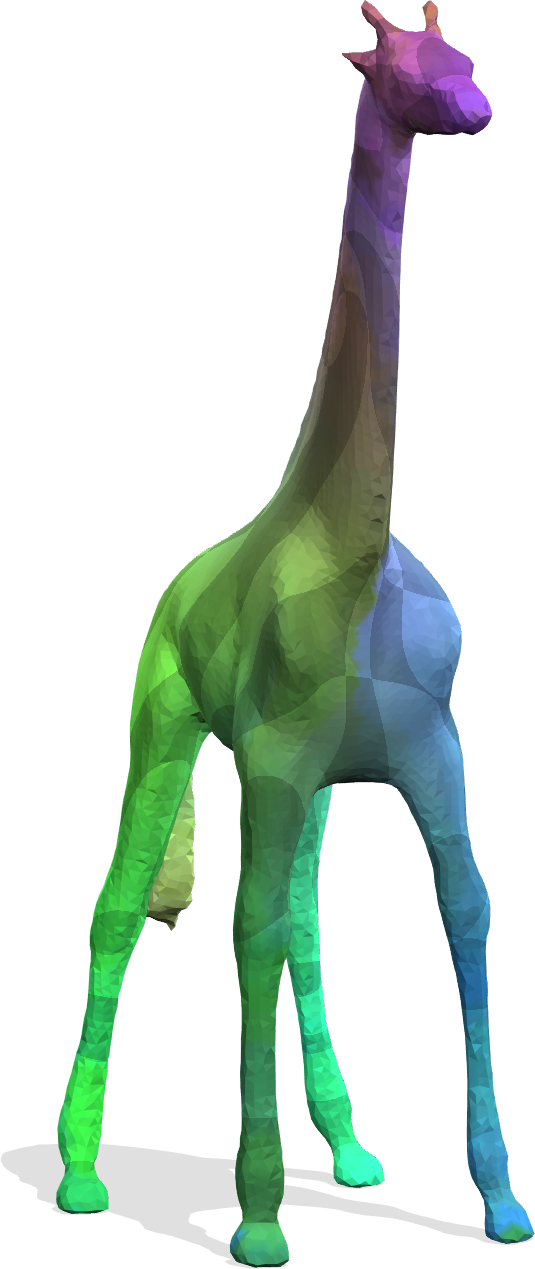}&
        \hspace{\hspaceCols}
        \includegraphics[height=\heightT, width=\widthT]{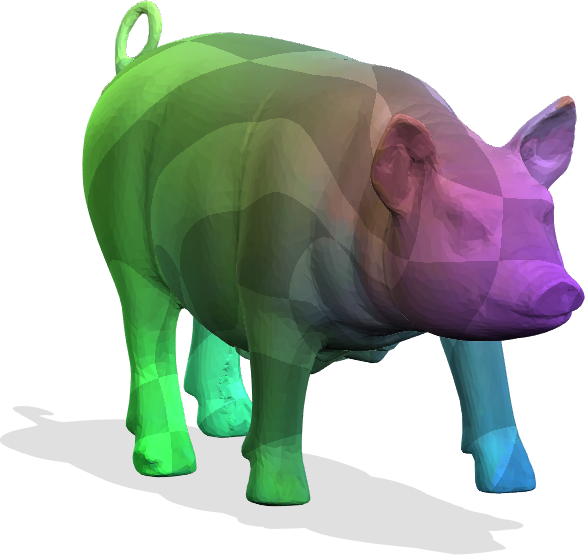}
        &
        \hspace{\hspaceCols}
        \includegraphics[height=\heightT, width=\widthT]{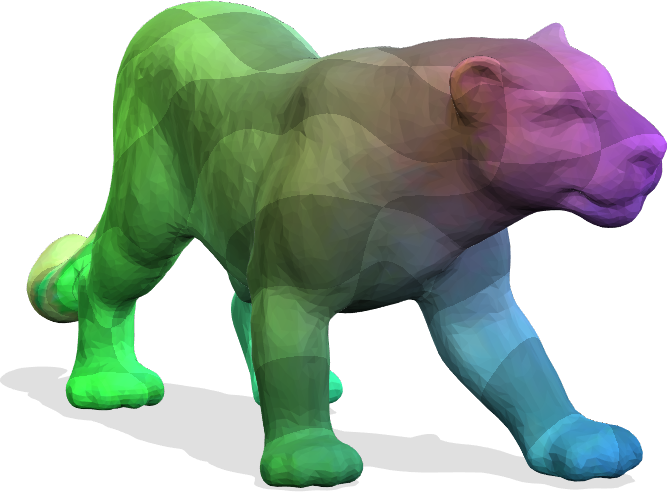}
    \end{tabular}
    \caption{\textbf{Qualitative results on the challenging SHREC'20 dataset of our method.} Our method is capable of finding reliable correspondences even for shapes with extremely large non-isometric deformations.}
    \label{fig:shrec20}
\end{figure}

\subsection{Matching with topological noise}
\textbf{Datasets.} The mesh topology is often degraded due to self-intersections of separate parts of  real-world scanned objects. Such topological noise presents a large challenge to matching methods based on the functional map framework as it distorts the intrinsic shape geometry~\cite{lahner2016shrec}. To evaluate our method for matching with topologically noisy shapes, we use the TOPKIDS dataset~\cite{lahner2016shrec}. Due to the small amount of data (26 shapes), we consider only axiomatic and unsupervised methods for comparison.

\begin{table}[ht!]
\setlength{\tabcolsep}{3pt}
    \centering
    \small
    \begin{tabular}{@{}lcc@{}}
    \toprule
    \multicolumn{1}{c}{\textbf{Geo. error ($\times$100)}}      & \multicolumn{1}{c}{\textbf{TOPKIDS}} & \multicolumn{1}{c}{\textbf{Fully intrinsic}}
    \\ \midrule
    \multicolumn{3}{c}{Axiomatic Methods} \\
    \multicolumn{1}{l}{ZoomOut~\cite{melzi2019zoomout}}  & \multicolumn{1}{c}{33.7} & \multicolumn{1}{c}{\cmark}  \\ 
    \multicolumn{1}{l}{Smooth Shells~\cite{eisenberger2020smooth}}  & \multicolumn{1}{c}{11.8} & \multicolumn{1}{c}{\xmark}\\
    \multicolumn{1}{l}{DiscreteOp~\cite{ren2021discrete}}  & \multicolumn{1}{c}{35.5} & \multicolumn{1}{c}{\cmark}
    \\\midrule
    \multicolumn{3}{c}{Unsupervised Methods} \\
    \multicolumn{1}{l}{WSupFMNet~\cite{sharma2020weakly}}  & \multicolumn{1}{c}{47.9} & \multicolumn{1}{c}{\cmark}\\
    \multicolumn{1}{l}{Deep Shells~\cite{eisenberger2020deep}}  & \multicolumn{1}{c}{13.7} & \multicolumn{1}{c}{\xmark} \\
    \multicolumn{1}{l}{NeuroMorph~\cite{eisenberger2021neuromorph}}  & \multicolumn{1}{c}{13.8} & \multicolumn{1}{c}{\xmark} \\
    \multicolumn{1}{l}{AttentiveFMaps~\cite{li2022learning}}  & \multicolumn{1}{c}{23.4} & \multicolumn{1}{c}{\cmark} \\
    \multicolumn{1}{l}{AttentiveFMaps-Fast~\cite{li2022learning}}  & \multicolumn{1}{c}{28.5} & \multicolumn{1}{c}{\cmark}\\
    \multicolumn{1}{l}{URSSM~\cite{cao2023unsupervised}}  & \multicolumn{1}{c}{{9.2}} & \multicolumn{1}{c}{\cmark} \\
    \multicolumn{1}{l}{Ours}  & \multicolumn{1}{c}{\textbf{6.6}} & \multicolumn{1}{c}{\cmark} \\
    \hline
    \end{tabular}
    \caption{\textbf{Quantitative results on the TOPKIDS dataset.} We distinguish fully intrinsic methods (i.e.\ methods based on the functional map framework) from the methods that rely on additional extrinsic information (e.g.\ a weak alignment).  Our method outperforms  all existing methods substantially, even in comparison to methods relying on additional extrinsic information.}
    \label{tab:topkids}
\end{table}

\begin{figure}[ht!]
    \centering
    \begin{tabular}{cc}
    \hspace{-1.2cm}
       \newcommand{\pckLineWidth}{2pt}
\newcommand{\plotWidth}{\columnwidth}
\newcommand{\plotHeight}{0.8\columnwidth}
\newcommand{\pckTitle}{\textbf{TOPKIDS}}
\definecolor{cPLOT0}{RGB}{28,213,227}
\definecolor{cPLOT1}{RGB}{80,150,80}
\definecolor{cPLOT2}{RGB}{90,130,213}
\definecolor{cPLOT3}{RGB}{247,179,43}
\definecolor{cPLOT4}{RGB}{124,42,43}
\definecolor{cPLOT5}{RGB}{242,64,0}

\pgfplotsset{%
    label style = {font=\large},
    tick label style = {font=\large},
    title style =  {font=\LARGE},
    legend style={  fill= gray!10,
                    fill opacity=0.6, 
                    font=\large,
                    draw=gray!20, %
                    text opacity=1}
}
\begin{tikzpicture}[scale=0.5, transform shape]
	\begin{axis}[
		width=\plotWidth,
		height=\plotHeight,
		grid=major,
		title=\pckTitle,
		legend style={
			at={(0.97,0.03)},
			anchor=south east,
			legend columns=1},
		legend cell align={left},
        xlabel={\LARGE Mean Geodesic Error},
		xmin=0,
        xmax=0.2,
        ylabel near ticks,
        xtick={0, 0.05, 0.1, 0.15, 0.2},
	ymin=0,
        ymax=1,
        ytick={0, 0.20, 0.40, 0.60, 0.80, 1.0}
	]

\addplot [color=cPLOT5, smooth, line width=\pckLineWidth]
table[row sep=crcr]{%
0.0 0.06531700633935275 \\
0.010526315789473684 0.16553915459815935 \\
0.021052631578947368 0.2928293328585914 \\
0.031578947368421054 0.37656451974952204 \\
0.042105263157894736 0.4298994527567283 \\
0.05263157894736842 0.46736665299203517 \\
0.06315789473684211 0.4941521599467463 \\
0.07368421052631578 0.5145363912905498 \\
0.08421052631578947 0.5312052510584939 \\
0.09473684210526316 0.5456061860936738 \\
0.10526315789473684 0.5583274635622673 \\
0.11578947368421053 0.5697406206218604 \\
0.12631578947368421 0.5801591417491659 \\
0.1368421052631579 0.5905737926977468 \\
0.14736842105263157 0.5999860673565905 \\
0.15789473684210525 0.6084191867980463 \\
0.16842105263157894 0.6156177192262738 \\
0.17894736842105263 0.6220847878755041 \\
0.18947368421052632 0.6285015442013112 \\
0.2 0.6345622440844318 \\
    };
\addlegendentry{\textcolor{black}{AttentiveFMaps: {0.50}}} 

\addplot [color=cPLOT3, smooth, line width=\pckLineWidth]
table[row sep=crcr]{%
0.0 0.1457277097056342 \\
0.010526315789473684 0.30846098472827477 \\
0.021052631578947368 0.48728259271013136 \\
0.031578947368421054 0.609193222543017 \\
0.042105263157894736 0.6856911752184716 \\
0.05263157894736842 0.737915366931645 \\
0.06315789473684211 0.7737454815663387 \\
0.07368421052631578 0.7998188756356769 \\
0.08421052631578947 0.8187479197789354 \\
0.09473684210526316 0.832680563188408 \\
0.10526315789473684 0.843435789864776 \\
0.11578947368421053 0.8532583034684542 \\
0.12631578947368421 0.8613972893268211 \\
0.1368421052631579 0.8683016881719597 \\
0.14736842105263157 0.8742308019784354 \\
0.15789473684210525 0.8792426834271206 \\
0.16842105263157894 0.8829967567902286 \\
0.17894736842105263 0.8861354717360848 \\
0.18947368421052632 0.8891309900691214 \\
0.2 0.8914879289125572 \\
    };
\addlegendentry{\textcolor{black}{URSSM: {0.76}}}       

\addplot [color=cPLOT1, smooth, line width=\pckLineWidth]
table[row sep=crcr]{%
0.0  0.15510515275595427 \\
0.005128205128205128  0.2005681422368085 \\
0.010256410256410256  0.32403845409581017 \\
0.015384615384615385  0.42582415455945755 \\
0.020512820512820513  0.5108984232891874 \\
0.02564102564102564  0.5813318059027965 \\
0.03076923076923077  0.6377899731409596 \\
0.035897435897435895  0.6823628215150975 \\
0.041025641025641026  0.7189747122522118 \\
0.046153846153846156  0.7491272746975456 \\
0.05128205128205128  0.7750265107242652 \\
0.05641025641025641  0.7959254758384742 \\
0.06153846153846154  0.8132058238449451 \\
0.06666666666666667  0.8273745481566339 \\
0.07179487179487179  0.8397591200761652 \\
0.07692307692307693  0.8501969920970951 \\
0.08205128205128205  0.8588932836918409 \\
0.08717948717948718  0.8658325141455032 \\
0.09230769230769231  0.8718003297392274 \\
0.09743589743589744  0.8770792535199275 \\
0.10256410256410256  0.8819943805004915 \\
0.1076923076923077  0.8864566965702476 \\
0.11282051282051282  0.8906829317377877 \\
0.11794871794871795  0.8946614754669371 \\
0.12307692307692308  0.8984271593662195 \\
0.1282051282051282  0.9017090709248953 \\
0.13333333333333333  0.904278869598198 \\
0.13846153846153847  0.9064887416500894 \\
0.14358974358974358  0.9085244556593624 \\
0.14871794871794872  0.910304737872795 \\
0.15384615384615385  0.912073409550053 \\
0.15897435897435896  0.9134628037122754 \\
0.1641025641025641  0.914716741619128 \\
0.16923076923076924  0.915982290062155 \\
0.17435897435897435  0.9170311084965904 \\
0.1794871794871795  0.918176681399147 \\
0.18461538461538463  0.9195622053826445 \\
0.18974358974358974  0.9207967923958729 \\
0.19487179487179487  0.9219036635111809 \\
0.2  0.9229563521243411 \\
    };
\addlegendentry{\textcolor{black}{Ours: \textbf{0.79}}}  
\end{axis}
\end{tikzpicture}  & 
       \hspace{-1.3cm} \newcommand{\pckLineWidth}{2pt}
\newcommand{\plotWidth}{\columnwidth}
\newcommand{\plotHeight}{0.8\columnwidth}
\newcommand{\pckTitle}{\textbf{SHREC'16}}
\definecolor{cPLOT0}{RGB}{28,213,227}
\definecolor{cPLOT1}{RGB}{80,150,80}
\definecolor{cPLOT2}{RGB}{90,130,213}
\definecolor{cPLOT3}{RGB}{247,179,43}
\definecolor{cPLOT5}{RGB}{242,64,0}

\pgfplotsset{%
    label style = {font=\large},
    tick label style = {font=\large},
    title style =  {font=\LARGE},
    legend style={  fill= gray!10,
                    fill opacity=0.6, 
                    font=\large,
                    draw=gray!20, %
                    text opacity=1}
}
\begin{tikzpicture}[scale=0.5, transform shape]
	\begin{axis}[
		width=\plotWidth,
		height=\plotHeight,
		grid=major,
		title=\pckTitle,
		legend style={
			at={(0.97,0.03)},
			anchor=south east,
			legend columns=1},
		legend cell align={left},
        xlabel={\LARGE Mean Geodesic Error},
		xmin=0,
        xmax=0.2,
        ylabel near ticks,
        xtick={0, 0.05, 0.1, 0.15, 0.2},
	ymin=0,
        ymax=1,
        ytick={0, 0.20, 0.40, 0.60, 0.80, 1.0}
	]   

\addplot [color=cPLOT5, smooth, line width=\pckLineWidth]
table[row sep=crcr]{%
0.0 0.09820953899901268 \\
0.010526315789473684 0.2295535239614187 \\
0.021052631578947368 0.4490357047922837 \\
0.031578947368421054 0.6150508847877268 \\
0.042105263157894736 0.7275171831092884 \\
0.05263157894736842 0.8033685064935064 \\
0.06315789473684211 0.8551836504139135 \\
0.07368421052631578 0.8906309333940913 \\
0.08421052631578947 0.9150563435102909 \\
0.09473684210526316 0.931423017771702 \\
0.10526315789473684 0.9422823156375788 \\
0.11578947368421053 0.949004851143009 \\
0.12631578947368421 0.9532175609478241 \\
0.1368421052631579 0.9562068048910154 \\
0.14736842105263157 0.9580342902711324 \\
0.15789473684210525 0.9593598105111263 \\
0.16842105263157894 0.9602723665223665 \\
0.17894736842105263 0.9610579479000532 \\
0.18947368421052632 0.9617877553732816 \\
0.2 0.9624594155844156 \\
    };
\addlegendentry{\textcolor{black}{DPFM (CUTS): 0.82}}

\addplot [color=cPLOT5, smooth, dashed, line width=\pckLineWidth]
table[row sep=crcr]{%
0.0 0.025942907493282673 \\
0.010526315789473684 0.07528458314250837 \\
0.021052631578947368 0.19621478735278905 \\
0.031578947368421054 0.32684963805656314 \\
0.042105263157894736 0.4429562515998823 \\
0.05263157894736842 0.5427140272454511 \\
0.06315789473684211 0.6240629260910979 \\
0.07368421052631578 0.6905988740311462 \\
0.08421052631578947 0.7444385044431017 \\
0.09473684210526316 0.7867607237823916 \\
0.10526315789473684 0.8196455966886137 \\
0.11578947368421053 0.845340359174661 \\
0.12631578947368421 0.8644849641539301 \\
0.1368421052631579 0.879182984740388 \\
0.14736842105263157 0.8888336488112548 \\
0.15789473684210525 0.8954491076357266 \\
0.16842105263157894 0.8999575106092874 \\
0.17894736842105263 0.9036222205582479 \\
0.18947368421052632 0.9066243398036781 \\
0.2 0.9091440999824122 \\
    };
\addlegendentry{\textcolor{black}{DPFM (HOLES): 0.67}}  

\addplot [color=cPLOT3, smooth, line width=\pckLineWidth]
table[row sep=crcr]{%
0.0 0.30081240031897927 \\
0.010526315789473684 0.523318238778765 \\
0.021052631578947368 0.7392498765854029 \\
0.031578947368421054 0.8435755582137161 \\
0.042105263157894736 0.8925260594668489 \\
0.05263157894736842 0.919430441634389 \\
0.06315789473684211 0.9356464076858814 \\
0.07368421052631578 0.9457165641376167 \\
0.08421052631578947 0.9525720076706918 \\
0.09473684210526316 0.9573353364471785 \\
0.10526315789473684 0.9603566207184628 \\
0.11578947368421053 0.9624736557302347 \\
0.12631578947368421 0.9640732987772461 \\
0.1368421052631579 0.9652884578871421 \\
0.14736842105263157 0.9662745879851143 \\
0.15789473684210525 0.9669972753854332 \\
0.16842105263157894 0.9678516841345789 \\
0.17894736842105263 0.9686693058403585 \\
0.18947368421052632 0.9694786207944103 \\
0.2 0.9702392819169134 \\
    };
\addlegendentry{\textcolor{black}{URSSM (CUTS): 0.90}} 

\addplot [color=cPLOT3, smooth, dashed, line width=\pckLineWidth]
table[row sep=crcr]{%
0.0 0.13387466674560872 \\
0.010526315789473684 0.30691445179108456 \\
0.021052631578947368 0.53647557469513 \\
0.031578947368421054 0.6711887103599234 \\
0.042105263157894736 0.7441372476974756 \\
0.05263157894736842 0.7889208913716934 \\
0.06315789473684211 0.8188245414541472 \\
0.07368421052631578 0.8400239612301723 \\
0.08421052631578947 0.8551817030931929 \\
0.09473684210526316 0.8658946019318742 \\
0.10526315789473684 0.874153912593751 \\
0.11578947368421053 0.8804071667760827 \\
0.12631578947368421 0.8854623594497972 \\
0.1368421052631579 0.8897705050490979 \\
0.14736842105263157 0.8934534993670126 \\
0.15789473684210525 0.8964469117700836 \\
0.16842105263157894 0.8988787328409904 \\
0.17894736842105263 0.9011242274854117 \\
0.18947368421052632 0.9031198357541257 \\
0.2 0.9050370824416075 \\
    };
\addlegendentry{\textcolor{black}{URSSM (HOLES): 0.79}}

\addplot [color=cPLOT1, smooth, line width=\pckLineWidth]
table[row sep=crcr]{%
0.0  0.2883249791144528 \\
0.005128205128205128  0.3406788752183489 \\
0.010256410256410256  0.5121053106250475 \\
0.015384615384615385  0.6441589295207716 \\
0.020512820512820513  0.7370687609174451 \\
0.02564102564102564  0.8037541771094403 \\
0.03076923076923077  0.8490497076023392 \\
0.035897435897435895  0.8807126718310929 \\
0.041025641025641026  0.9037935748462064 \\
0.046153846153846156  0.9207262949039264 \\
0.05128205128205128  0.9335341193893826 \\
0.05641025641025641  0.9436529296726666 \\
0.06153846153846154  0.9513497284878863 \\
0.06666666666666667  0.9575133382699172 \\
0.07179487179487179  0.9621212121212122 \\
0.07692307692307693  0.9658485702893598 \\
0.08205128205128205  0.9687725468975469 \\
0.08717948717948718  0.9708741550846814 \\
0.09230769230769231  0.9727040138224349 \\
0.09743589743589744  0.9741743088782563 \\
0.10256410256410256  0.975434561783246 \\
0.1076923076923077  0.9763957716260347 \\
0.11282051282051282  0.9772181400470874 \\
0.11794871794871795  0.9779242139439508 \\
0.12307692307692308  0.9786575814536341 \\
0.1282051282051282  0.979209387104124 \\
0.13333333333333333  0.9797232190324295 \\
0.13846153846153847  0.980162290195185 \\
0.14358974358974358  0.9806393350801246 \\
0.14871794871794872  0.9809905920103289 \\
0.15384615384615385  0.9813240487582593 \\
0.15897435897435896  0.9816468253968254 \\
0.1641025641025641  0.9819185748462064 \\
0.16923076923076924  0.9821665907192223 \\
0.17435897435897435  0.9824241000227842 \\
0.1794871794871795  0.9826721158958001 \\
0.18461538461538463  0.9829177584111795 \\
0.18974358974358974  0.983159840890104 \\
0.19487179487179487  0.9833924299384825 \\
0.2  0.9836451925267715 \\
    };
\addlegendentry{\textcolor{black}{Ours (CUTS): \textbf{0.91}}}  

\addplot [color=cPLOT1, smooth, dashed, line width=\pckLineWidth]
table[row sep=crcr]{%
0.0  0.1236702475007009 \\
0.005128205128205128  0.15652203440595827 \\
0.010256410256410256  0.29120556680673076 \\
0.015384615384615385  0.4231769178126322 \\
0.020512820512820513  0.5313402790717113 \\
0.02564102564102564  0.616178009650664 \\
0.03076923076923077  0.6798380875594895 \\
0.035897435897435895  0.7268245623505688 \\
0.041025641025641026  0.7627246147657598 \\
0.046153846153846156  0.7900249189828319 \\
0.05128205128205128  0.8117754818801903 \\
0.05641025641025641  0.8296959048237648 \\
0.06153846153846154  0.8446020189426062 \\
0.06666666666666667  0.8573758273676951 \\
0.07179487179487179  0.8682593803166156 \\
0.07692307692307693  0.8774224612153707 \\
0.08205128205128205  0.8851593613356993 \\
0.08717948717948718  0.8920525684314276 \\
0.09230769230769231  0.897734653755 \\
0.09743589743589744  0.9024694346298983 \\
0.10256410256410256  0.9064893837471115 \\
0.1076923076923077  0.909861543792805 \\
0.11282051282051282  0.9127443792979151 \\
0.11794871794871795  0.9151613987368113 \\
0.12307692307692308  0.9174660999092747 \\
0.1282051282051282  0.9194869580208302 \\
0.13333333333333333  0.9213188776531925 \\
0.13846153846153847  0.9229122298049145 \\
0.14358974358974358  0.9244315737965838 \\
0.14871794871794872  0.9257854877834295 \\
0.15384615384615385  0.9269791958708671 \\
0.15897435897435896  0.9281345938519244 \\
0.1641025641025641  0.9291402341444047 \\
0.16923076923076924  0.9301162711728639 \\
0.17435897435897435  0.9311088512018054 \\
0.1794871794871795  0.9320544142820076 \\
0.18461538461538463  0.9328998486750798 \\
0.18974358974358974  0.933825386017856 \\
0.19487179487179487  0.9346891047798823 \\
0.2  0.9355023238562257 \\
    };
\addlegendentry{\textcolor{black}{Ours (HOLES): \textbf{0.81}}}  
\end{axis}
\end{tikzpicture}
       \\
    \end{tabular}
    \caption{\textbf{Matching with topological noise on TOPKIDS and partial shape matching on SHREC'16.} Our method improves the state of the art over existing approaches.
    }
    \label{fig:topkids_shrec16_pck}
\end{figure}
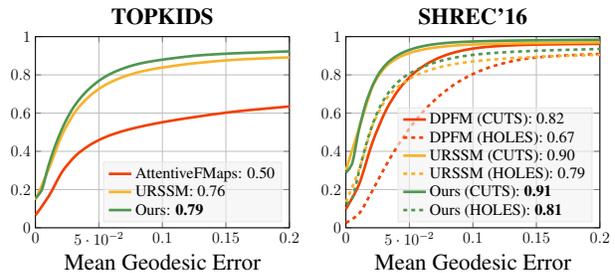

\begin{figure}
    \centering
    \includegraphics[width=\columnwidth]{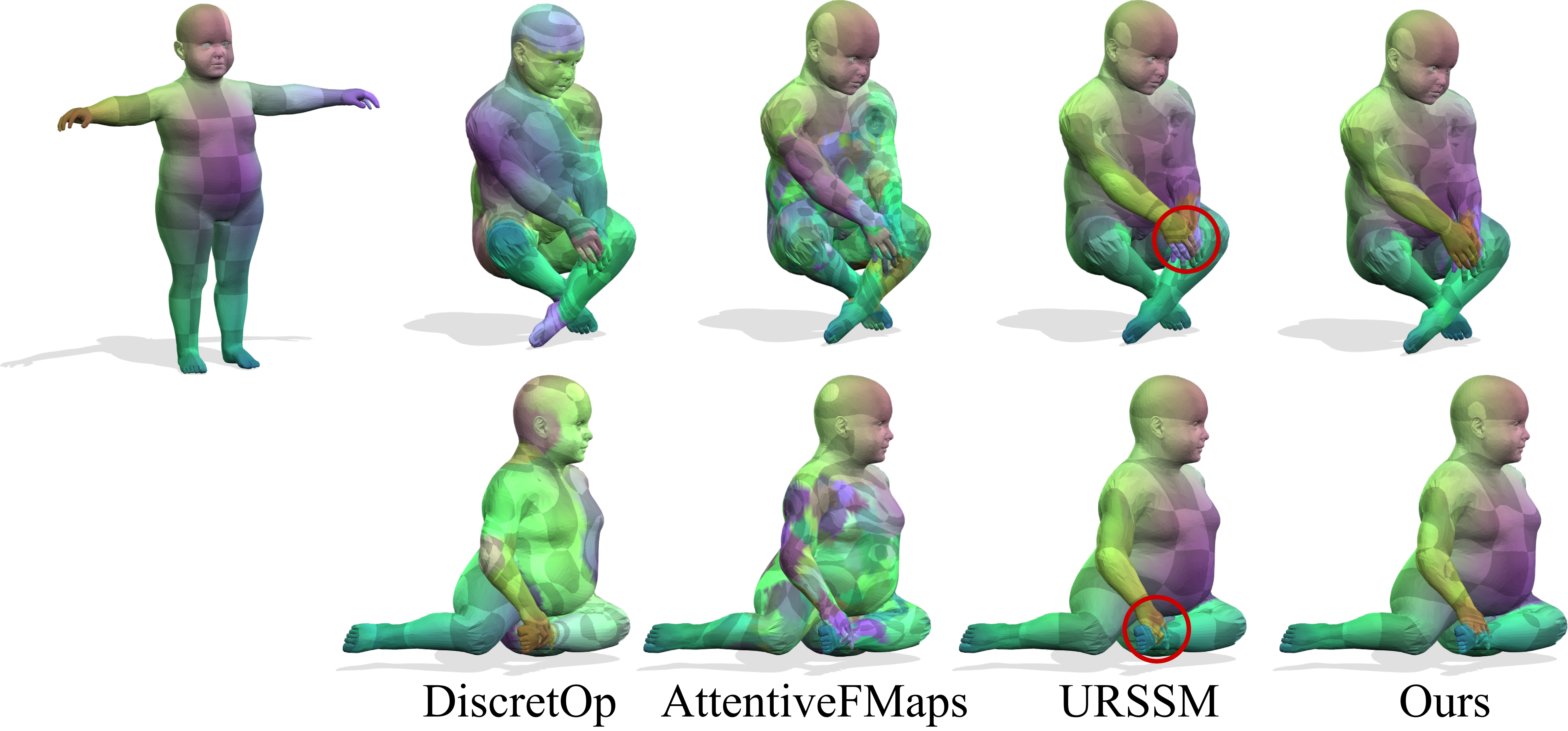}
    \caption{\textbf{Qualitative results on TOPKIDS dataset.} Compared to existing fully intrinsic approaches, our method is more robust to topological noise.}
    \label{fig:topkids}
\end{figure}

\noindent \textbf{Results.} We compare our method with  state-of-the-art axiomatic methods and unsupervised methods. The quantitative results are summarised in~\cref{tab:topkids}. Our method outperforms  the existing methods substantially, even in comparison to methods relying on additional extrinsic alignment information. We show the PCK curves of our method in~\cref{fig:topkids_shrec16_pck} (left) and qualitative results in~\cref{fig:topkids}. 

\subsection{Partial shape matching}
\textbf{Datasets.} We evaluate our method on the SHREC’16 partial dataset~\cite{cosmo2016shrec}. The dataset contains 200 training shapes and 400 test shapes, with 8 different classes (humans and animals). Each class has a complete shape to be matched by the other partial shapes. The dataset is divided into two subsets, namely CUTS (missing a large part) with 120/200 train/test split, and HOLES (missing many small parts) with 80/200 train/test split.

\begin{table}[th!]
    \setlength{\tabcolsep}{3.5pt}
    \small
    \centering
    \begin{tabular}{@{}lcccc@{}}
    \toprule
    \multicolumn{1}{l}{Train}  & \multicolumn{2}{c}{\textbf{CUTS}} & \multicolumn{2}{c}{\textbf{HOLES}}\\ \cmidrule(lr){2-3} \cmidrule(lr){4-5}
    \multicolumn{1}{l}{Test} & \multicolumn{1}{c}{\textbf{CUTS}} & \multicolumn{1}{c}{\textbf{HOLES}} & \multicolumn{1}{c}{\textbf{CUTS}} & \multicolumn{1}{c}{\textbf{HOLES}}
    \\ \midrule
    \multicolumn{5}{c}{Axiomatic Methods} \\
    \multicolumn{1}{l}{PFM~\cite{rodola2017partial}}  & \multicolumn{1}{c}{9.7} & \multicolumn{1}{c}{23.2} & \multicolumn{1}{c}{9.7} & \multicolumn{1}{c}{23.2} \\
    \multicolumn{1}{l}{FSP~\cite{litany2017fully}}  & \multicolumn{1}{c}{16.1} & \multicolumn{1}{c}{33.7}  & \multicolumn{1}{c}{16.1} & \multicolumn{1}{c}{33.7} \\
    \midrule
    \multicolumn{5}{c}{Supervised Methods} \\ 
    \multicolumn{1}{l}{GeomFMaps~\cite{donati2020deep}}  & \multicolumn{1}{c}{12.8} & \multicolumn{1}{c}{20.6} & \multicolumn{1}{c}{19.8} & \multicolumn{1}{c}{15.3}\\
    \multicolumn{1}{l}{DPFM~\cite{attaiki2021dpfm}}  & \multicolumn{1}{c}{3.2} & \multicolumn{1}{c}{15.8} & \multicolumn{1}{c}{8.6} & \multicolumn{1}{c}{13.1}\\
    \midrule
    \multicolumn{5}{c}{Unsupervised Methods} \\
    \multicolumn{1}{l}{DPFM-unsup~\cite{attaiki2021dpfm}}  & \multicolumn{1}{c}{9.0} & \multicolumn{1}{c}{22.8} & \multicolumn{1}{c}{16.5} & \multicolumn{1}{c}{20.5} \\
    \multicolumn{1}{l}{ConsistFMaps~\cite{cao2022unsupervised}}  & \multicolumn{1}{c}{8.4} & \multicolumn{1}{c}{23.7} & \multicolumn{1}{c}{15.7} & \multicolumn{1}{c}{17.9} \\
    \multicolumn{1}{l}{URSSM~\cite{cao2023unsupervised}}  & \multicolumn{1}{c}{3.3} & \multicolumn{1}{c}{\textbf{13.7}}  & \multicolumn{1}{c}{{5.2}}  & \multicolumn{1}{c}{{9.1}} \\
    \multicolumn{1}{l}{Ours}  & \multicolumn{1}{c}{\textbf{2.3}} & \multicolumn{1}{c}{{15.2}}  & \multicolumn{1}{c}{\textbf{5.1}}  & \multicolumn{1}{c}{\textbf{6.9}} \\
    \hline
    \end{tabular}
    \caption{\textbf{Partial shape matching on SHREC'16 dataset.} Our method substantially outperforms state-of-the-art methods and shows comparable cross-dataset generalisation ability, even in comparison to the supervised approached.}
    \label{tab:partial}
\end{table}

\begin{figure}[ht!]
    \centering
    \includegraphics[width=\columnwidth]{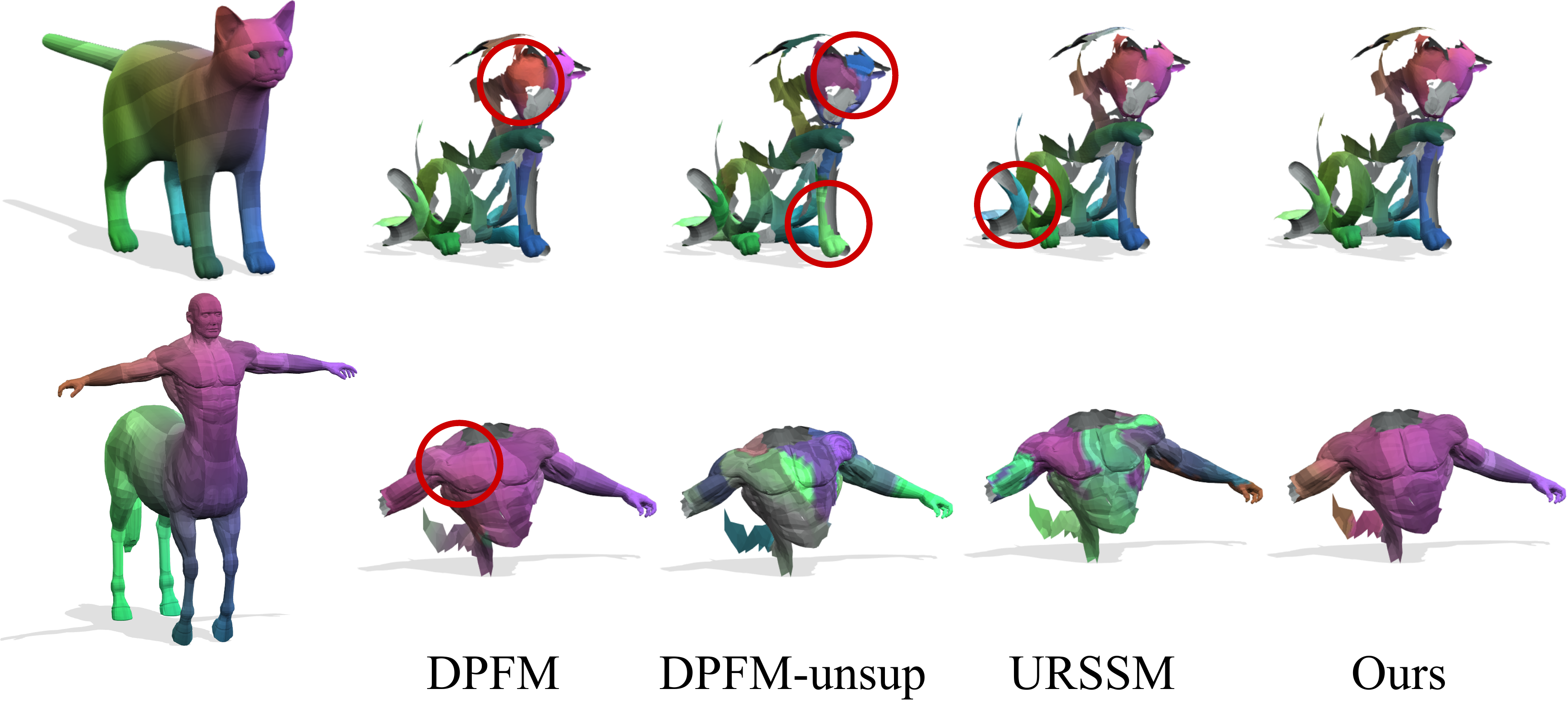}
    \caption{\textbf{Qualitative results on SHREC'16 dataset.} Compared to existing methods, our method is more robust to partiality.}
    \label{fig:shrec16}
\end{figure}

\noindent \textbf{Results.} We summarise the quantitative results on the SHREC'16 datasets in~\cref{tab:partial} and the corresponding PCK curve in~\cref{fig:topkids_shrec16_pck} (right). Compared to existing methods, our approach is more robust to partiality. We qualitatively compare our method to existing approaches in~\cref{fig:shrec16}.

\subsection{Analysis of self-adaptive functional map solver}
We summarise the learned parameters of the functional map solver for different kinds of datasets to better understand the learned regularisation strength and structure.~\cref{fig:fmap_reg} visualises the different regularisation strength (i.e.\ $\lambda$) and different regularisation structure (i.e.\ $\gamma$) for different datasets.

\begin{figure}[!ht]
    \begin{center}  \includegraphics[width=\columnwidth]{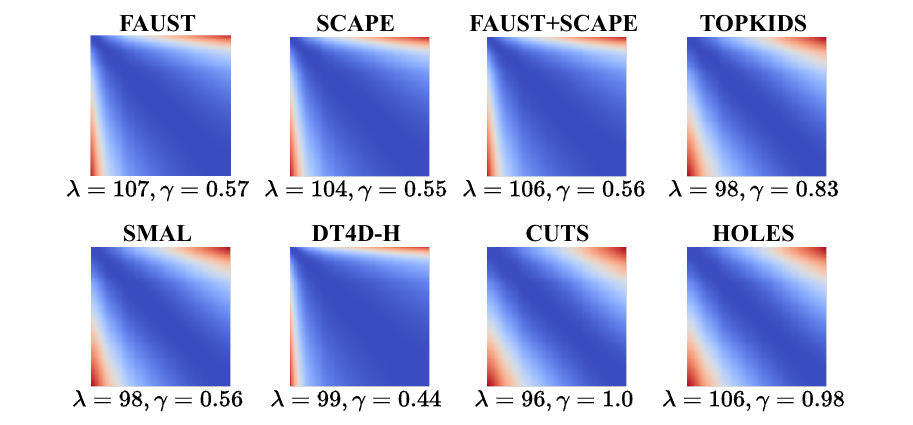}
    \caption{\textbf{Different regularisation strength and structure for different datasets.} The self-adaptive functional map solver enables to adjust the regularisation based on the training data.
    }
    \label{fig:fmap_reg}
    \vspace{-0.5cm}
    \end{center}
\end{figure}
We obverse that the regularisation strength (i.e.\ $\lambda$) for near-isometric shape matching (FAUST, SCAPE) is stronger than the strength for non-isometric shape matching (SMAL, DT4D-H), since in theory functional maps for isometric shape matching are diagonal matrices. In the context of regularisation structure, the funnel-like structure is narrower for topological noisy (TOPKIDS) and partial shapes (CUTS, HOLES).
\section{Limitation and future work}
\label{sec:limitation}
We build upon the existing state-of-the-art method~\cite{cao2023unsupervised} by introducing the self-adaptive functional map solver and the vertex-wise contrastive loss, and thereby achieve the new state of the art on a wide range of benchmark datasets. Yet, there are also some
limitations that give rise to interesting future researches.
Our unsupervised method is applicable in various settings. However, it can not be used for partial-to-partial shape matching. Therefore, it is interesting to investigate how to extend the current framework for partial-to-partial shape matching. 
{For functional map computation, we optimise the two parameters (i.e.\ $\gamma,\lambda$) that control the regularisation strength and structure. Meanwhile, the number of LBO eigenfunctions is also an important parameter for functional map computation. How to automatically select the best number of LBO eigenfunctions is thereby an another interesting future work direction.}   

\section{Conclusion}
\label{sec:conclusion}
We theoretically analyse the relationship between the functional map from the functional map solver and the functional map from the point-wise map. Based on our theoretical analysis, we extend the current state-of-the-art methods. We evaluate our proposed method on diverse shape matching benchmark datasets with different settings and demonstrate the new state-of-the-art performance. {We believe a more accurate and robust non-rigid 3D shape matching method would be beneficial for the shape analysis community to better explore the shape relationship.} 
\section{Acknowledgement}
This work was supported by the Visual Computing Incubator at the University of Bonn.

{
    \small
    \bibliographystyle{ieeenat_fullname}
    \bibliography{main}
}
\clearpage
\setcounter{page}{1}
\maketitlesupplementary
\section{Implementation details}
Our implementation is based on the official code\footnote{\url{https://github.com/dongliangcao/Unsupervised-Learning-of-Robust-Spectral-Shape-Matching}} from~\citet{cao2023unsupervised}. We use the DiffusionNet~\cite{sharp2020diffusionnet} as our feature extractor. The dimension of the output channels $F_{\mathcal{X}}$ is 256 (i.e.\ $c=256$) and the dimension of the LBO eigenfunctions $\Phi_{\mathcal{X}}$ is 200 (i.e.\ $k=200$). In the context of the functional map solver, we initialise the $\lambda=100$ in~\cref{eq:fmap} and the $\gamma=0.5$ in~\cref{eq:re} and~\cref{eq:img}. To compute the point-wise map $\Pi$ based on feature similarity, we use the row-wise softmax operator and set the $\tau=0.07$ in~\cref{eq:soft_corr}. To train the feature extractor and the functional map solver, we set $\lambda_{\mathrm{bij}} = 1.0, \lambda_{\mathrm{orth}} = 1.0$ in~\cref{eq:l_fmap} and $\lambda_{\mathrm{couple}} = 1.0, \lambda_{\mathrm{contrast}} = 10.0$ in~\cref{eq:l_total}, and use the Adam optimiser~\cite{kingma2015adam} with learning rate equal to $10^{-3}$. For inference, we follow the baseline~\cite{cao2023unsupervised} to use test-time-adaptation to refine the matching results.

\section{Qualitative results}
In this section, we show additional qualitative shape matching results of our method.

\begin{figure}[!ht]
    \centering
    \def\rowOnecolumnOne{3-6}
\def\rowOnecolumnTwo{3-8}
\def\rowOnecolumnThree{3-20}
\def\rowOnecolumnFour{3-21}
\def\rowOnecolumnFive{3-30}
\def\rowTwocolumnOne{10-8}
\def\rowTwocolumnTwo{10-12}
\def\rowTwocolumnThree{10-24}
\def\rowTwocolumnFour{10-32}
\def\rowTwocolumnFive{10-36}
\def\rowThreecolumnOne{25-31}
\def\rowThreecolumnTwo{25-4}
\def\rowThreecolumnThree{25-5}
\def\rowThreecolumnFour{25-34}
\def\rowThreecolumnFive{25-38}
\def\rowFourcolumnOne{41-9}
\def\rowFourcolumnTwo{41-25}
\def\rowFourcolumnThree{41-33}
\def\rowFourcolumnFour{41-35}
\def\rowFourcolumnFive{41-29}
\def\hspaceCols{-0.5cm}
\def\height{1.8cm}
\def\width{1.6cm}
\def\heightT{\height}
\def\widthT{\width}
\def\heightQ{\height}
\def\widthQ{\width}
\begin{tabular}{cccccc}%
        \setlength{\tabcolsep}{0pt} 
        \hspace{\hspaceCols}
        \includegraphics[height=\heightT, width=\widthT]{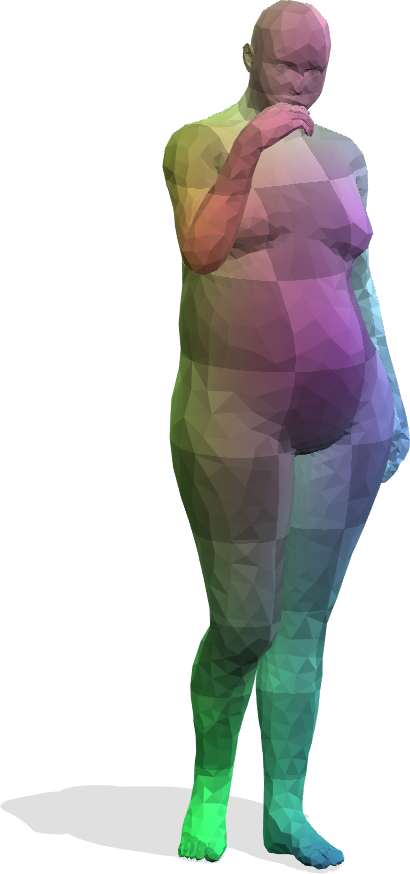}&
        \hspace{\hspaceCols}
        \includegraphics[height=\heightT, width=\widthT]{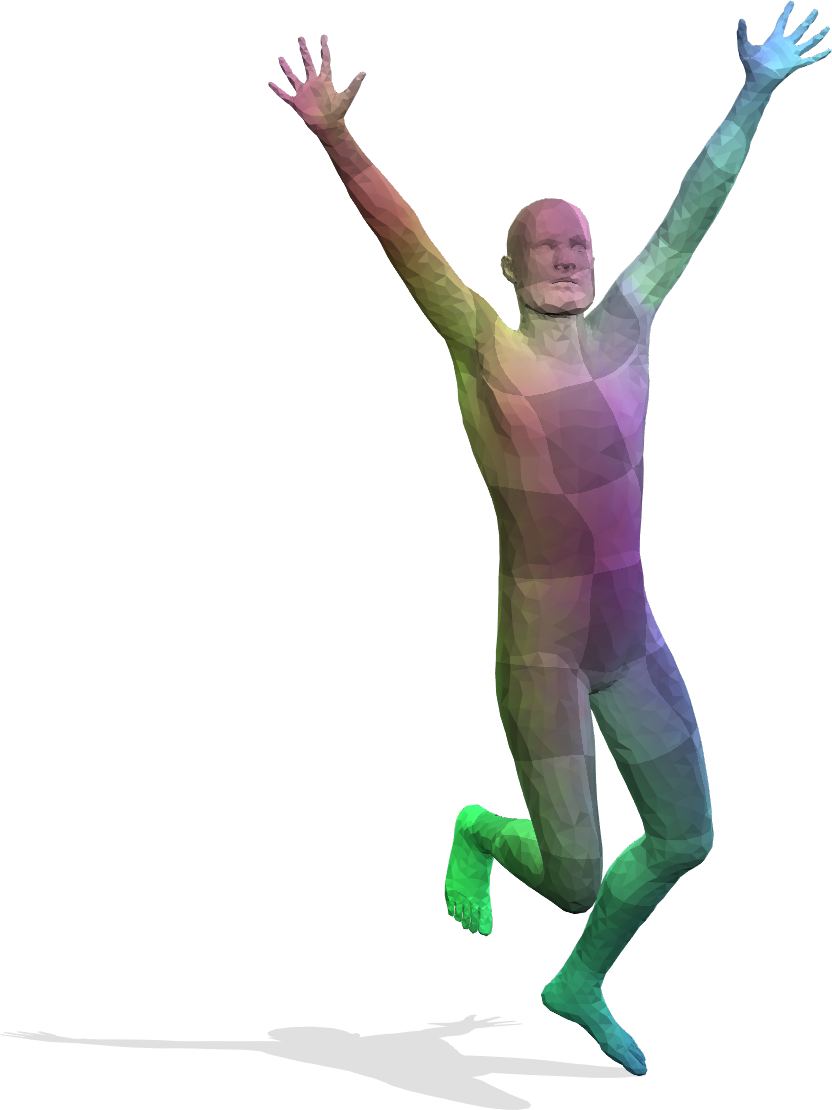}&
        \hspace{\hspaceCols}
        \includegraphics[height=\heightT, width=\widthT]{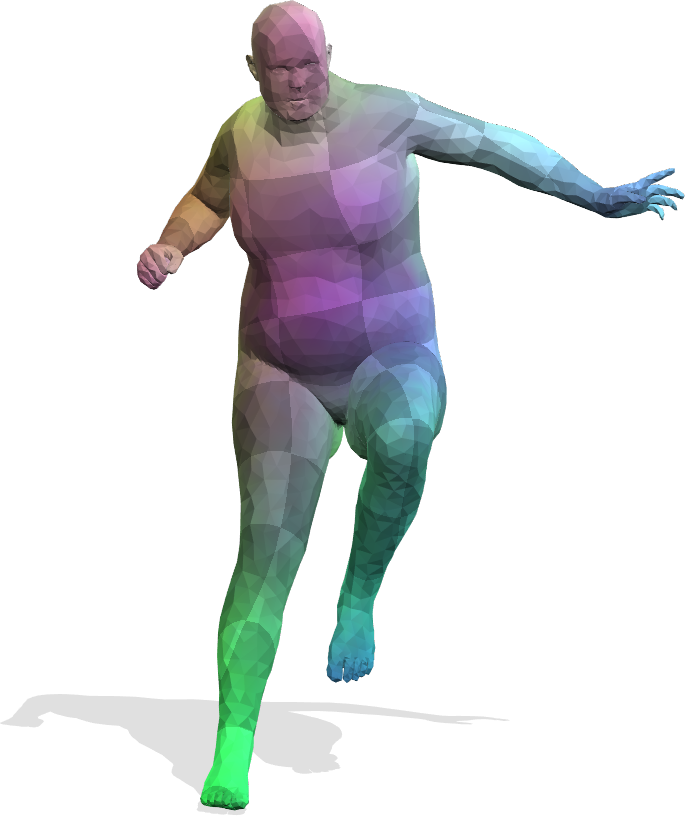}&
        \hspace{\hspaceCols}
        \includegraphics[height=\heightT, width=\widthT]{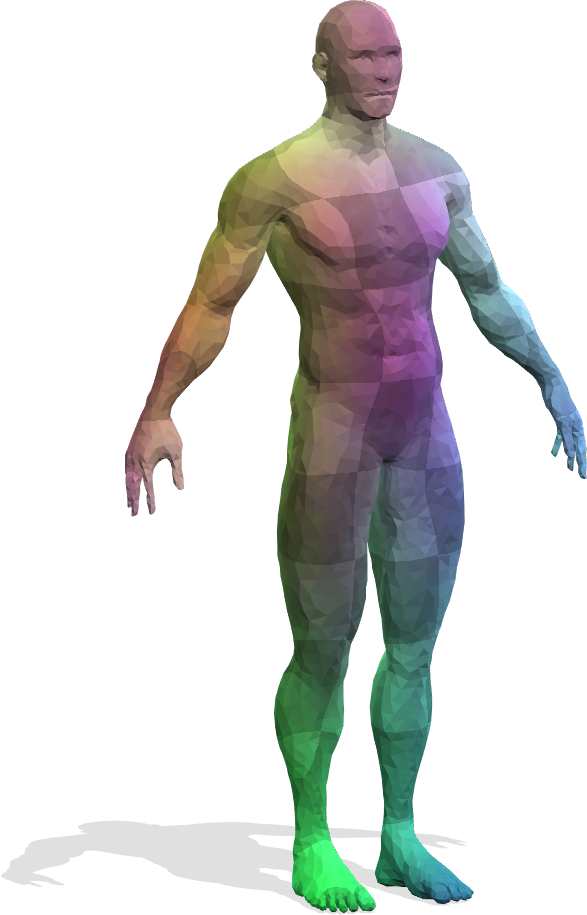}&
        \hspace{\hspaceCols}
        \includegraphics[height=\heightT, width=\widthT]{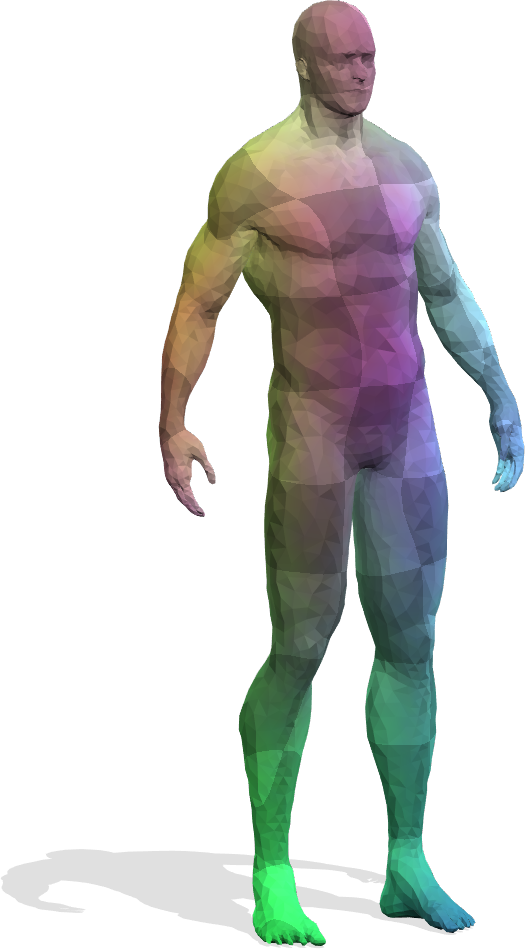}&
        \hspace{\hspaceCols}
        \includegraphics[height=\heightT, width=\widthT]{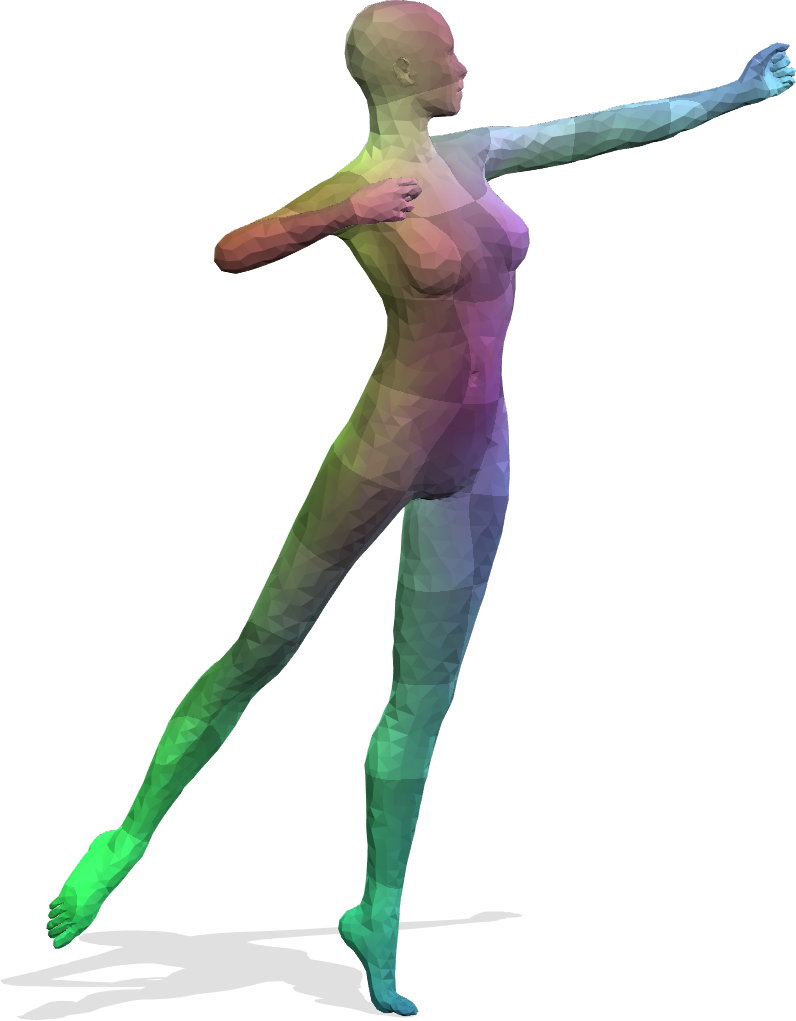}
        \\
        \hspace{\hspaceCols}
        \includegraphics[height=\heightT, width=\widthT]{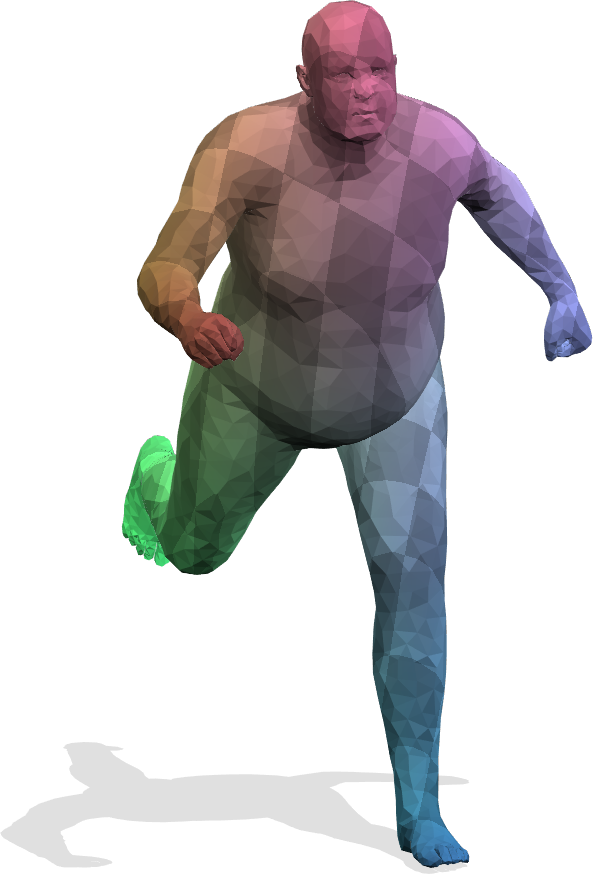}&
        \hspace{\hspaceCols}
        \includegraphics[height=\heightT, width=\widthT]{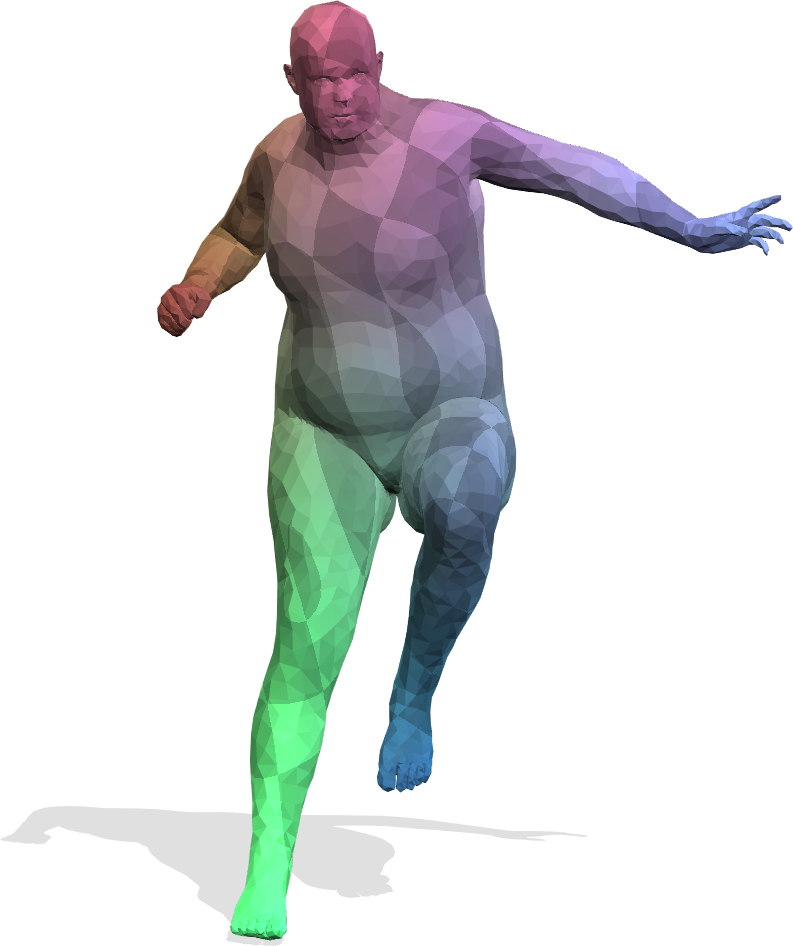}&
        \hspace{\hspaceCols}
        \includegraphics[height=\heightT, width=\widthT]{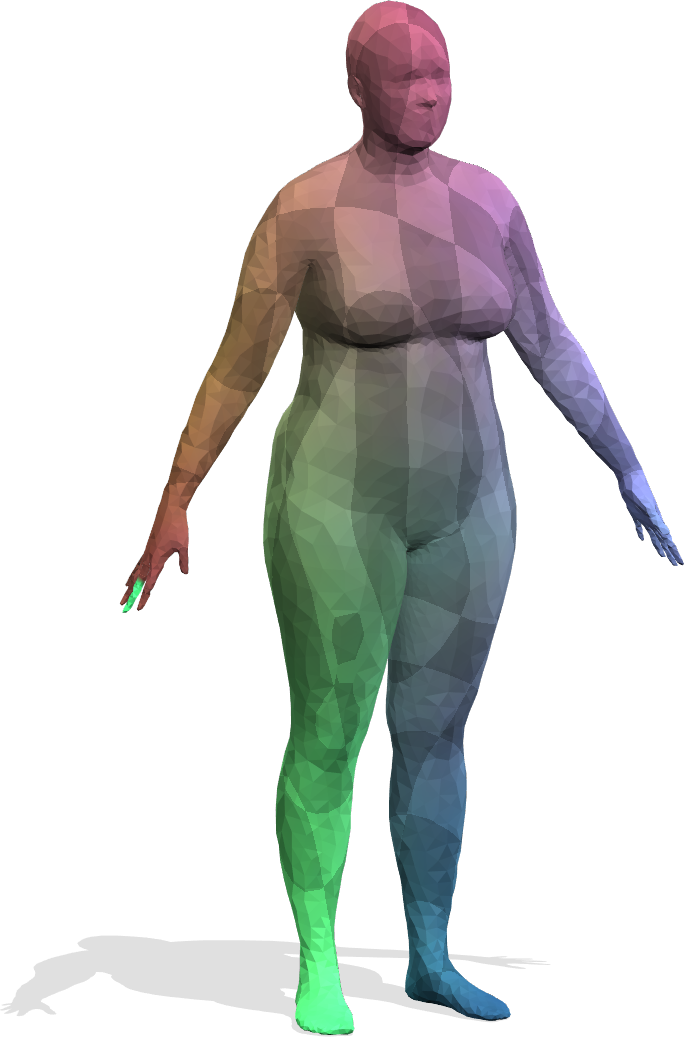}&
        \hspace{\hspaceCols}
        \includegraphics[height=\heightT, width=\widthT]{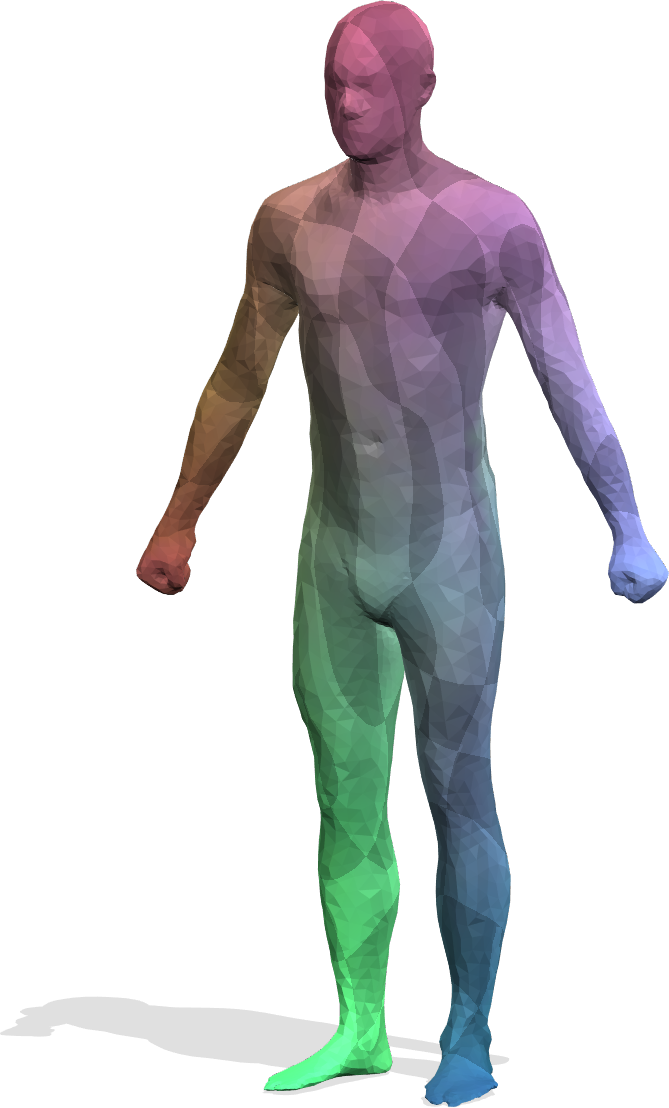}&
        \hspace{\hspaceCols}
        \includegraphics[height=\heightT, width=\widthT]{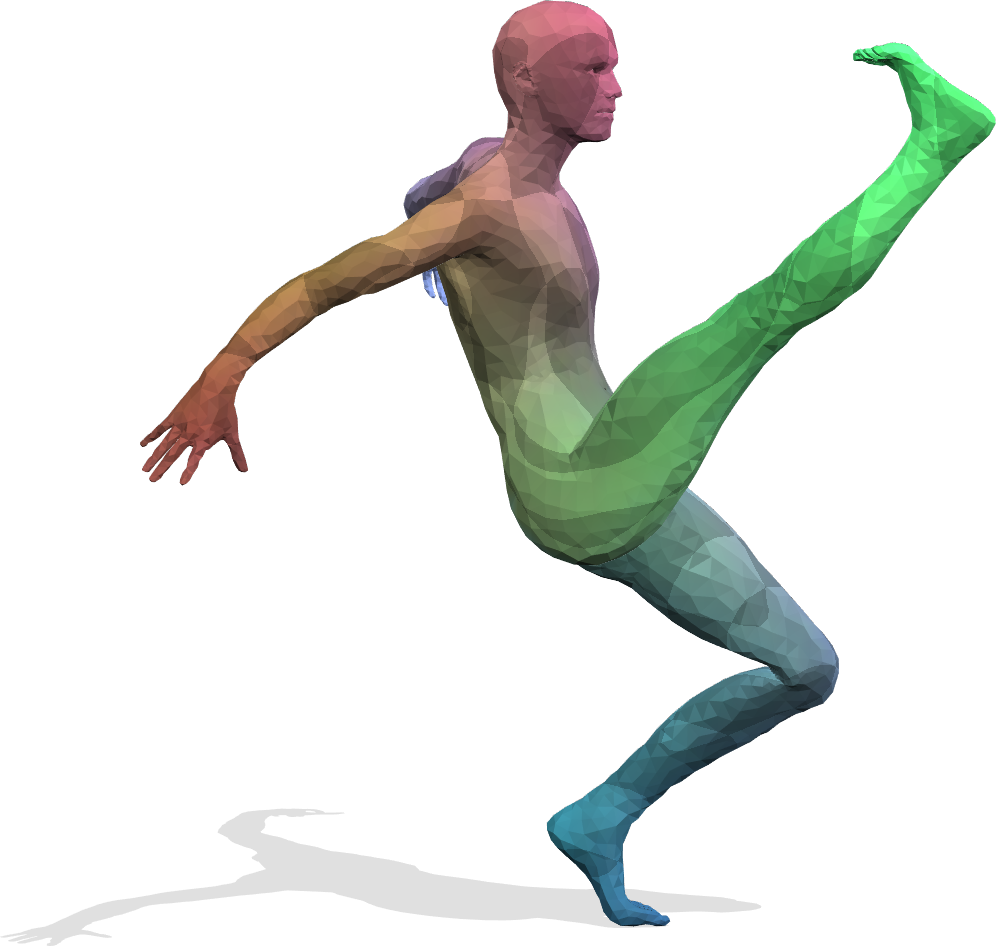}&
        \hspace{\hspaceCols}
        \includegraphics[height=\heightT, width=\widthT]{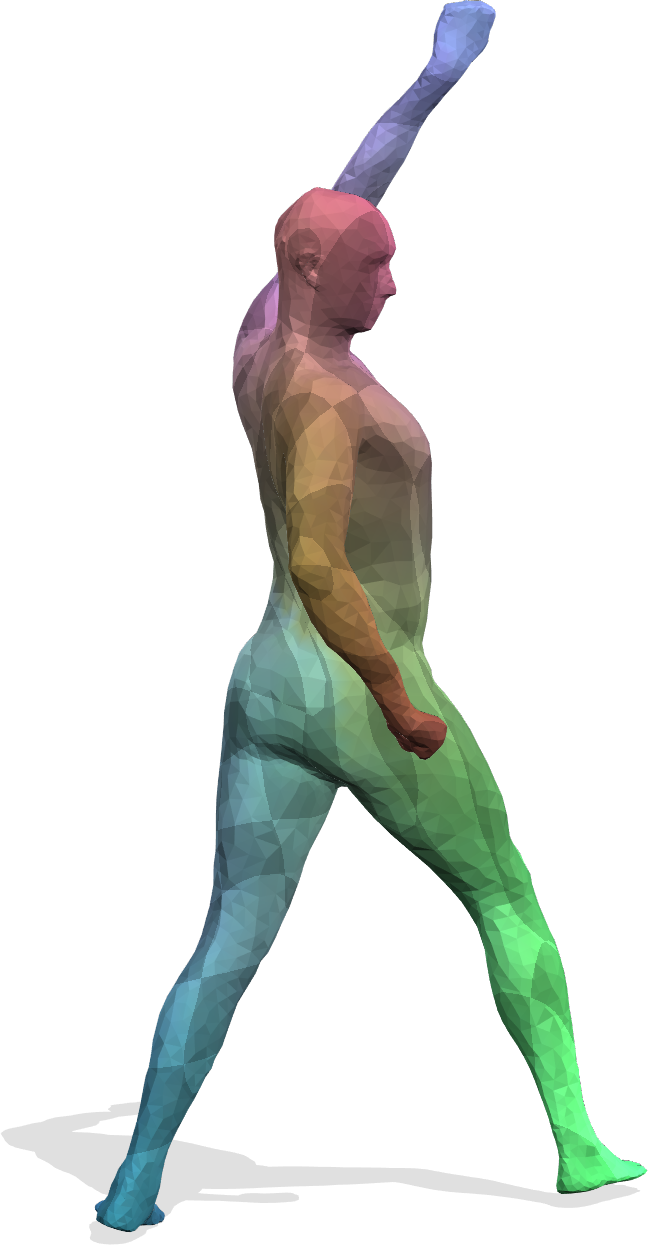} \\
        \hspace{\hspaceCols}
        \includegraphics[height=\heightT, width=\widthT]{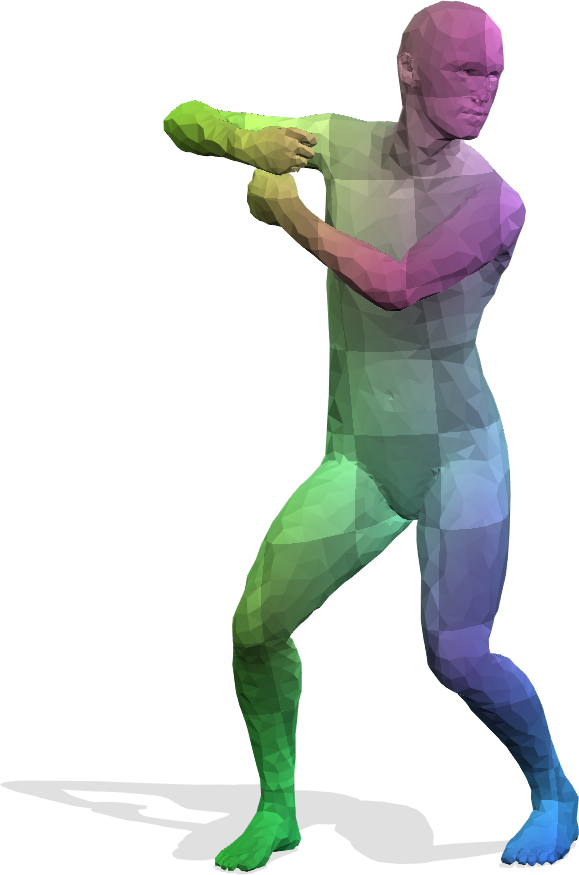}&
        \hspace{\hspaceCols}
        \includegraphics[height=\heightT, width=\widthT]{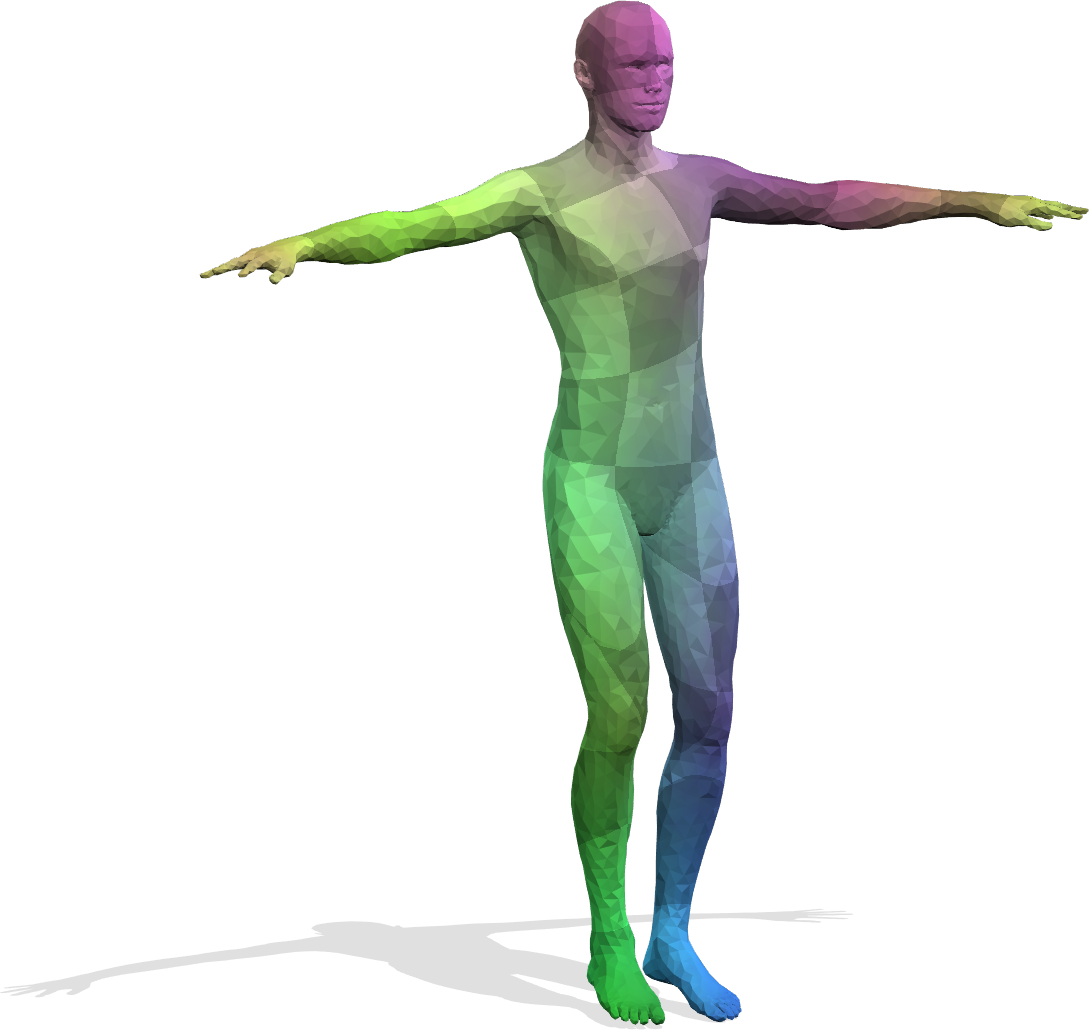}&
        \hspace{\hspaceCols}
        \includegraphics[height=\heightT, width=\widthT]{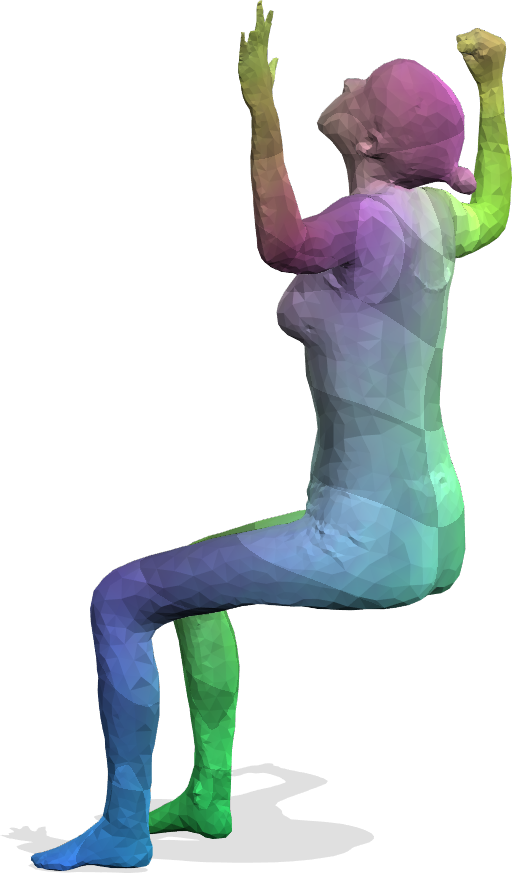}&
        \hspace{\hspaceCols}
        \includegraphics[height=\heightT, width=\widthT]{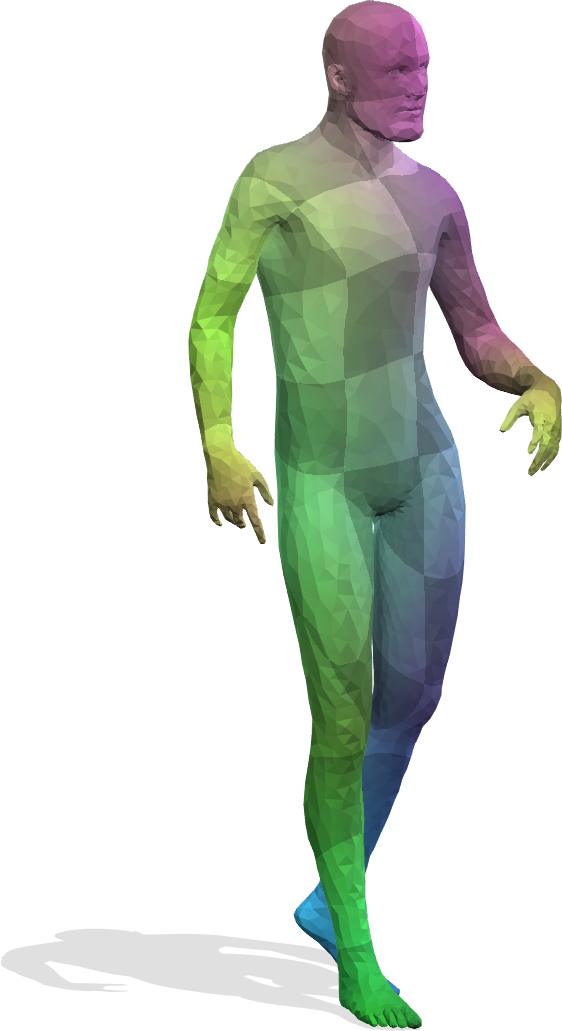}&
        \hspace{\hspaceCols}
        \includegraphics[height=\heightT, width=\widthT]{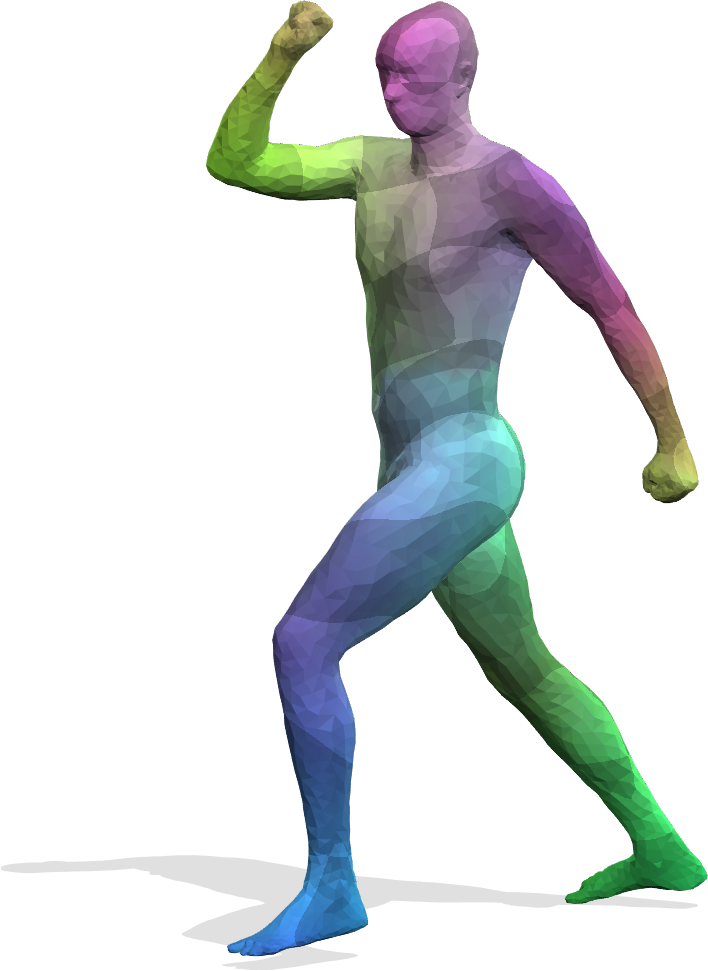}&
        \hspace{\hspaceCols}
        \includegraphics[height=\heightT, width=\widthT]{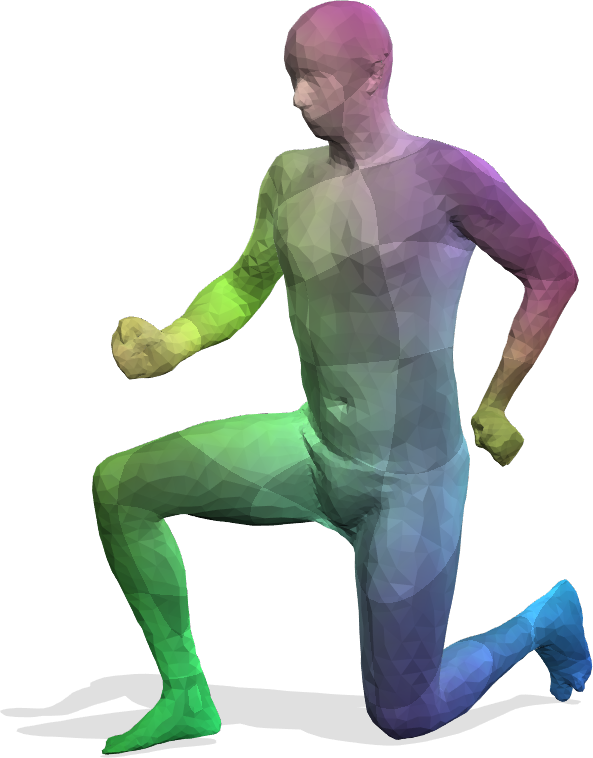} \\
        \hspace{\hspaceCols}
        \includegraphics[height=\heightT, width=\widthT]{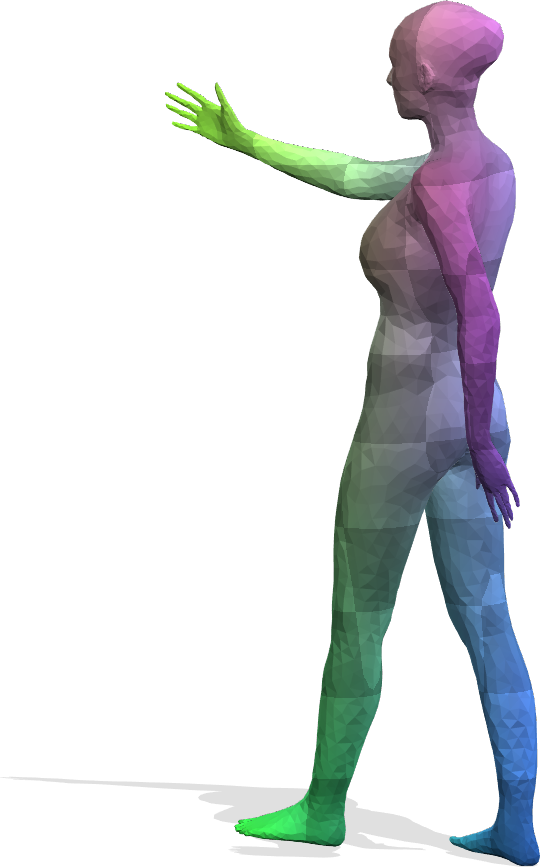}&
        \hspace{\hspaceCols}
        \includegraphics[height=\heightT, width=\widthT]{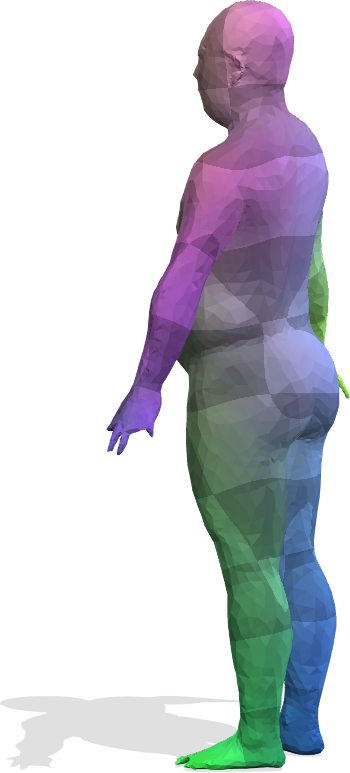}&
        \hspace{\hspaceCols}
        \includegraphics[height=\heightT, width=\widthT]{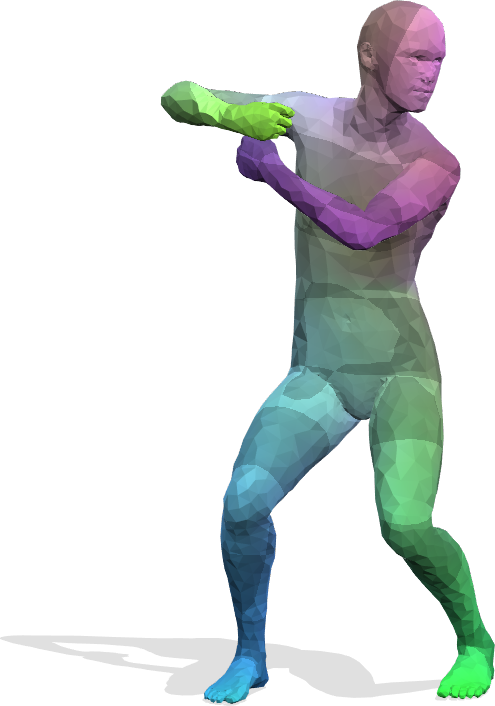}&
        \hspace{\hspaceCols}
        \includegraphics[height=\heightT, width=\widthT]{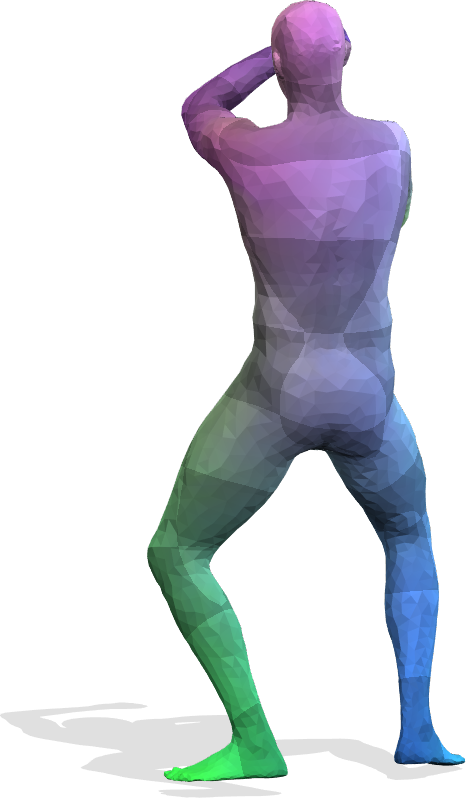}&
        \hspace{\hspaceCols}
        \includegraphics[height=\heightT, width=\widthT]{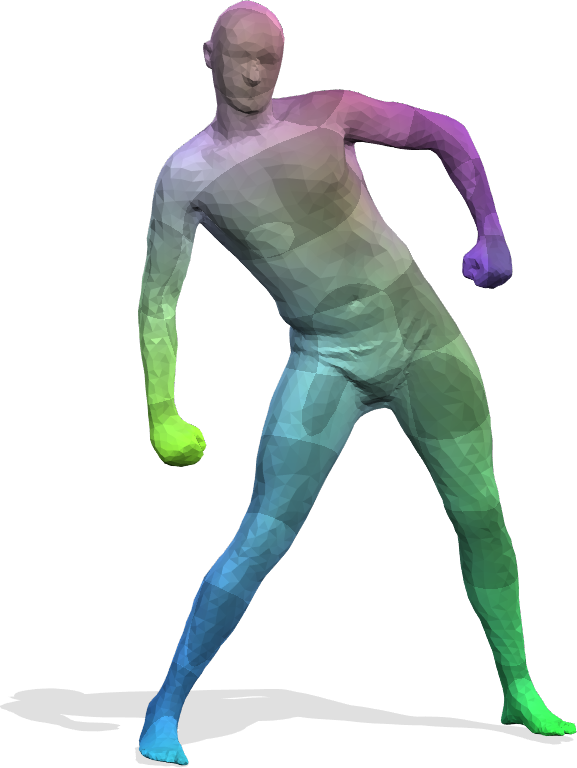}&
        \hspace{\hspaceCols}
        \includegraphics[height=\heightT, width=\widthT]{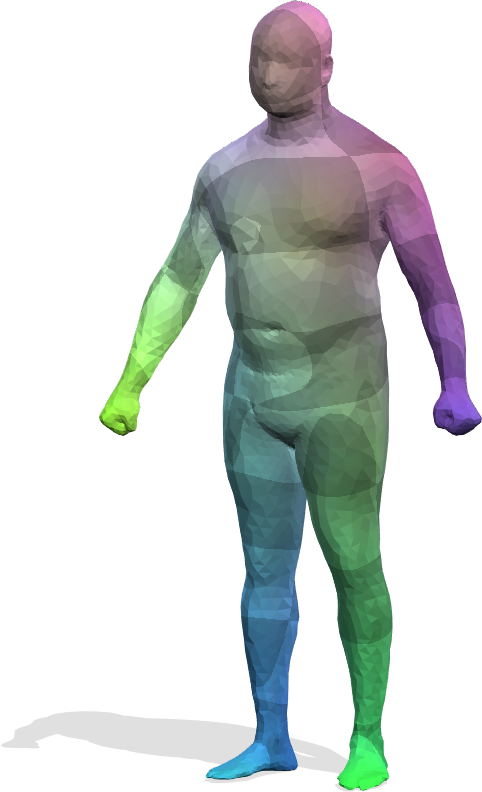}
    \end{tabular}
    \caption{\textbf{Qualitative results of our method on the SHREC'19 dataset.} The leftmost shape on each row is the reference shape to be matched by other shapes. Our method obtains accurate matchings for human shapes with diverse poses and appearances.}
    \label{fig:shrec19_qualitative}
\end{figure}

\begin{figure}[!ht]
    \centering
    \def\rowOnecolumnOne{kid00-kid01}
\def\rowOnecolumnTwo{kid00-kid02}
\def\rowOnecolumnThree{kid00-kid03}
\def\rowOnecolumnFour{kid00-kid04}
\def\rowOnecolumnFive{kid00-kid05}
\def\rowOnecolumnSix{kid00-kid24}
\def\rowTwocolumnOne{kid00-kid06}
\def\rowTwocolumnTwo{kid00-kid07}
\def\rowTwocolumnThree{kid00-kid08}
\def\rowTwocolumnFour{kid00-kid09}
\def\rowTwocolumnFive{kid00-kid10}
\def\rowTwocolumnSix{kid00-kid16}
\def\hspaceCols{-0.6cm}
\def\height{1.7cm}
\def\width{1.75cm}
\def\heightT{\height}
\def\widthT{\width}
\def\heightQ{\height}
\def\widthQ{\width}
\begin{tabular}{cccccc}%
        \setlength{\tabcolsep}{0pt} 
        \hspace{\hspaceCols}
        \includegraphics[height=\heightT, width=\widthT]{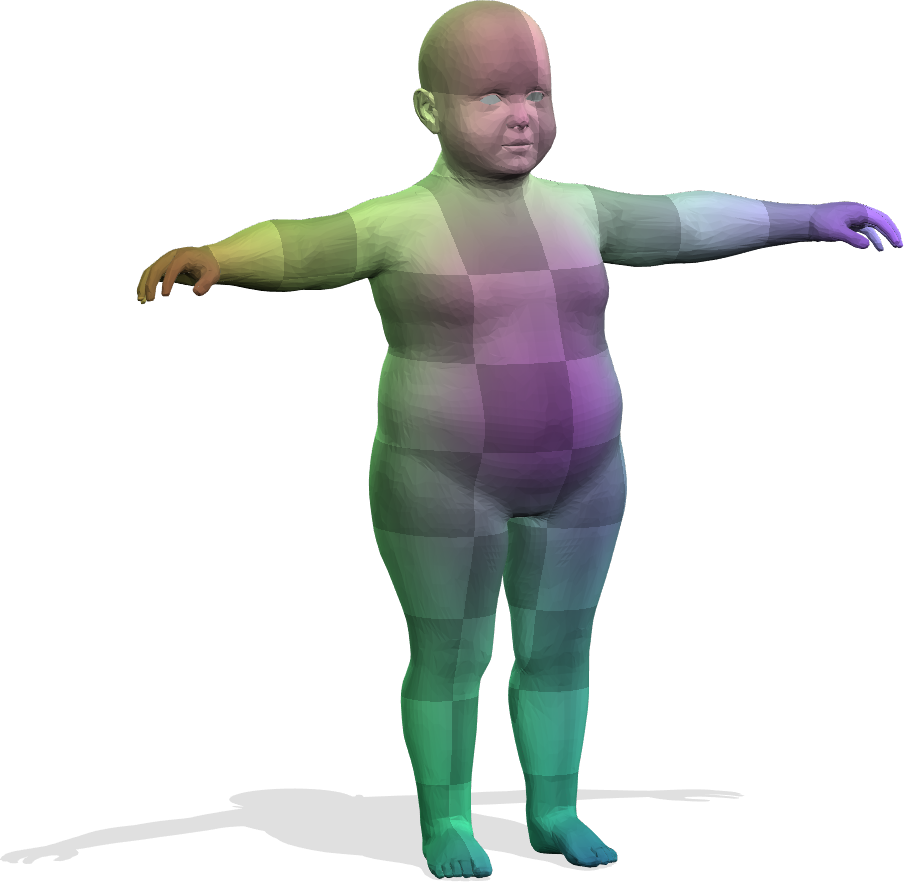}&
        \hspace{\hspaceCols}
        \includegraphics[height=\heightT, width=\widthT]{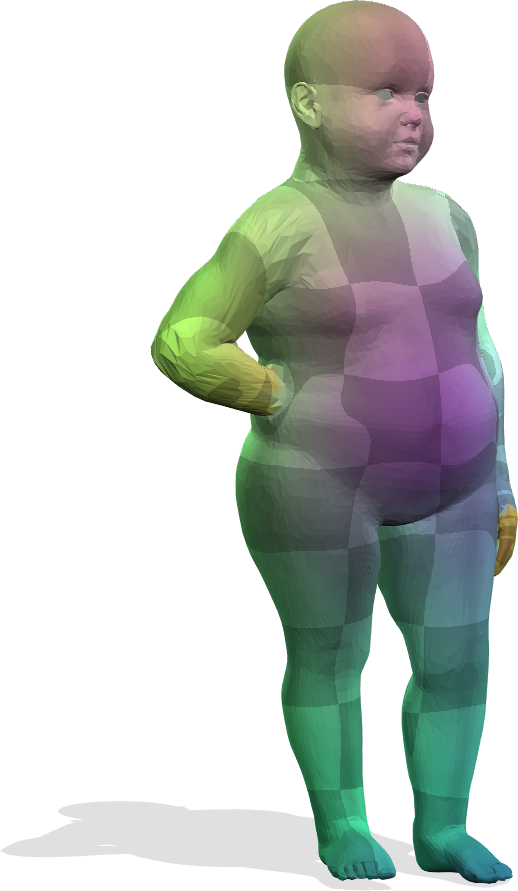}&
        \hspace{\hspaceCols}
        \includegraphics[height=\heightT, width=\widthT]{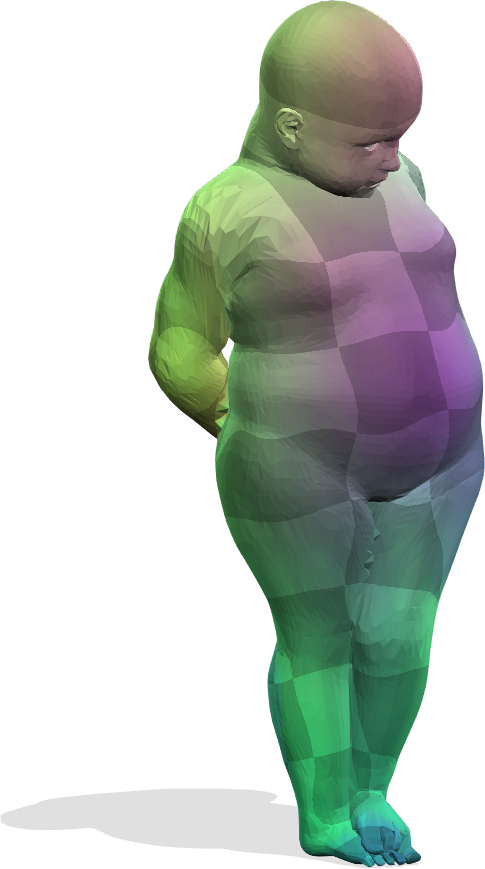}&
        \hspace{\hspaceCols}
        \includegraphics[height=\heightT, width=\widthT]{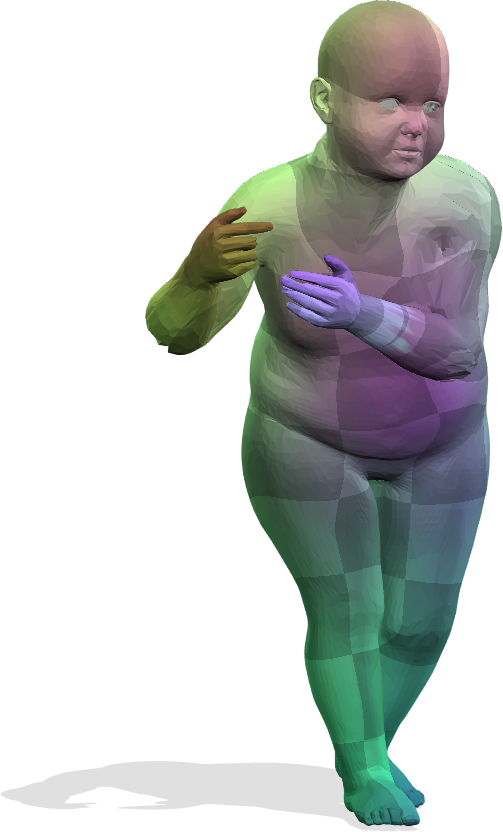}&
        \hspace{\hspaceCols}
        \includegraphics[height=\heightT, width=\widthT]{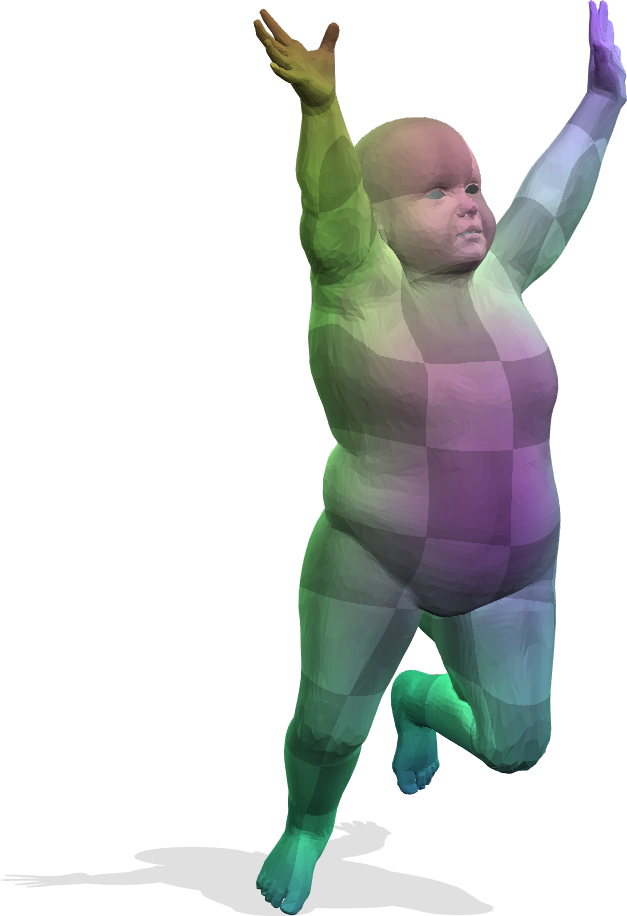}&
        \hspace{\hspaceCols}
        \includegraphics[height=\heightT, width=\widthT]{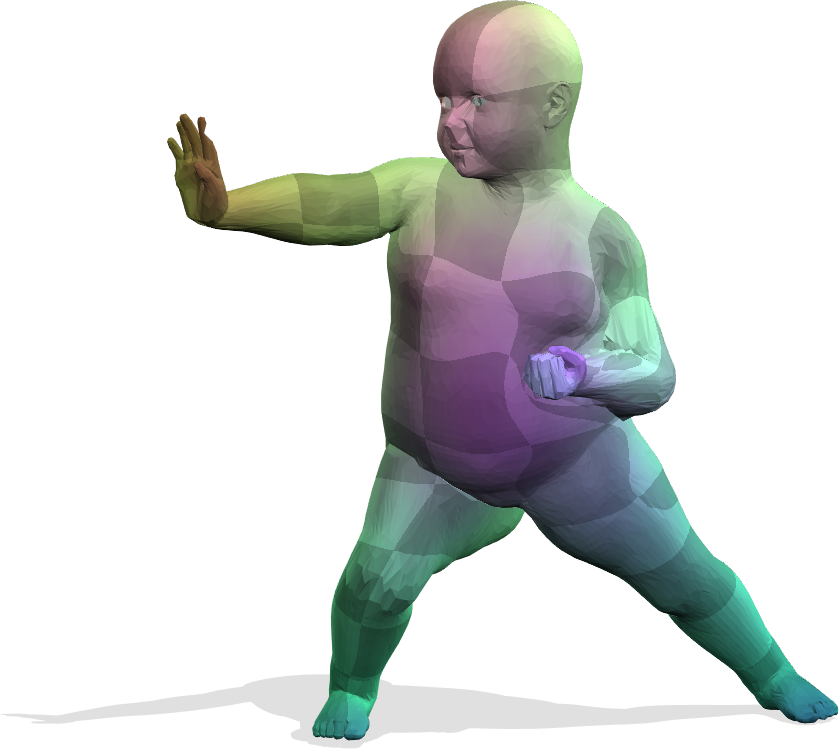}
        \\
        \hspace{\hspaceCols}
        \includegraphics[height=\heightT, width=\widthT]{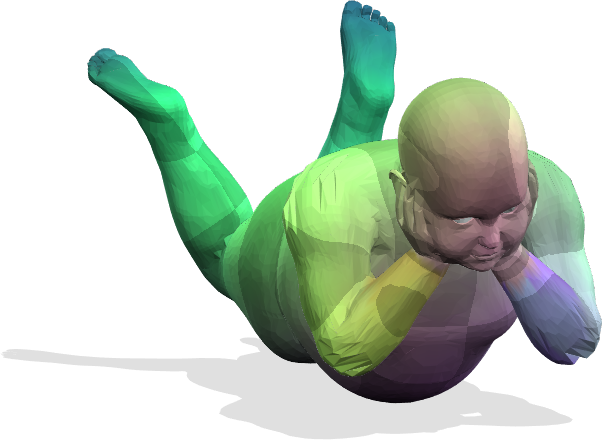}&
        \hspace{\hspaceCols}
        \includegraphics[height=\heightT, width=\widthT]{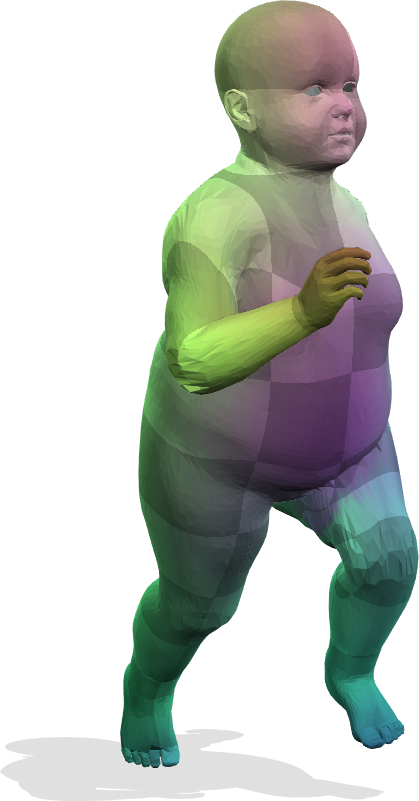}&
        \hspace{\hspaceCols}
        \includegraphics[height=\heightT, width=\widthT]{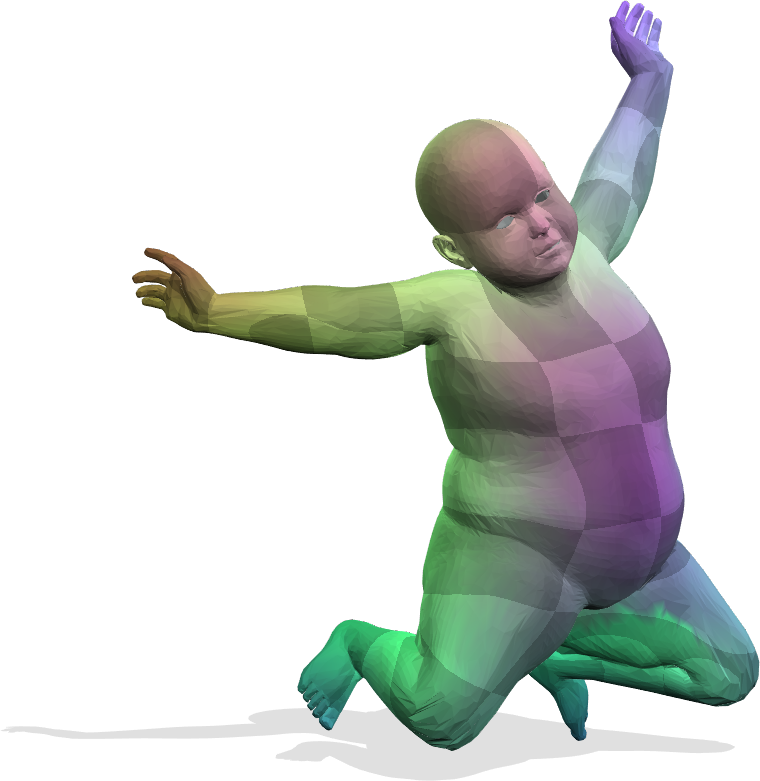}&
        \hspace{\hspaceCols}
        \includegraphics[height=\heightT, width=\widthT]{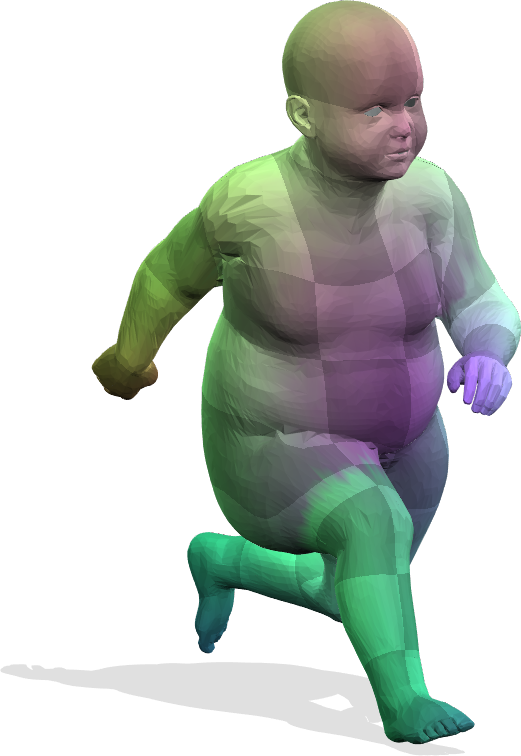}&
        \hspace{\hspaceCols}
        \includegraphics[height=\heightT, width=\widthT]{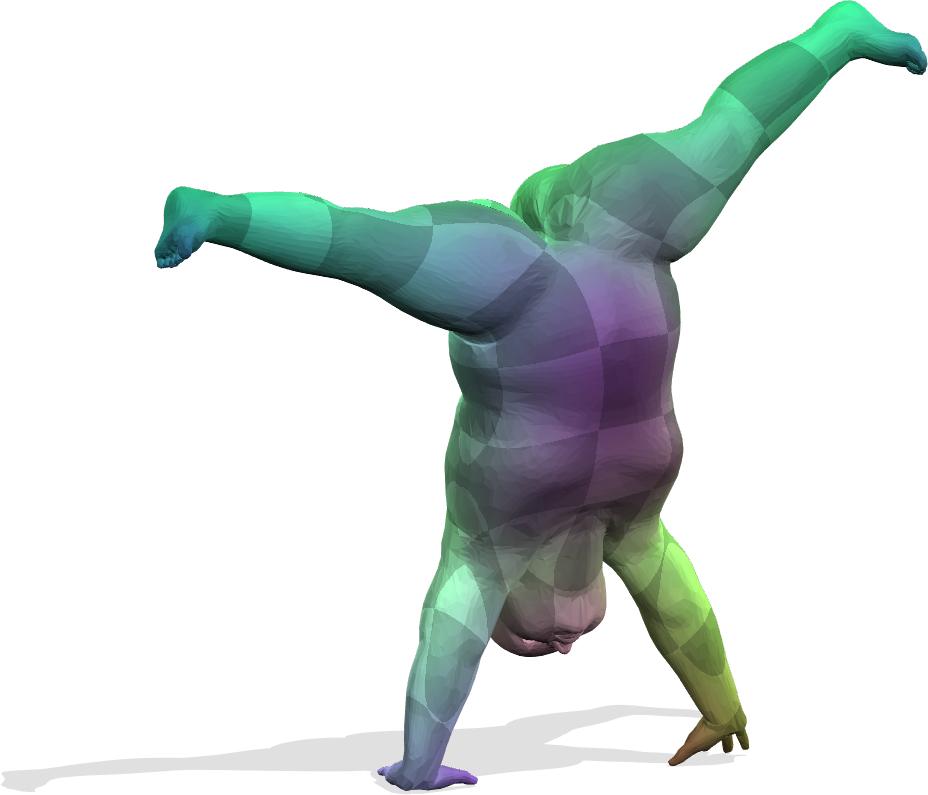}&
        \hspace{\hspaceCols}
        \includegraphics[height=\heightT, width=\widthT]{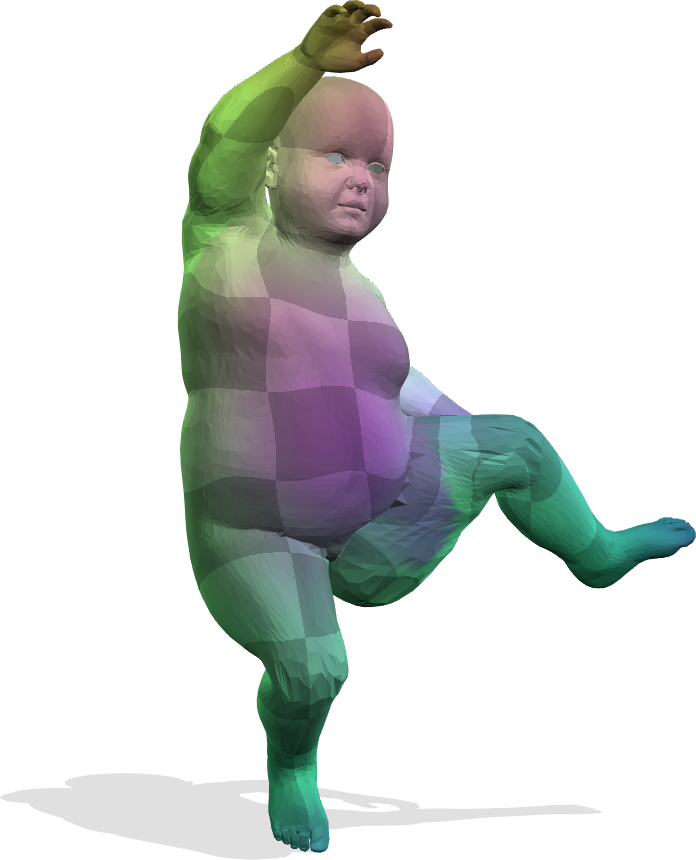}
    \end{tabular}
    \caption{\textbf{Qualitative results of our method on the TOPKIDS dataset.} The top-left shape is the reference shape to be matched by other shapes. Our method is robust against topological noise.}
    \label{fig:topkids_qualitative}
     \vspace{-0.5cm}
\end{figure}

\begin{figure}[!ht]
    \centering
    \def\rowOnecolumnOne{Standing2HMagicAttack01034-BrooklynUprock134}
\def\rowOnecolumnTwo{Standing2HMagicAttack01034-BrooklynUprock134}
\def\rowOnecolumnThree{Standing2HMagicAttack01034-DancingRunningMan259}
\def\rowOnecolumnFour{Standing2HMagicAttack01034-BrooklynUprock146}
\def\rowOnecolumnFive{Standing2HMagicAttack01034-GoalkeeperScoop046}
\def\rowOnecolumnSix{Standing2HMagicAttack01034-Floating099}
\def\rowTwocolumnOne{Standing2HMagicAttack01034-Falling271}
\def\rowTwocolumnTwo{Standing2HMagicAttack01034-GoalkeeperScoop065}
\def\rowTwocolumnThree{Standing2HMagicAttack01034-InvertedDoubleKickToKipUp189}
\def\rowTwocolumnFour{Standing2HMagicAttack01034-Falling237}
\def\rowTwocolumnFive{Standing2HMagicAttack01034-InvertedDoubleKickToKipUp255}
\def\rowTwocolumnSix{Standing2HMagicAttack01034-KettlebellSwing057}
\def\hspaceCols{-0.75cm}
\def\height{1.7cm}
\def\width{1.75cm}
\def\heightT{\height}
\def\widthT{\width}
\def\heightQ{\height}
\def\widthQ{\width}
\begin{tabular}{cccccc}%
        \setlength{\tabcolsep}{0pt} 
        \hspace{\hspaceCols}
        \includegraphics[height=\heightT, width=\widthT]{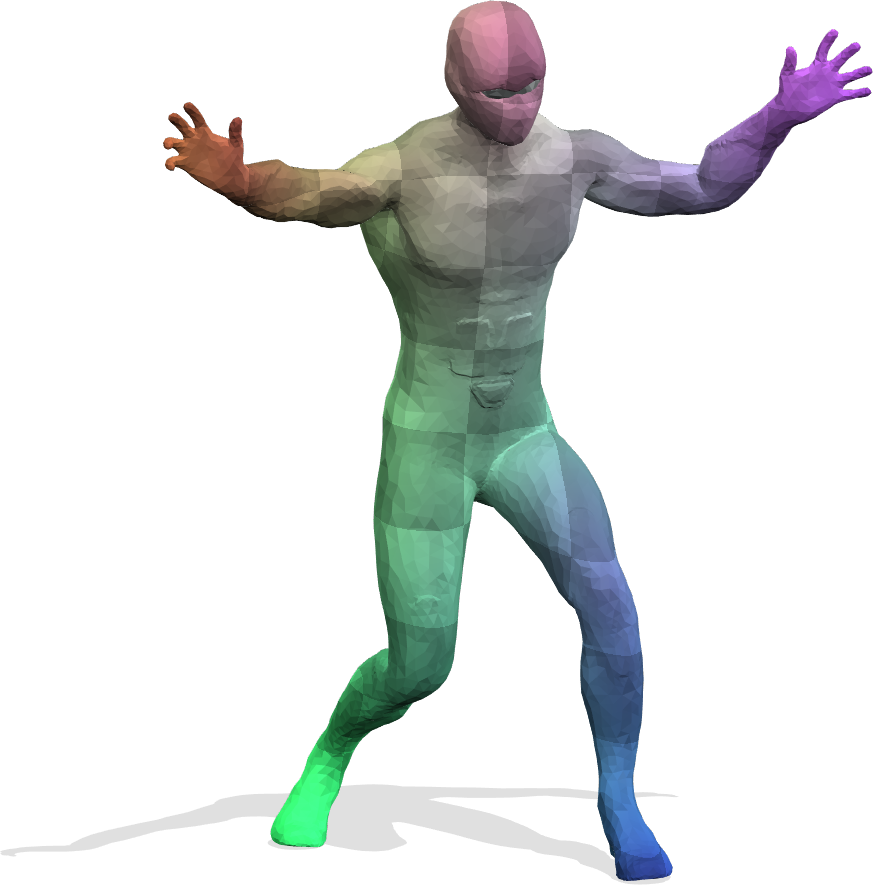}&
        \hspace{\hspaceCols}
        \includegraphics[height=\heightT, width=\widthT]{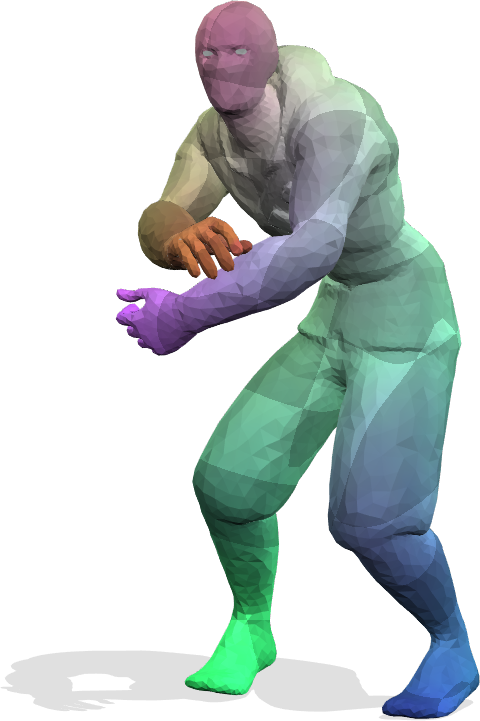}&
        \hspace{\hspaceCols}
        \includegraphics[height=\heightT, width=\widthT]{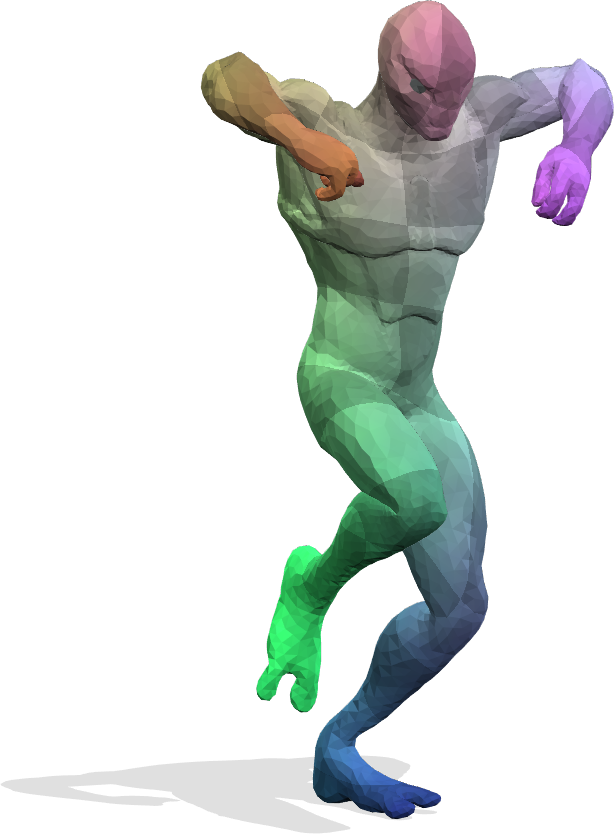}&
        \hspace{\hspaceCols}
        \includegraphics[height=\heightT, width=\widthT]{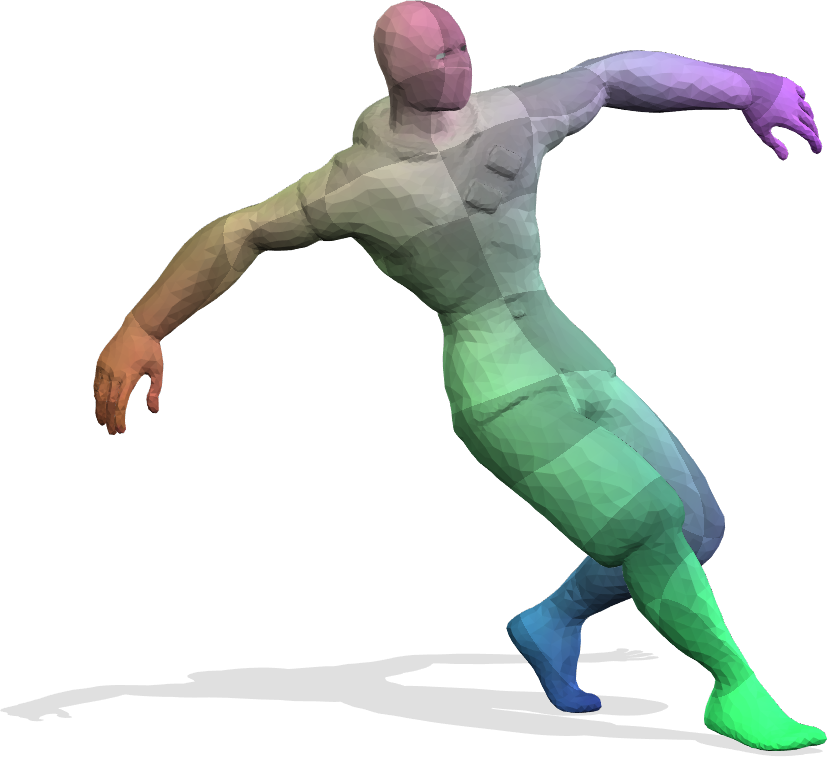}&
        \hspace{\hspaceCols}
        \includegraphics[height=\heightT, width=\widthT]{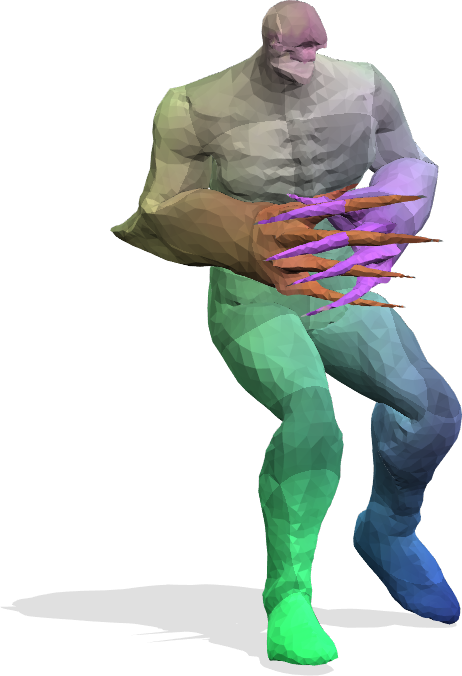}&
        \hspace{\hspaceCols}
        \includegraphics[height=\heightT, width=\widthT]{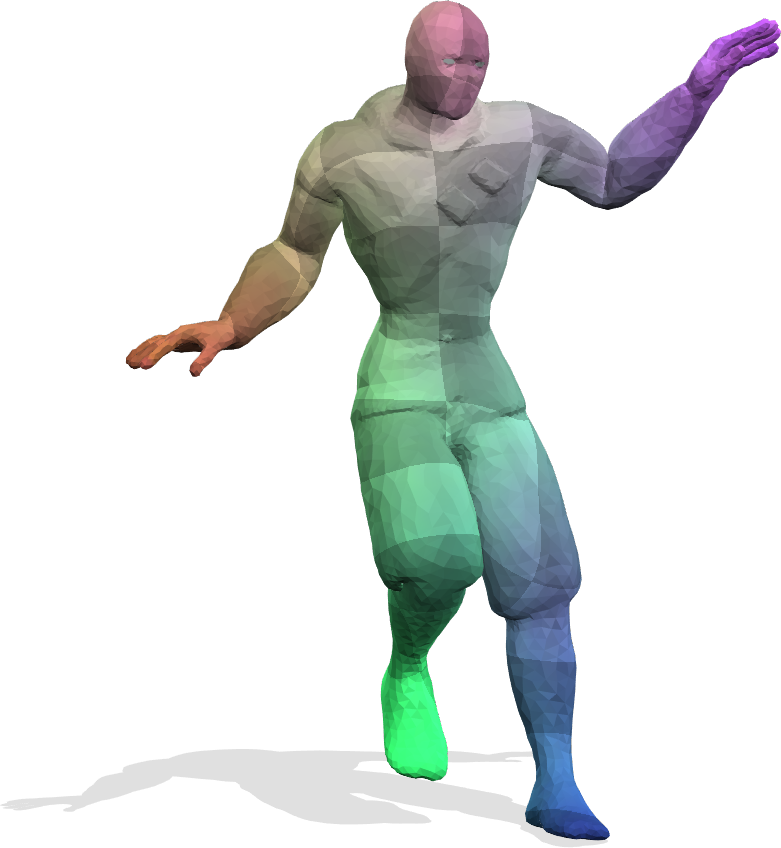}
        \\
        \hspace{\hspaceCols}
        \includegraphics[height=\heightT, width=\widthT]{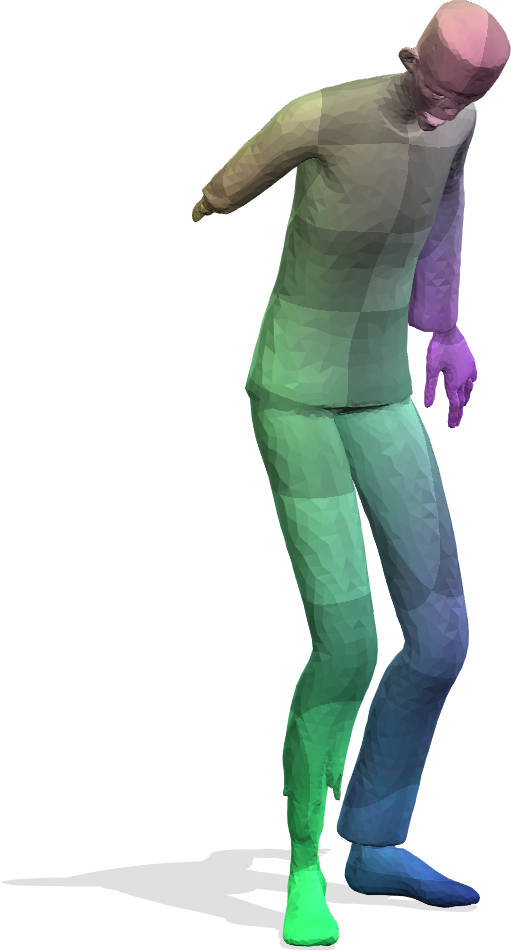}&
        \hspace{\hspaceCols}
        \includegraphics[height=\heightT, width=\widthT]{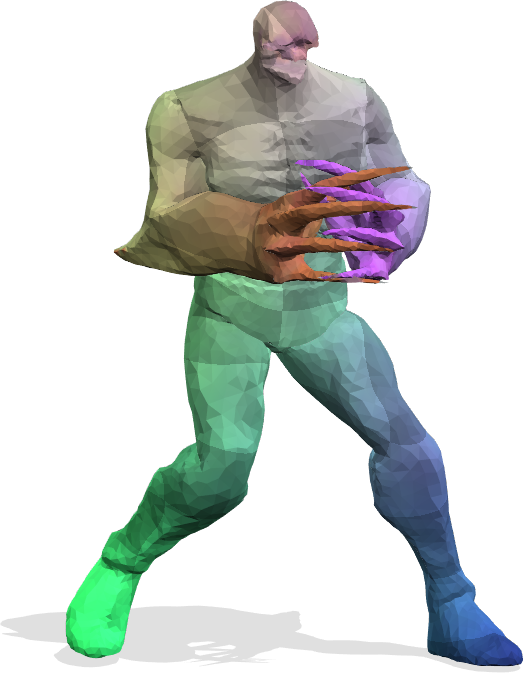}&
        \hspace{\hspaceCols}
        \includegraphics[height=\heightT, width=\widthT]{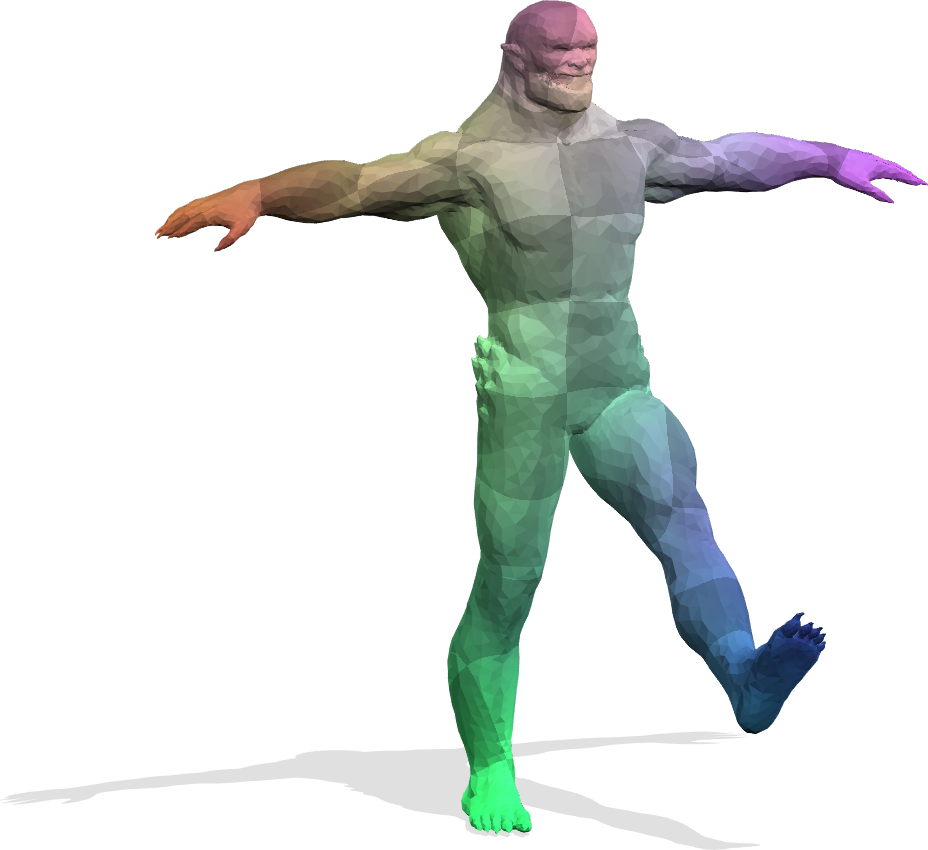}&
        \hspace{\hspaceCols}
        \includegraphics[height=\heightT, width=\widthT]{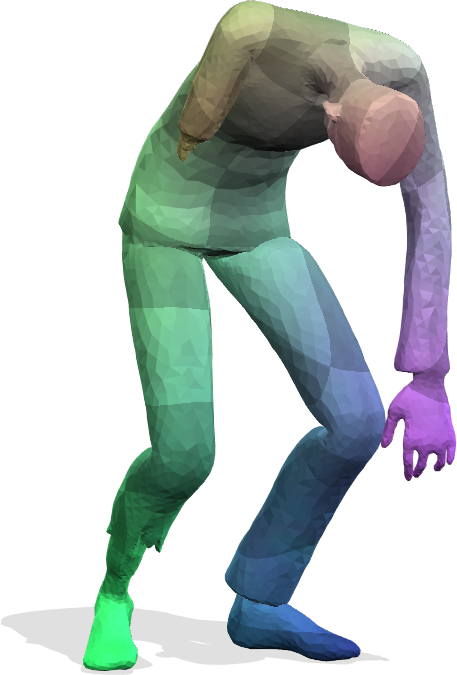}&
        \hspace{\hspaceCols}
        \includegraphics[height=\heightT, width=\widthT]{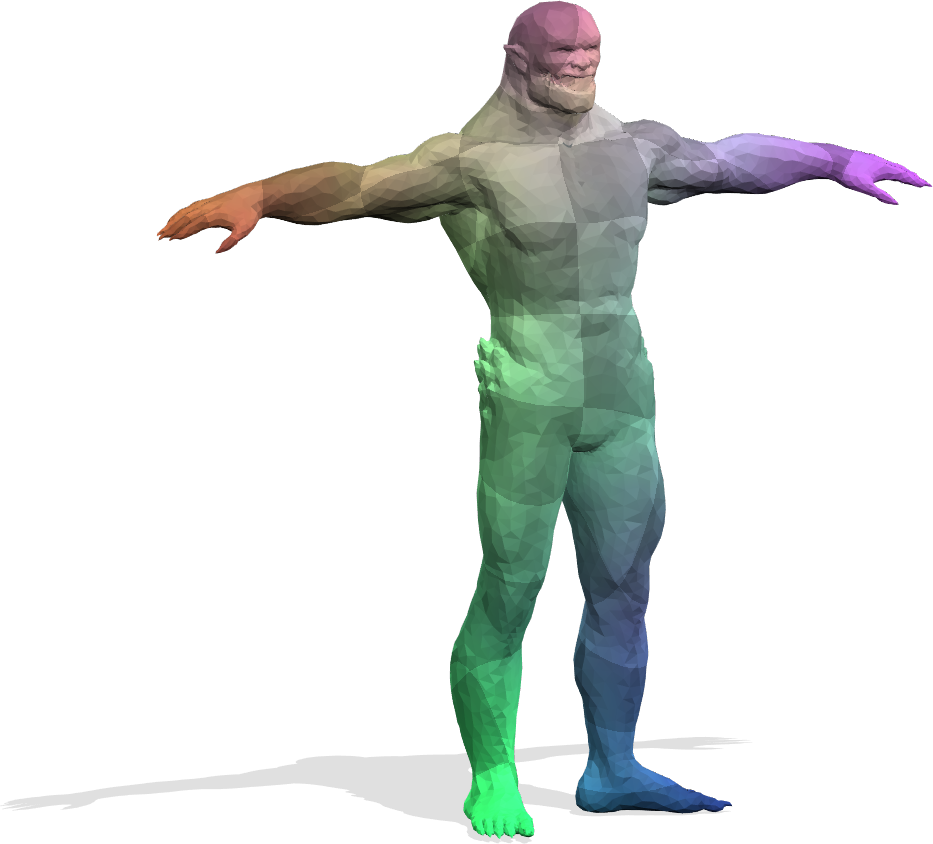}&
        \hspace{\hspaceCols}
        \includegraphics[height=\heightT, width=\widthT]{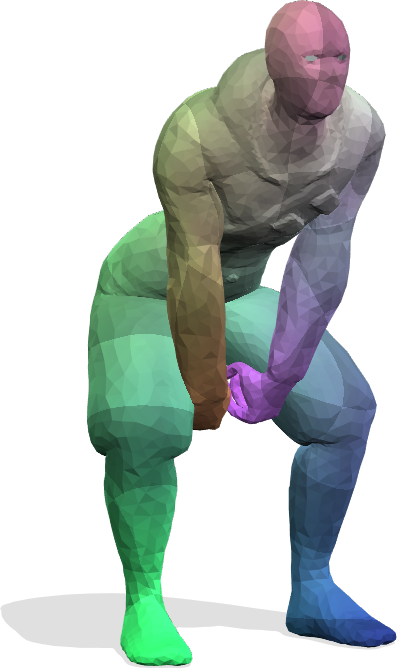}
    \end{tabular}
    \caption{\textbf{Qualitative results of our method on the DT4D-H dataset.} The top-left shape is the reference shape to be matched by other shapes. Our method obtains accurate correspondences for non-isometric deformed shapes.}
    \label{fig:dt4d_qualitative}
    \vspace{-0.5cm}
\end{figure}

\begin{figure}[!ht]
    \centering
    \def\rowOnecolumnOne{cougar_01-hippo_01}
\def\rowOnecolumnTwo{cougar_01-hippo_01}
\def\rowOnecolumnThree{cougar_01-hippo_02}
\def\rowOnecolumnFour{cougar_01-horse_02}
\def\rowTwocolumnFive{cougar_01-hippo_04}
\def\rowTwocolumnOne{cougar_01-hippo_05}
\def\rowTwocolumnTwo{cougar_01-hippo_06}
\def\rowTwocolumnThree{cougar_01-horse_01}
\def\rowTwocolumnFour{cougar_01-horse_03}
\def\rowOnecolumnFive{cougar_01-horse_05}
\def\rowThreecolumnOne{cougar_01-horse_06}
\def\rowThreecolumnTwo{cougar_01-horse_07}
\def\rowThreecolumnThree{cougar_01-horse_08}
\def\rowThreecolumnFour{cougar_01-cougar_01}
\def\rowThreecolumnFive{cougar_01-cougar_04}
\def\hspaceCols{-0.45cm}
\def\height{1.5cm}
\def\width{1.7cm}
\def\heightT{\height}
\def\widthT{\width}
\def\heightQ{\height}
\def\widthQ{\width}
\begin{tabular}{ccccc}%
        \setlength{\tabcolsep}{0pt} 
        \hspace{\hspaceCols}
        \includegraphics[height=\heightT, width=\widthT]{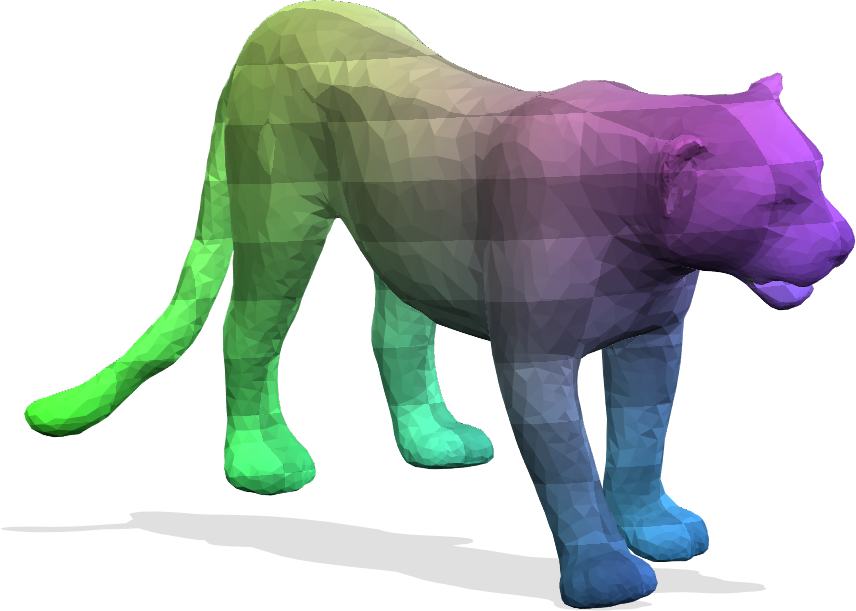}&
        \hspace{\hspaceCols}
        \includegraphics[height=\heightT, width=\widthT]{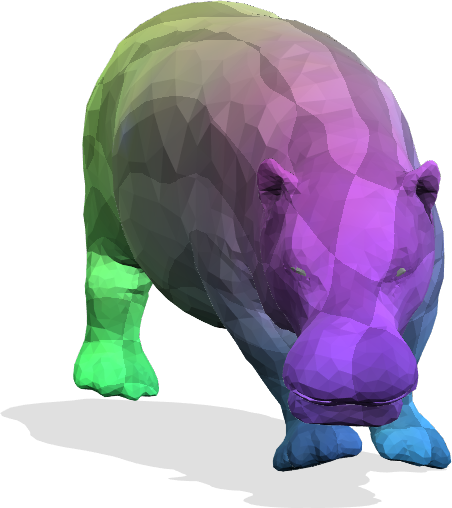}&
        \hspace{\hspaceCols}
        \includegraphics[height=\heightT, width=\widthT]{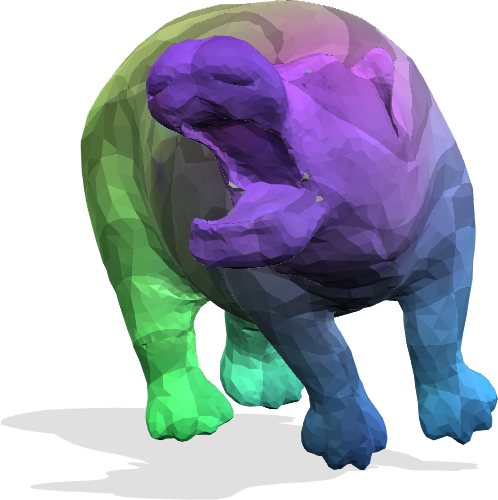}&
        \hspace{\hspaceCols}
        \includegraphics[height=\heightT, width=\widthT]{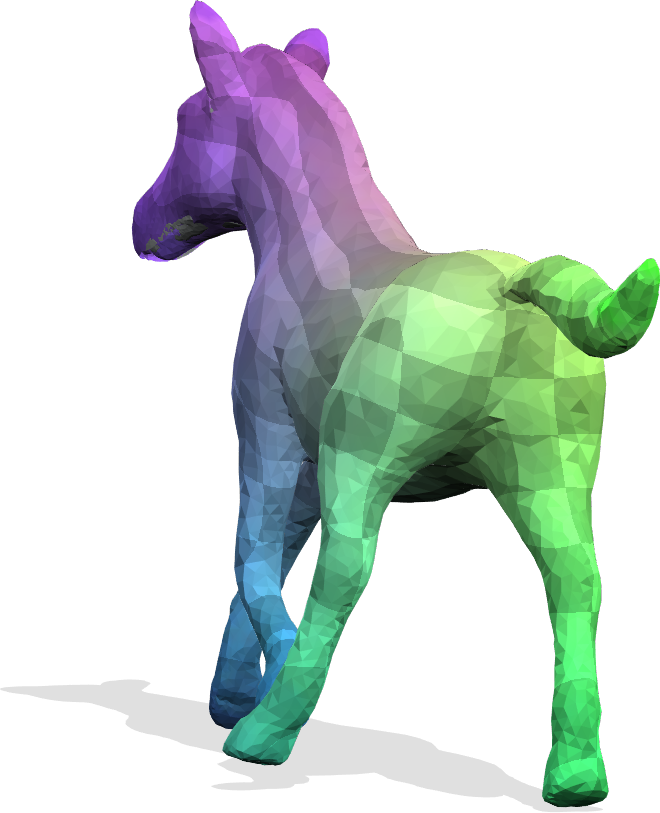}&
        \hspace{\hspaceCols}
        \includegraphics[height=\heightT, width=\widthT]{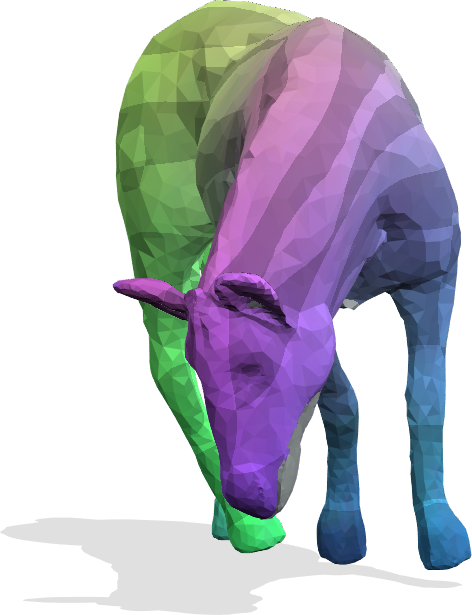}
        \\
        \hspace{\hspaceCols}
        \includegraphics[height=\heightT, width=\widthT]{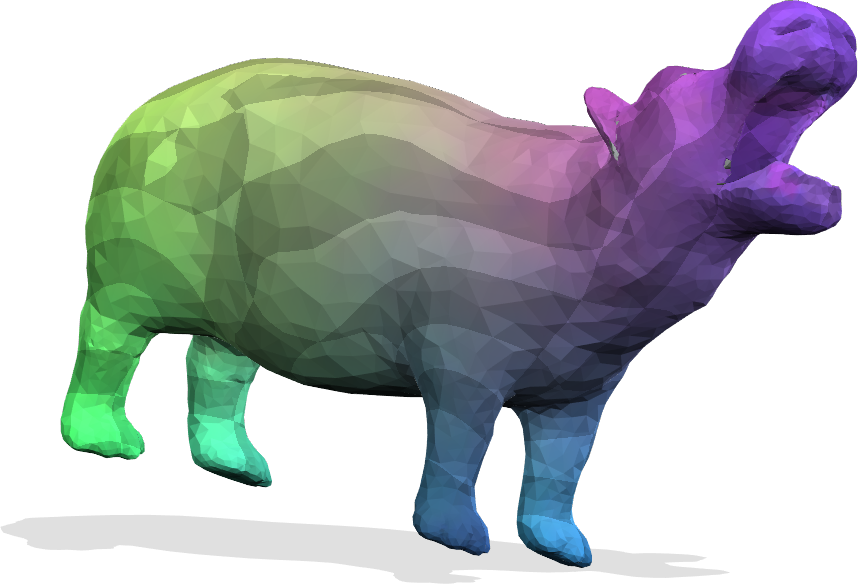}&
        \hspace{\hspaceCols}
        \includegraphics[height=\heightT, width=\widthT]{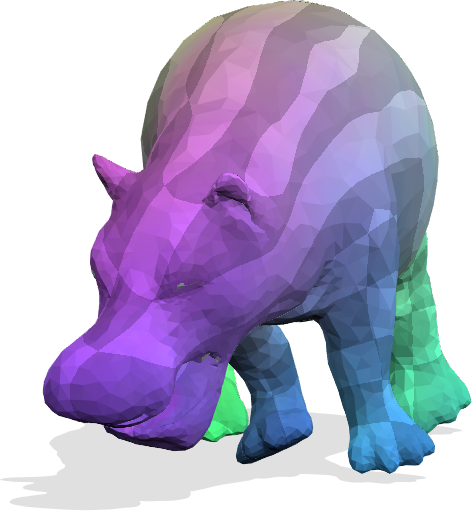}&
        \hspace{\hspaceCols}
        \includegraphics[height=\heightT, width=\widthT]{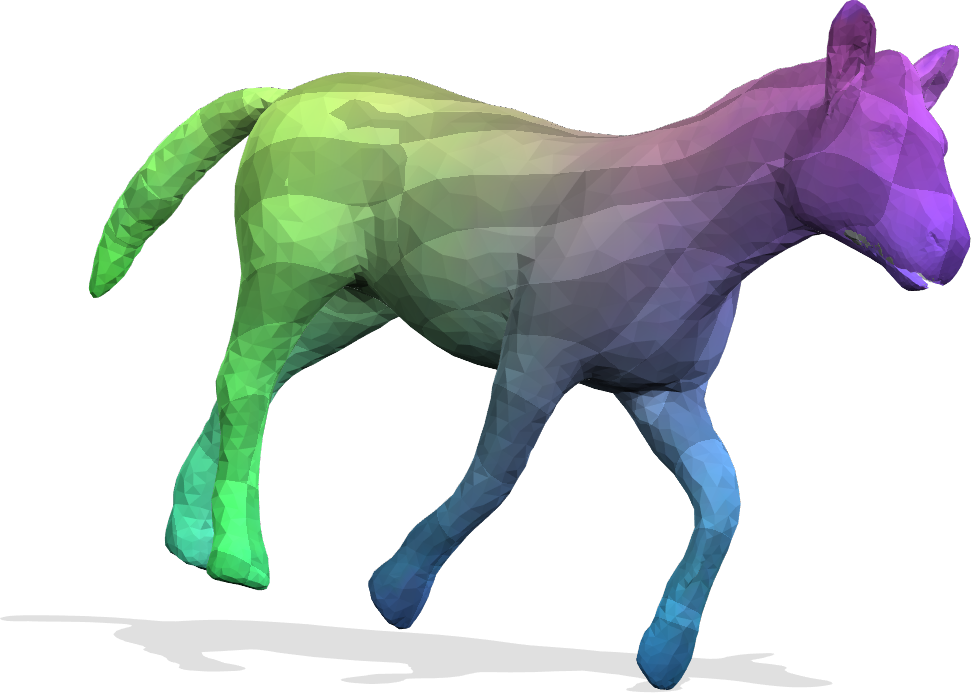}&
        \hspace{\hspaceCols}
        \includegraphics[height=\heightT, width=\widthT]{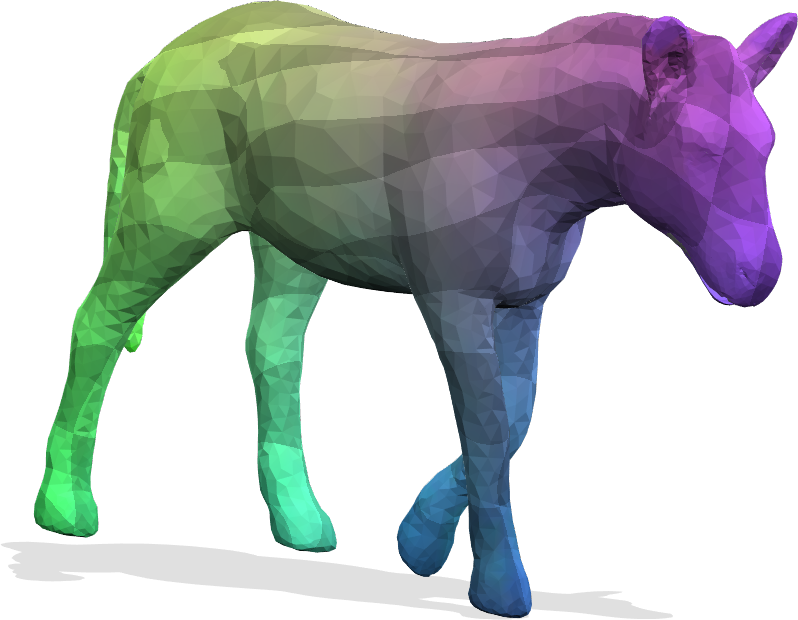}&
        \hspace{\hspaceCols}
        \includegraphics[height=\heightT, width=\widthT]{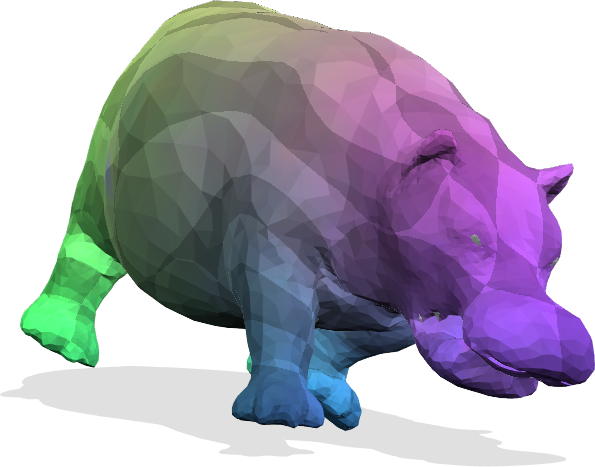}
        \\
        \hspace{\hspaceCols}
        \includegraphics[height=\heightT, width=\widthT]{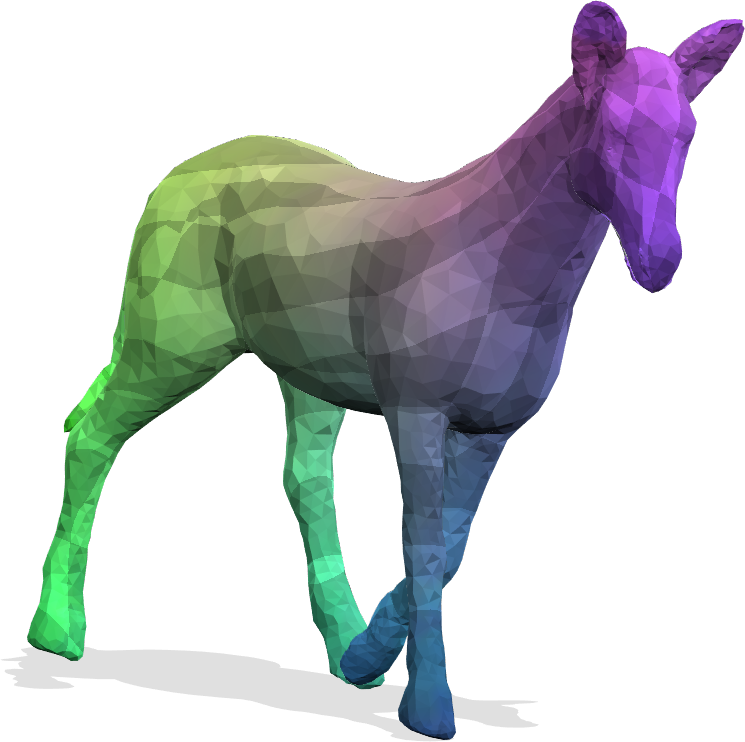}&
        \hspace{\hspaceCols}
        \includegraphics[height=\heightT, width=\widthT]{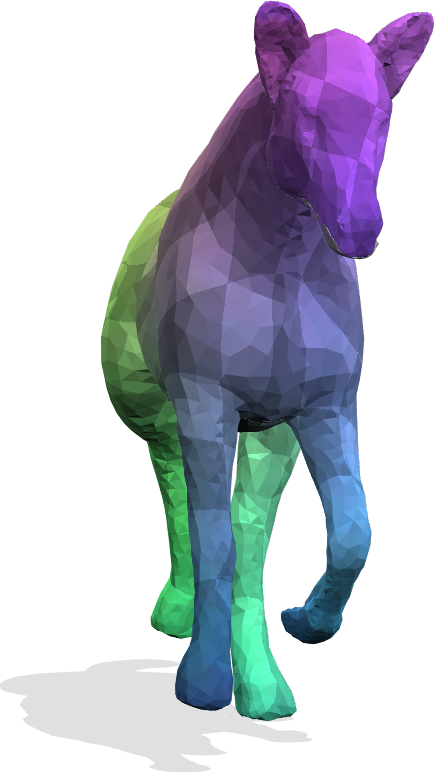}&
        \hspace{\hspaceCols}
        \includegraphics[height=\heightT, width=\widthT]{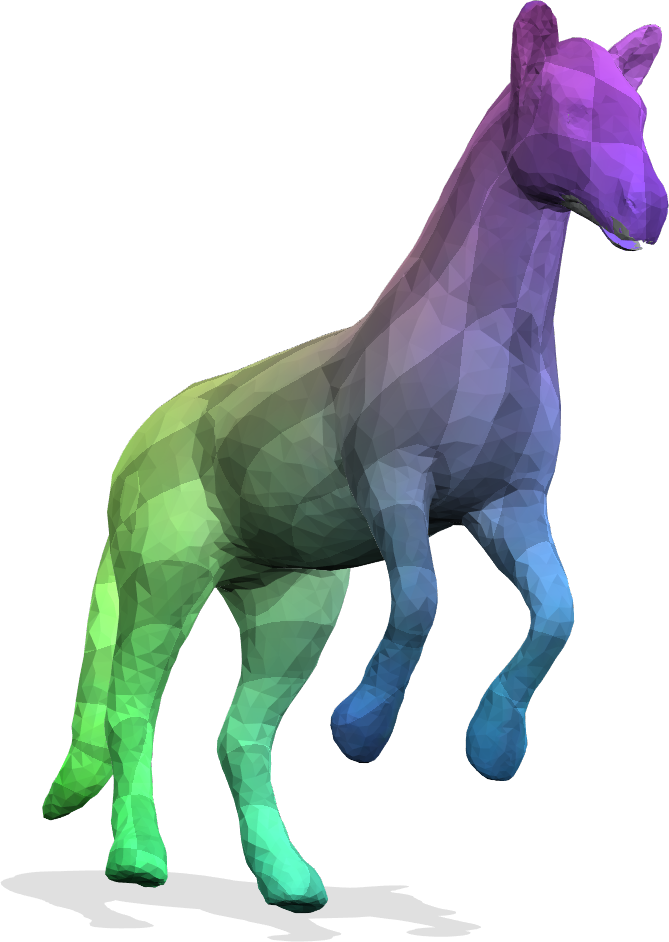}&
        \hspace{\hspaceCols}
        \includegraphics[height=\heightT, width=\widthT]{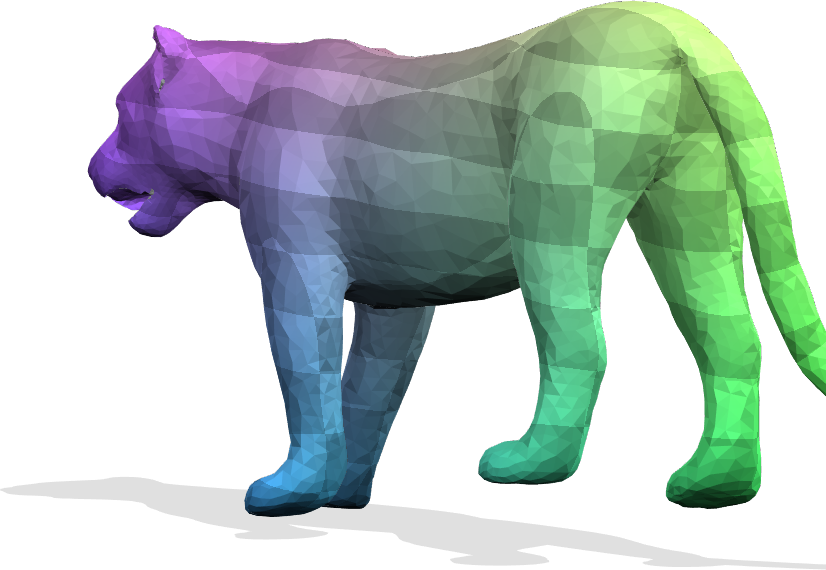}&
        \hspace{\hspaceCols}
        \includegraphics[height=\heightT, width=\widthT]{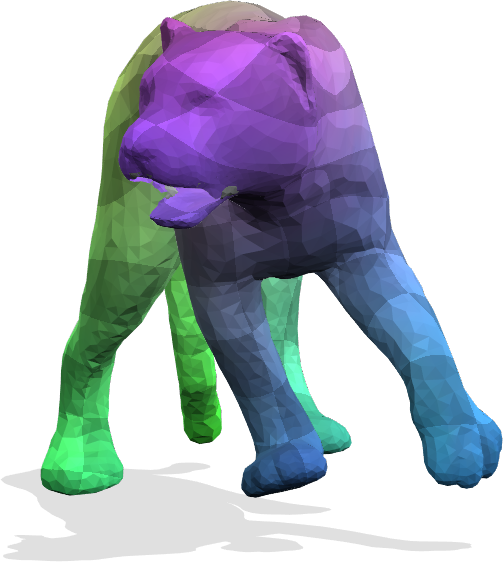}
    \end{tabular}
    \caption{\textbf{Qualitative results of our method on the SMAL dataset.} The top-left shape is the reference shape to be matched by other shapes. Our method obtains accurate correspondences for shapes in different classes.}
    \label{fig:smal_qualitative}
\end{figure}

\begin{figure}[!ht]
    \centering
    \def\rowOnecolumnOne{cat-cuts_cat_shape_1}
\def\rowOnecolumnTwo{cat-cuts_cat_shape_13}
\def\rowOnecolumnThree{cat-cuts_cat_shape_5}
\def\rowOnecolumnFour{cat-cuts_cat_shape_6}
\def\rowTwocolumnFive{cat-cuts_cat_shape_7}
\def\rowTwocolumnTwo{cat-cuts_cat_shape_2}
\def\rowTwocolumnThree{cat-cuts_cat_shape_11}
\def\rowTwocolumnFour{cat-cuts_cat_shape_14}
\def\rowOnecolumnFive{cat-cuts_cat_shape_21}
\def\rowThreecolumnTwo{cat-cuts_cat_shape_25}
\def\rowThreecolumnThree{cat-cuts_cat_shape_17}
\def\rowThreecolumnFour{cat-cuts_cat_shape_29}
\def\rowThreecolumnFive{cat-cuts_cat_shape_26}

\def\rowFourcolumnOne{centaur-cuts_centaur_shape_1}
\def\rowFourcolumnTwo{centaur-cuts_centaur_shape_1}
\def\rowFourcolumnThree{centaur-cuts_centaur_shape_2}
\def\rowFourcolumnFour{centaur-cuts_centaur_shape_4}
\def\rowFourcolumnFive{centaur-cuts_centaur_shape_3}
\def\rowFivecolumnTwo{centaur-cuts_centaur_shape_8}
\def\rowFivecolumnThree{centaur-cuts_centaur_shape_9}
\def\rowFivecolumnFour{centaur-cuts_centaur_shape_11}
\def\rowFivecolumnFive{centaur-cuts_centaur_shape_12}

\def\rowSixcolumnOne{horse-cuts_horse_shape_1}
\def\rowSixcolumnTwo{horse-cuts_horse_shape_1}
\def\rowSixcolumnThree{horse-cuts_horse_shape_3}
\def\rowSixcolumnFour{horse-cuts_horse_shape_4}
\def\rowSixcolumnFive{horse-cuts_horse_shape_5}
\def\rowSevencolumnTwo{horse-cuts_horse_shape_6}
\def\rowSevencolumnThree{horse-cuts_horse_shape_7}
\def\rowSevencolumnFour{horse-cuts_horse_shape_8}
\def\rowSevencolumnFive{horse-cuts_horse_shape_9}
\def\rowEightcolumnTwo{horse-cuts_horse_shape_10}
\def\rowEightcolumnThree{horse-cuts_horse_shape_12}
\def\rowEightcolumnFour{horse-cuts_horse_shape_2}
\def\rowEightcolumnFive{horse-cuts_horse_shape_16}

\def\rowNinecolumnOne{dog-cuts_dog_shape_1}
\def\rowNinecolumnTwo{dog-cuts_dog_shape_1}
\def\rowNinecolumnThree{dog-cuts_dog_shape_19}
\def\rowNinecolumnFour{dog-cuts_dog_shape_2}
\def\rowNinecolumnFive{dog-cuts_dog_shape_16}
\def\rowTencolumnTwo{dog-cuts_dog_shape_3}
\def\rowTencolumnThree{dog-cuts_dog_shape_15}
\def\rowTencolumnFour{dog-cuts_dog_shape_5}
\def\rowTencolumnFive{dog-cuts_dog_shape_13}
\def\rowElevencolumnTwo{dog-cuts_dog_shape_7}
\def\rowElevencolumnThree{dog-cuts_dog_shape_12}
\def\rowElevencolumnFour{dog-cuts_dog_shape_10}
\def\rowElevencolumnFive{dog-cuts_dog_shape_9}

\def\rowTwelvecolumnOne{wolf-cuts_wolf_shape_1}
\def\rowTwelvecolumnTwo{wolf-cuts_wolf_shape_1}
\def\rowTwelvecolumnThree{wolf-cuts_wolf_shape_2}
\def\rowTwelvecolumnFour{wolf-cuts_wolf_shape_5}
\def\rowTwelvecolumnFive{wolf-cuts_wolf_shape_3}

\def\hspaceCols{-0.45cm}
\def\height{1.5cm}
\def\width{1.7cm}
\def\heightT{\height}
\def\widthT{\width}
\def\heightQ{\height}
\def\widthQ{\width}
\begin{tabular}{ccccc}%
        \setlength{\tabcolsep}{0pt} 
        \hspace{\hspaceCols}
        \includegraphics[height=\heightT, width=\widthT]{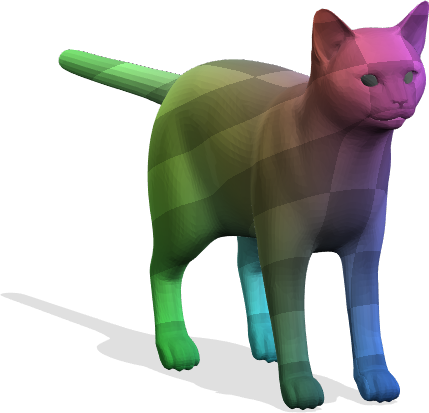}&
        \hspace{\hspaceCols}
        \includegraphics[height=\heightT, width=\widthT]{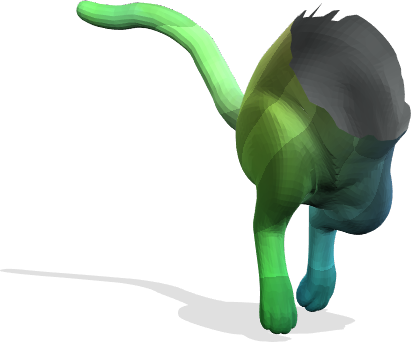}&
        \hspace{\hspaceCols}
        \includegraphics[height=\heightT, width=\widthT]{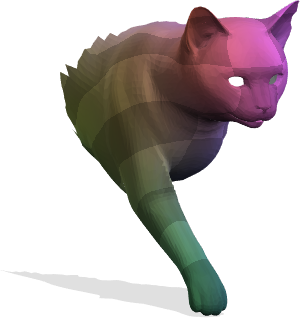}&
        \hspace{\hspaceCols}
        \includegraphics[height=\heightT, width=\widthT]{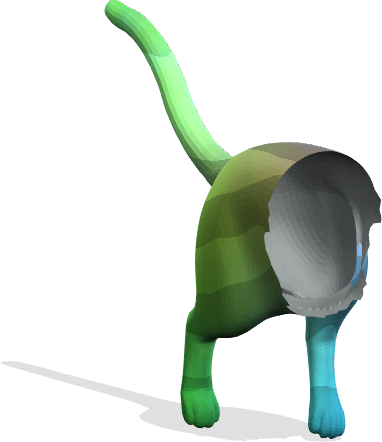}&
        \hspace{\hspaceCols}
        \includegraphics[height=\heightT, width=\widthT]{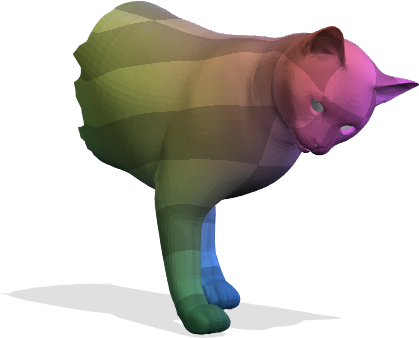}
        \\
        &
        \hspace{\hspaceCols}
        \includegraphics[height=\heightT, width=\widthT]{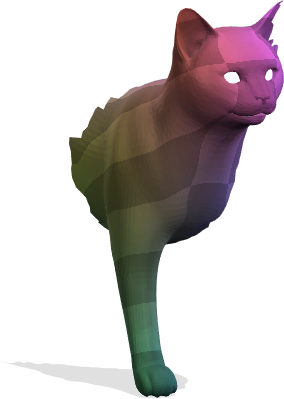}&
        \hspace{\hspaceCols}
        \includegraphics[height=\heightT, width=\widthT]{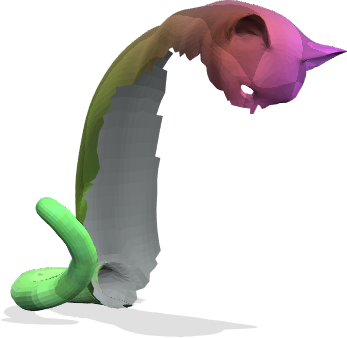}&
        \hspace{\hspaceCols}
        \includegraphics[height=\heightT, width=\widthT]{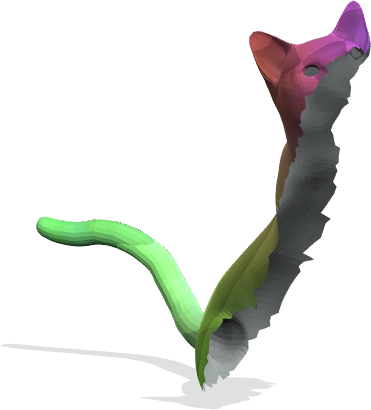}&
        \hspace{\hspaceCols}
        \includegraphics[height=\heightT, width=\widthT]{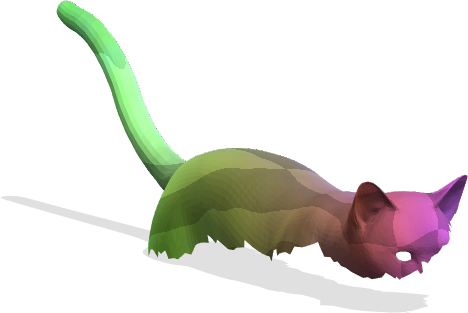}
        \\
        &
        \hspace{\hspaceCols}
        \includegraphics[height=\heightT, width=\widthT]{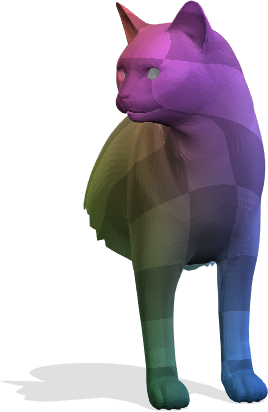}&
        \hspace{\hspaceCols}
        \includegraphics[height=\heightT, width=\widthT]{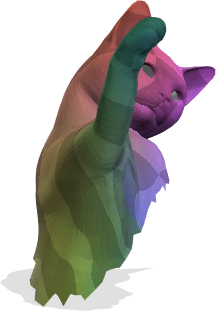}&
        \hspace{\hspaceCols}
        \includegraphics[height=\heightT, width=\widthT]{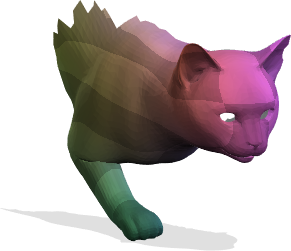}&
        \hspace{\hspaceCols}
        \includegraphics[height=\heightT, width=\widthT]{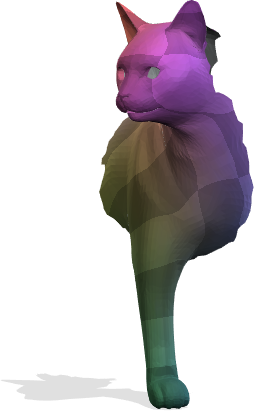} \\

        \hspace{\hspaceCols}
        \includegraphics[height=\heightT, width=\widthT]{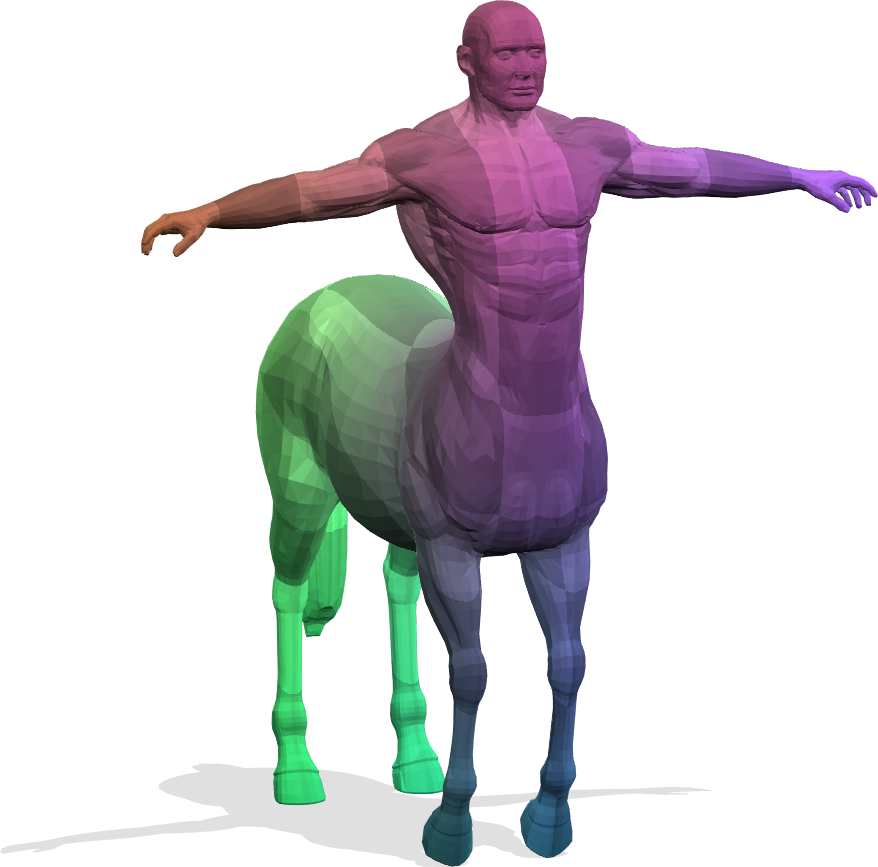}&
        \hspace{\hspaceCols}
        \includegraphics[height=\heightT, width=\widthT]{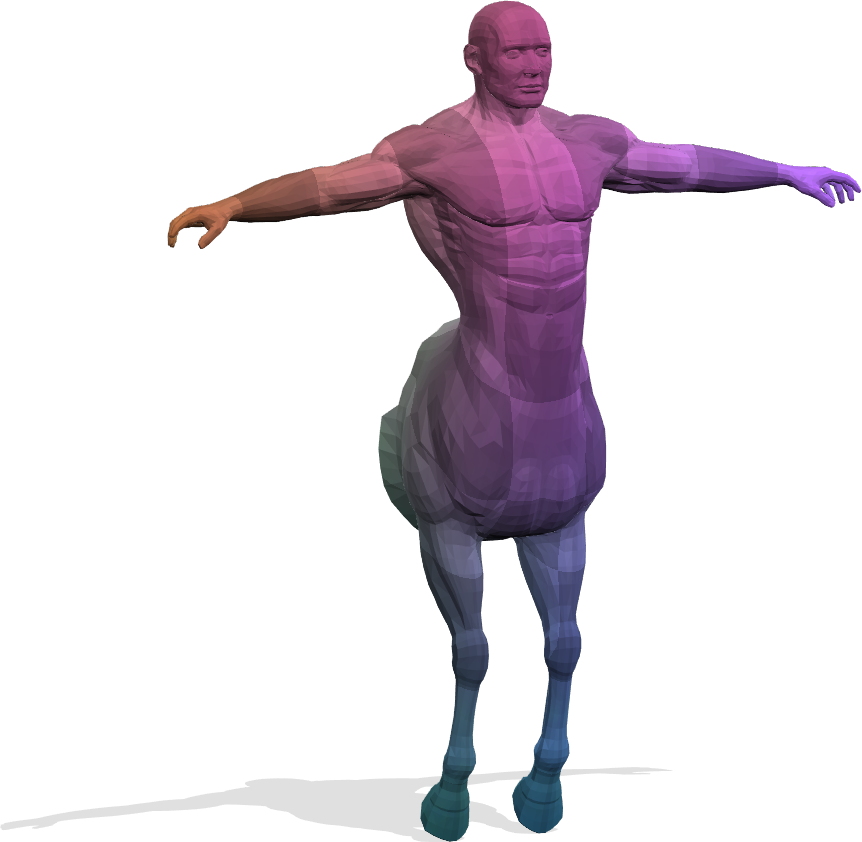}&
        \hspace{\hspaceCols}
        \includegraphics[height=\heightT, width=\widthT]{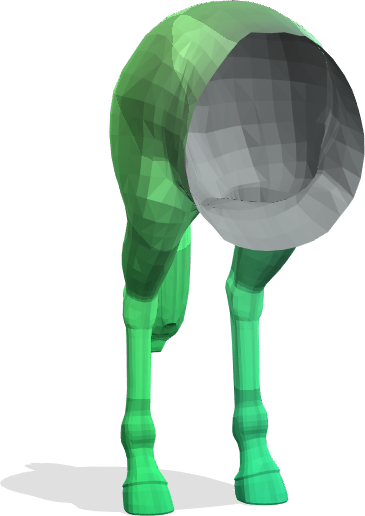}&
        \hspace{\hspaceCols}
        \includegraphics[height=\heightT, width=\widthT]{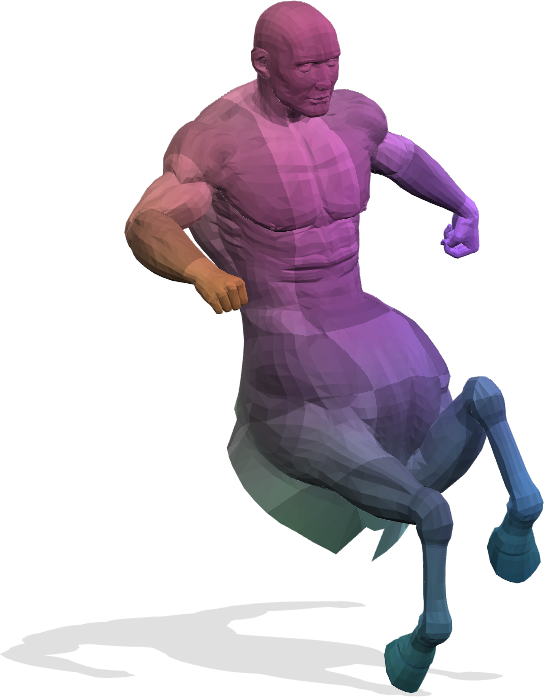}&
        \hspace{\hspaceCols}
        \includegraphics[height=\heightT, width=\widthT]{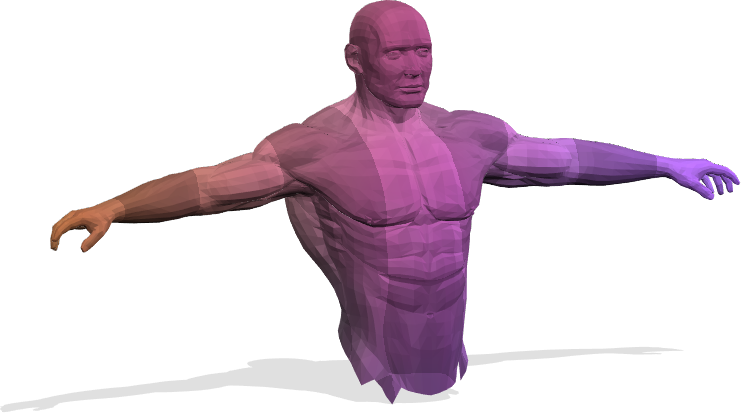}
        \\
        &
        \hspace{\hspaceCols}
        \includegraphics[height=\heightT, width=\widthT]{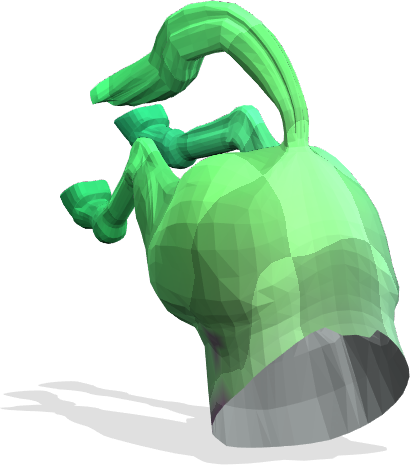}&
        \hspace{\hspaceCols}
        \includegraphics[height=\heightT, width=\widthT]{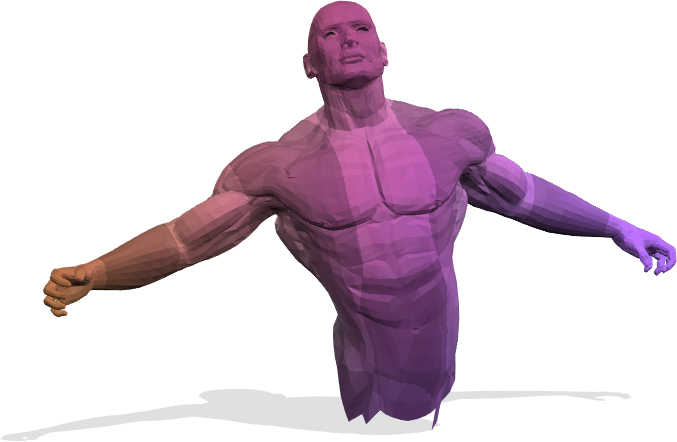}&
        \hspace{\hspaceCols}
        \includegraphics[height=\heightT, width=\widthT]{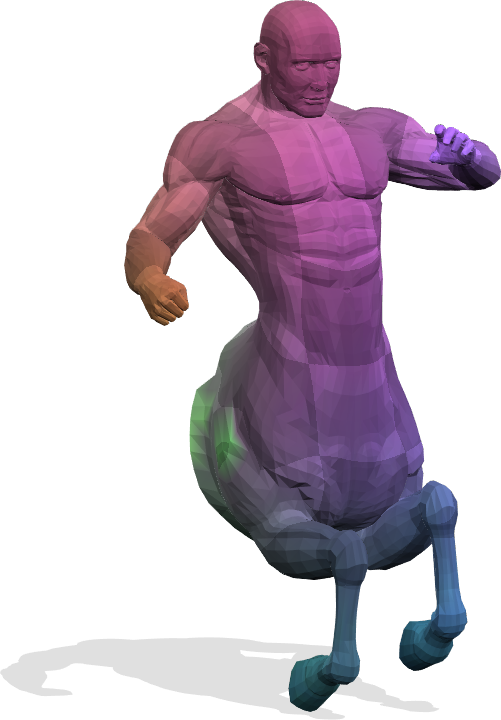}&
        \hspace{\hspaceCols}
        \includegraphics[height=\heightT, width=\widthT]{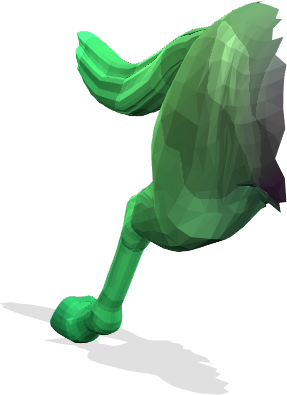} \\
        \hspace{\hspaceCols}
        \includegraphics[height=\heightT, width=\widthT]{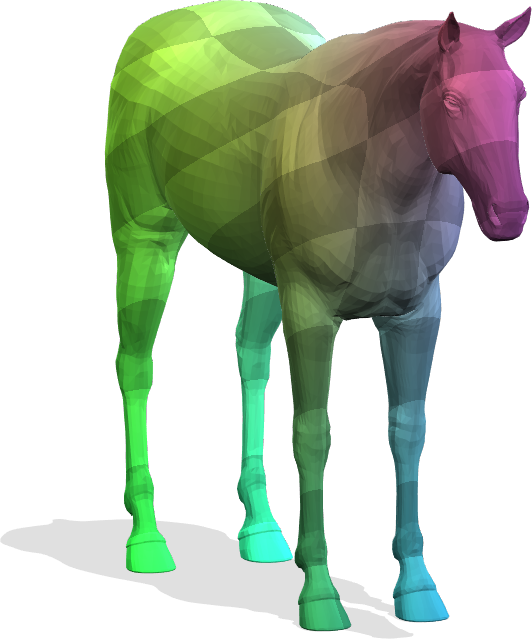}&
        \hspace{\hspaceCols}
        \includegraphics[height=\heightT, width=\widthT]{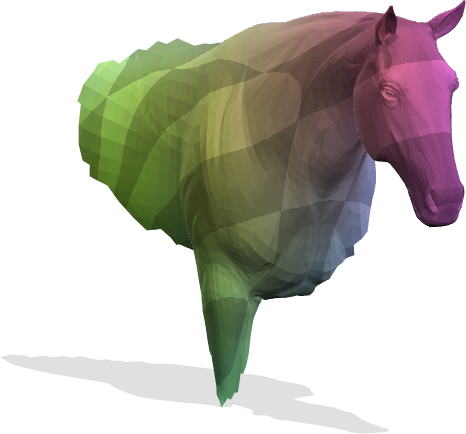}&
        \hspace{\hspaceCols}
        \includegraphics[height=\heightT, width=\widthT]{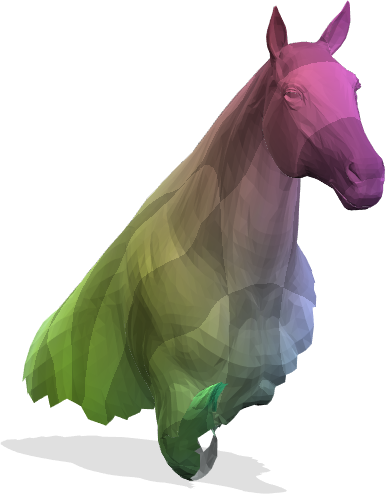}&
        \hspace{\hspaceCols}
        \includegraphics[height=\heightT, width=\widthT]{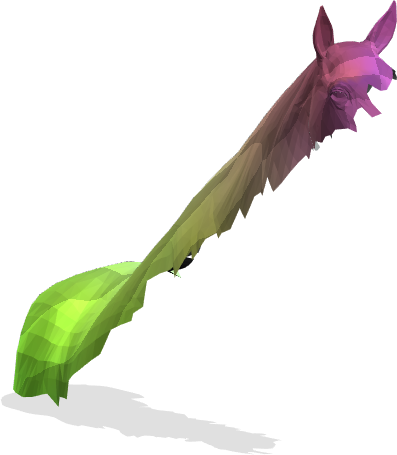}&
        \hspace{\hspaceCols}
        \includegraphics[height=\heightT, width=\widthT]{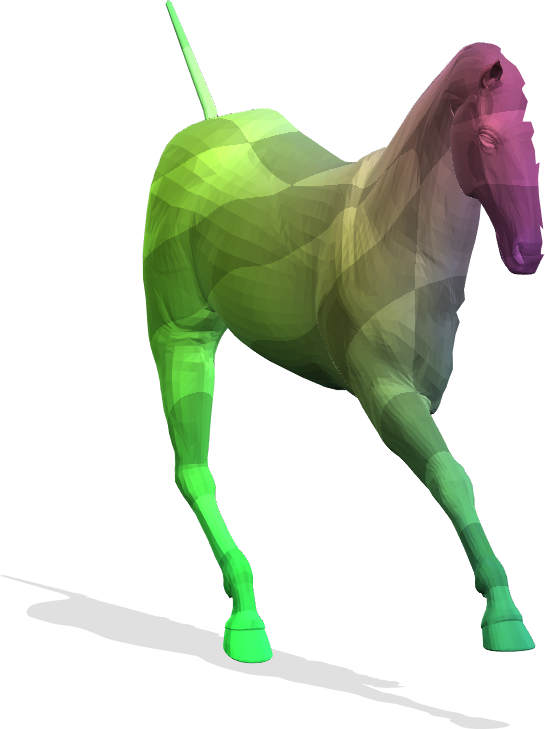} \\
        &
        \hspace{\hspaceCols}
        \includegraphics[height=\heightT, width=\widthT]{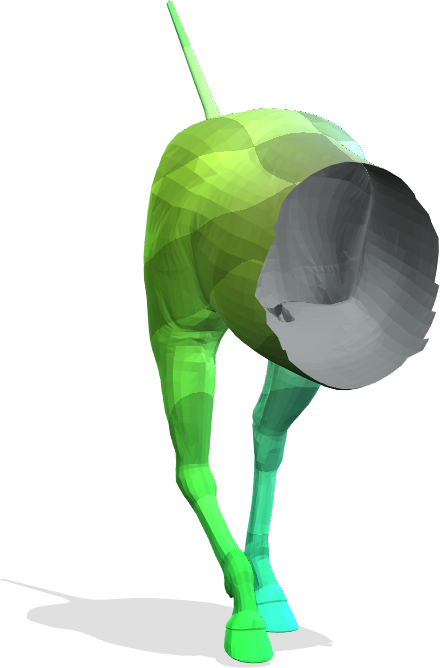}&
        \hspace{\hspaceCols}
        \includegraphics[height=\heightT, width=\widthT]{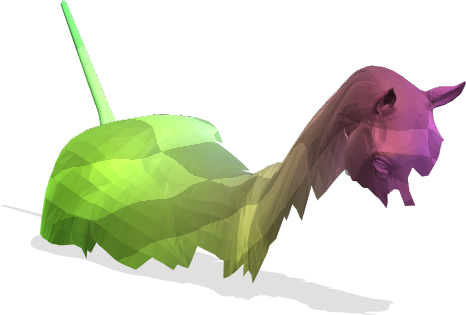}&
        \hspace{\hspaceCols}
        \includegraphics[height=\heightT, width=\widthT]{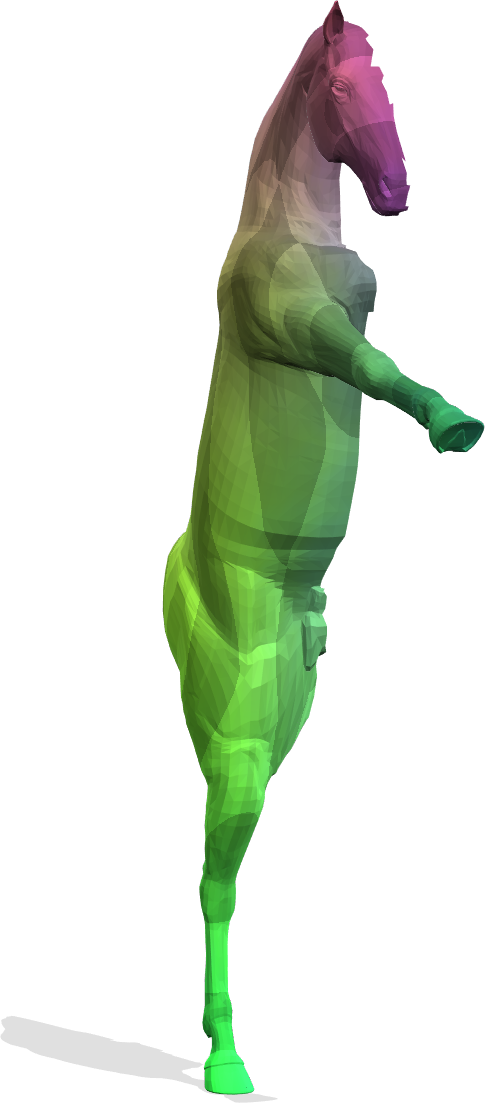}&
        \hspace{\hspaceCols}
        \includegraphics[height=\heightT, width=\widthT]{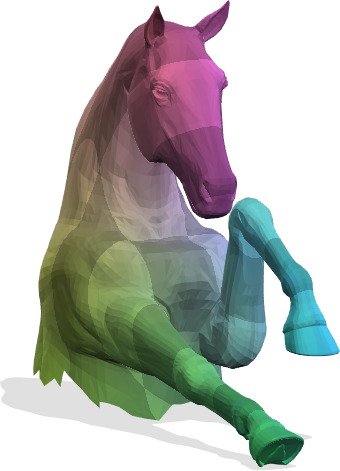}
        \\
        &
        \hspace{\hspaceCols}
        \includegraphics[height=\heightT, width=\widthT]{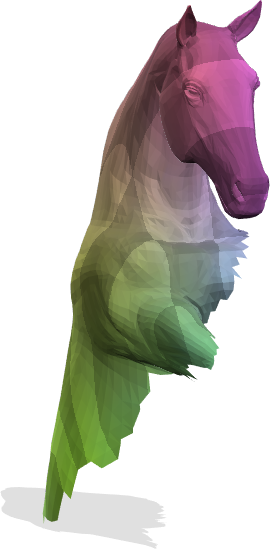}&
        \hspace{\hspaceCols}
        \includegraphics[height=\heightT, width=\widthT]{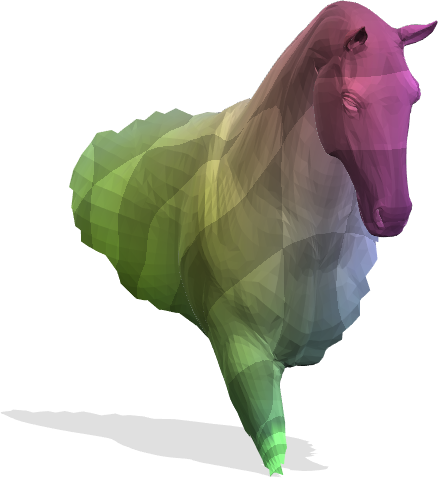}&
        \hspace{\hspaceCols}
        \includegraphics[height=\heightT, width=\widthT]{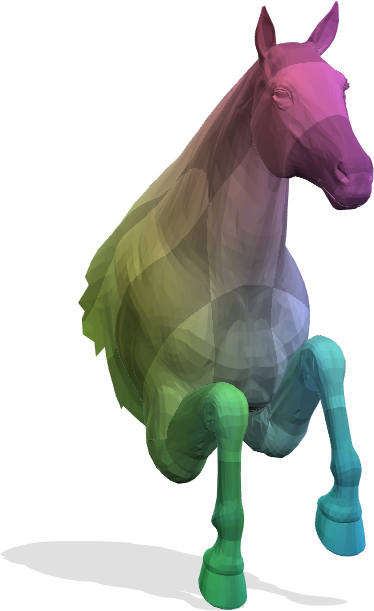}&
        \hspace{\hspaceCols}
        \includegraphics[height=\heightT, width=\widthT]{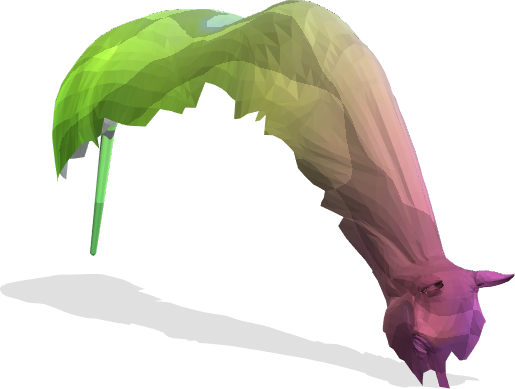} \\

        \hspace{\hspaceCols}
        \includegraphics[height=\heightT, width=\widthT]{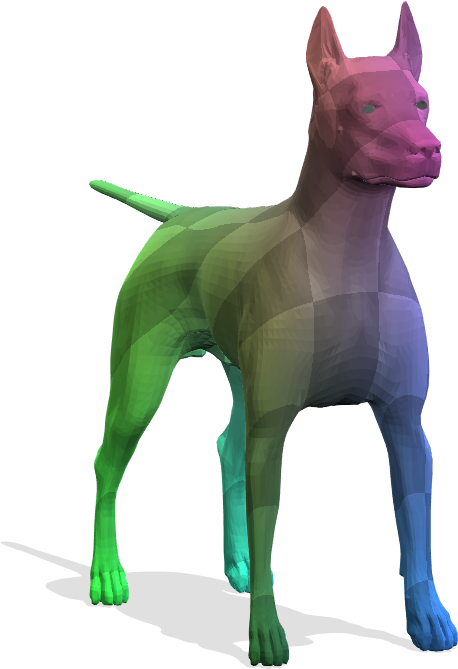}&
        \hspace{\hspaceCols}
        \includegraphics[height=\heightT, width=\widthT]{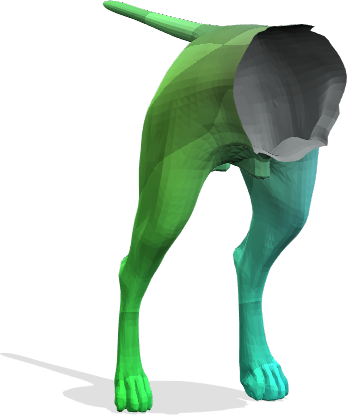}&
        \hspace{\hspaceCols}
        \includegraphics[height=\heightT, width=\widthT]{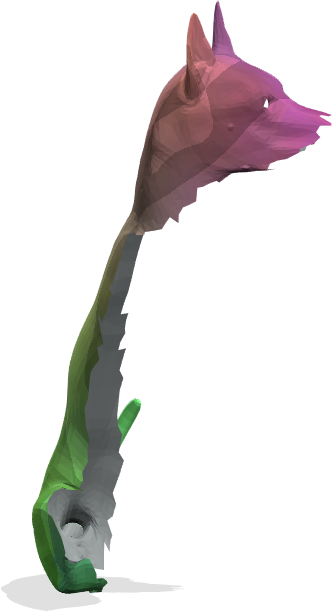}&
        \hspace{\hspaceCols}
        \includegraphics[height=\heightT, width=\widthT]{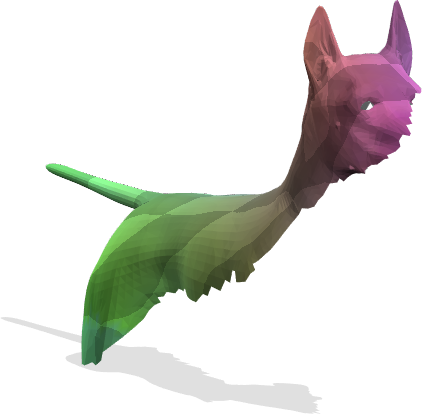}&
        \hspace{\hspaceCols}
        \includegraphics[height=\heightT, width=\widthT]{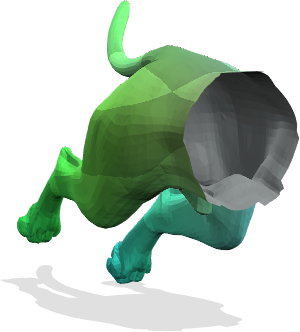} \\
        &
        \hspace{\hspaceCols}
        \includegraphics[height=\heightT, width=\widthT]{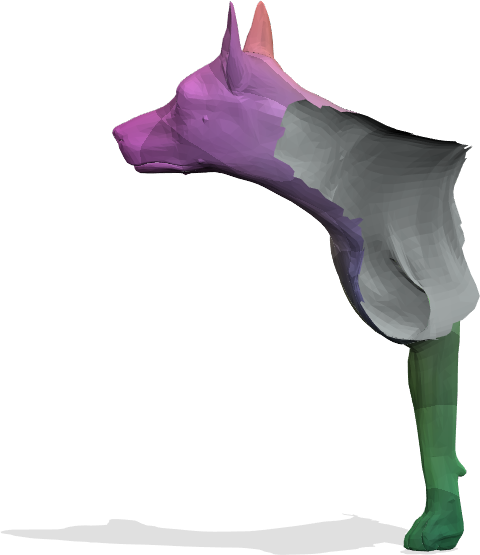}&
        \hspace{\hspaceCols}
        \includegraphics[height=\heightT, width=\widthT]{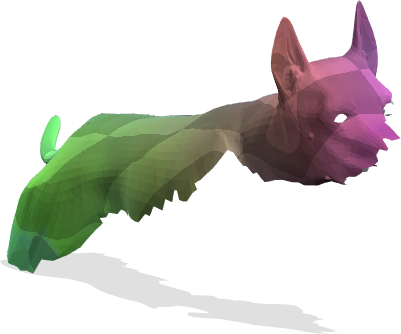}&
        \hspace{\hspaceCols}
        \includegraphics[height=\heightT, width=\widthT]{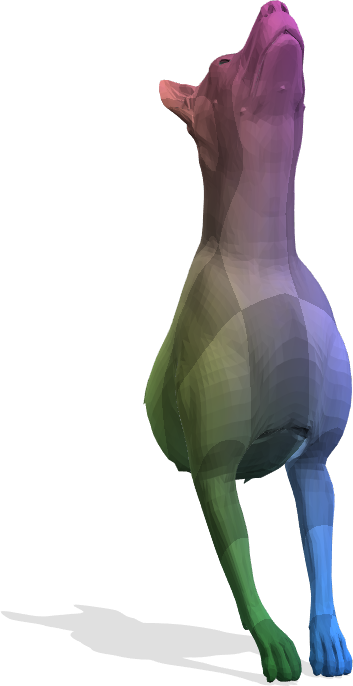}&
        \hspace{\hspaceCols}
        \includegraphics[height=\heightT, width=\widthT]{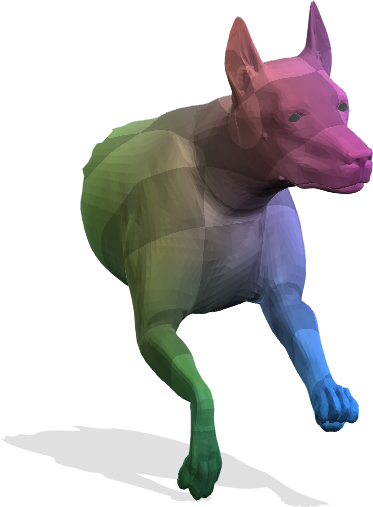}
        \\
        &
        \hspace{\hspaceCols}
        \includegraphics[height=\heightT, width=\widthT]{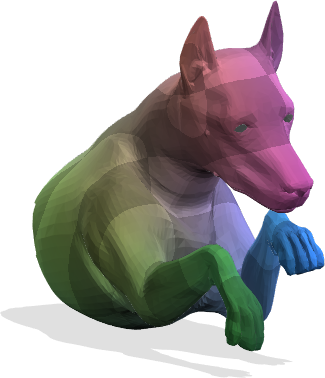}&
        \hspace{\hspaceCols}
        \includegraphics[height=\heightT, width=\widthT]{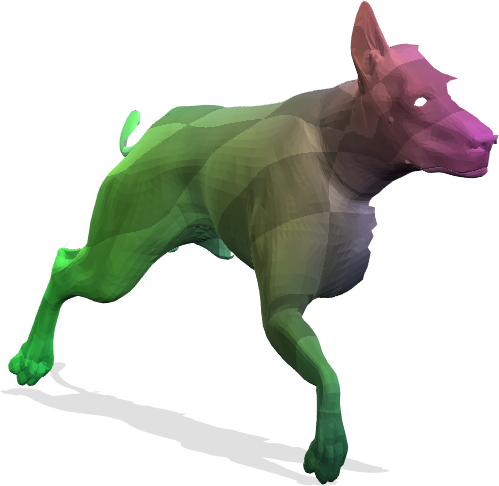}&
        \hspace{\hspaceCols}
        \includegraphics[height=\heightT, width=\widthT]{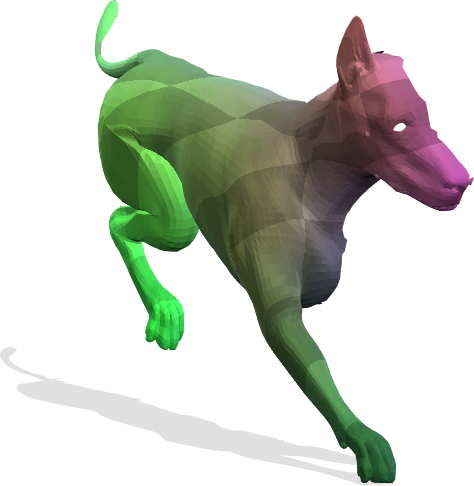}&
        \hspace{\hspaceCols}
        \includegraphics[height=\heightT, width=\widthT]{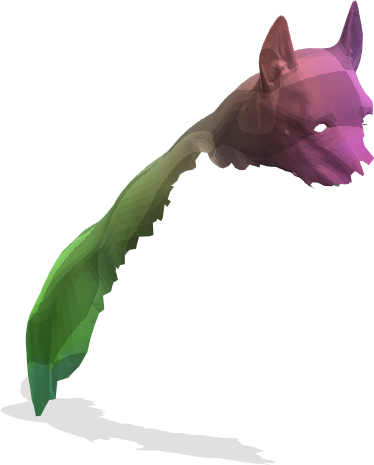} \\

        \hspace{\hspaceCols}
        \includegraphics[height=\heightT, width=\widthT]{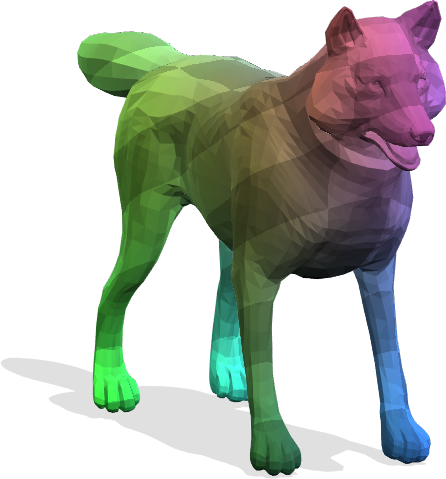}&
        \hspace{\hspaceCols}
        \includegraphics[height=\heightT, width=\widthT]{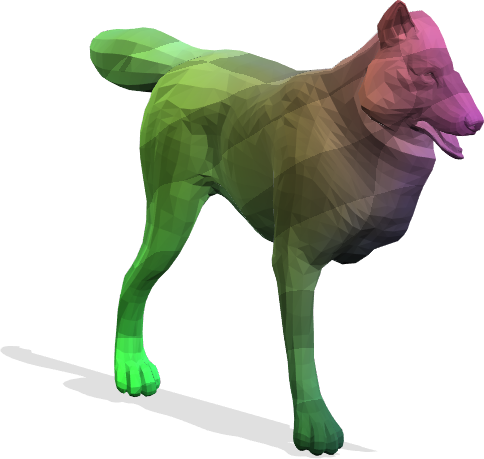}&
        \hspace{\hspaceCols}
        \includegraphics[height=\heightT, width=\widthT]{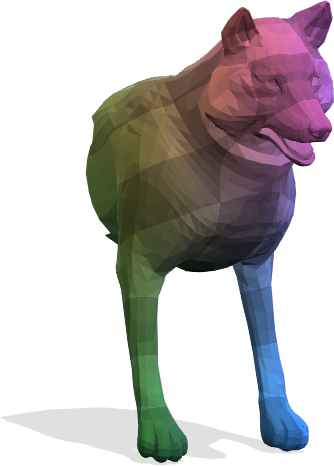}&
        \hspace{\hspaceCols}
        \includegraphics[height=\heightT, width=\widthT]{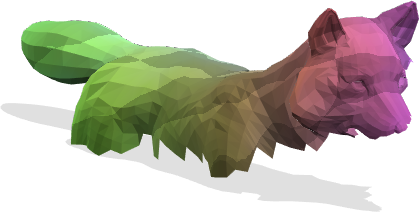}&
        \hspace{\hspaceCols}
        \includegraphics[height=\heightT, width=\widthT]{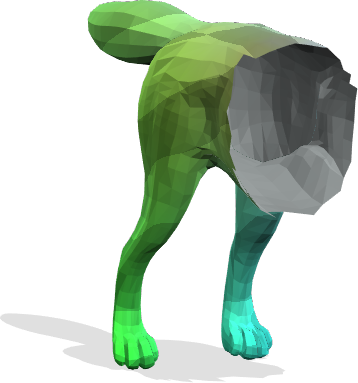} \\
        
    \end{tabular}
    \caption{\textbf{Qualitative results of our method on the SHREC'16 CUTS dataset.} For each shape category, the top-left shape is the reference shape to be matched by other shapes. Our method obtains accurate correspondences for partial shapes with a large missing part.}
    \label{fig:cuts_qualitative}
\end{figure}

\begin{figure}[!ht]
    \centering
    \def\rowOnecolumnOne{cat-holes_cat_shape_1}
\def\rowOnecolumnTwo{cat-holes_cat_shape_2}
\def\rowOnecolumnThree{cat-holes_cat_shape_24}
\def\rowOnecolumnFour{cat-holes_cat_shape_4}
\def\rowTwocolumnFive{cat-holes_cat_shape_23}
\def\rowTwocolumnTwo{cat-holes_cat_shape_5}
\def\rowTwocolumnThree{cat-holes_cat_shape_18}
\def\rowTwocolumnFour{cat-holes_cat_shape_7}
\def\rowOnecolumnFive{cat-holes_cat_shape_16}
\def\rowThreecolumnTwo{cat-holes_cat_shape_8}
\def\rowThreecolumnThree{cat-holes_cat_shape_14}
\def\rowThreecolumnFour{cat-holes_cat_shape_11}
\def\rowThreecolumnFive{cat-holes_cat_shape_13}

\def\rowFourcolumnOne{centaur-holes_centaur_shape_1}
\def\rowFourcolumnTwo{centaur-holes_centaur_shape_2}
\def\rowFourcolumnThree{centaur-holes_centaur_shape_17}
\def\rowFourcolumnFour{centaur-holes_centaur_shape_3}
\def\rowFourcolumnFive{centaur-holes_centaur_shape_12}
\def\rowFivecolumnTwo{centaur-holes_centaur_shape_5}
\def\rowFivecolumnThree{centaur-holes_centaur_shape_10}
\def\rowFivecolumnFour{centaur-holes_centaur_shape_6}
\def\rowFivecolumnFive{centaur-holes_centaur_shape_1}

\def\rowSixcolumnOne{horse-holes_horse_shape_1}
\def\rowSixcolumnTwo{horse-holes_horse_shape_1}
\def\rowSixcolumnThree{horse-holes_horse_shape_21}
\def\rowSixcolumnFour{horse-holes_horse_shape_2}
\def\rowSixcolumnFive{horse-holes_horse_shape_20}
\def\rowSevencolumnTwo{horse-holes_horse_shape_3}
\def\rowSevencolumnThree{horse-holes_horse_shape_19}
\def\rowSevencolumnFour{horse-holes_horse_shape_5}
\def\rowSevencolumnFive{horse-holes_horse_shape_15}
\def\rowEightcolumnTwo{horse-holes_horse_shape_12}
\def\rowEightcolumnThree{horse-holes_horse_shape_8}
\def\rowEightcolumnFour{horse-holes_horse_shape_11}
\def\rowEightcolumnFive{horse-holes_horse_shape_7}

\def\rowNinecolumnOne{dog-holes_dog_shape_1}
\def\rowNinecolumnTwo{dog-holes_dog_shape_12}
\def\rowNinecolumnThree{dog-holes_dog_shape_25}
\def\rowNinecolumnFour{dog-holes_dog_shape_3}
\def\rowNinecolumnFive{dog-holes_dog_shape_23}
\def\rowTencolumnTwo{dog-holes_dog_shape_4}
\def\rowTencolumnThree{dog-holes_dog_shape_19}
\def\rowTencolumnFour{dog-holes_dog_shape_6}
\def\rowTencolumnFive{dog-holes_dog_shape_17}
\def\rowElevencolumnTwo{dog-holes_dog_shape_8}
\def\rowElevencolumnThree{dog-holes_dog_shape_15}
\def\rowElevencolumnFour{dog-holes_dog_shape_10}
\def\rowElevencolumnFive{dog-holes_dog_shape_14}

\def\rowTwelvecolumnOne{wolf-holes_wolf_shape_1}
\def\rowTwelvecolumnTwo{wolf-holes_wolf_shape_1}
\def\rowTwelvecolumnThree{wolf-holes_wolf_shape_2}
\def\rowTwelvecolumnFour{wolf-holes_wolf_shape_5}
\def\rowTwelvecolumnFive{wolf-holes_wolf_shape_8}

\def\hspaceCols{-0.45cm}
\def\height{1.5cm}
\def\width{1.7cm}
\def\heightT{\height}
\def\widthT{\width}
\def\heightQ{\height}
\def\widthQ{\width}
\begin{tabular}{ccccc}%
        \setlength{\tabcolsep}{0pt} 
        \hspace{\hspaceCols}
        \includegraphics[height=\heightT, width=\widthT]{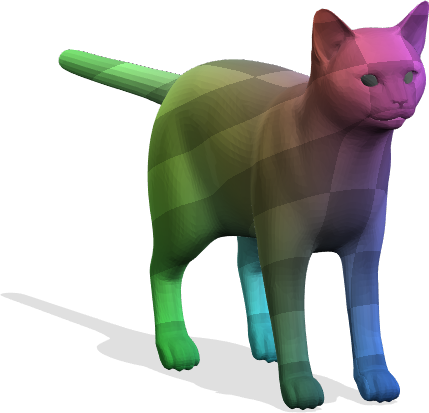}&
        \hspace{\hspaceCols}
        \includegraphics[height=\heightT, width=\widthT]{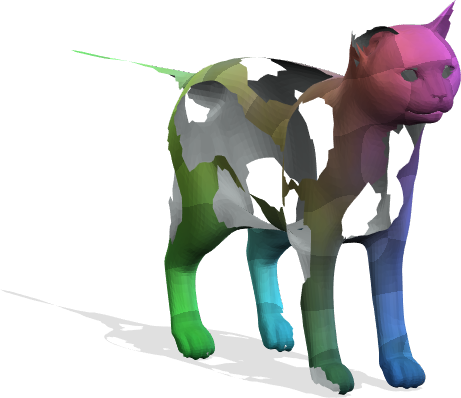}&
        \hspace{\hspaceCols}
        \includegraphics[height=\heightT, width=\widthT]{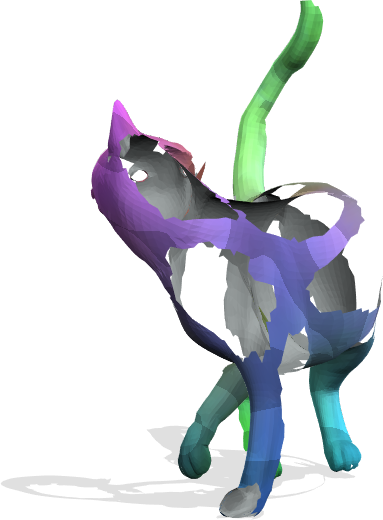}&
        \hspace{\hspaceCols}
        \includegraphics[height=\heightT, width=\widthT]{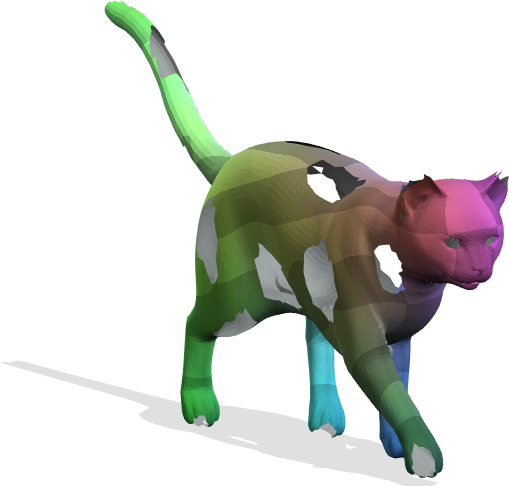}&
        \hspace{\hspaceCols}
        \includegraphics[height=\heightT, width=\widthT]{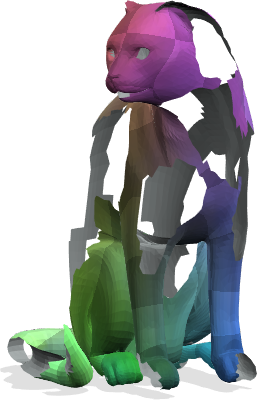}
        \\
        &
        \hspace{\hspaceCols}
        \includegraphics[height=\heightT, width=\widthT]{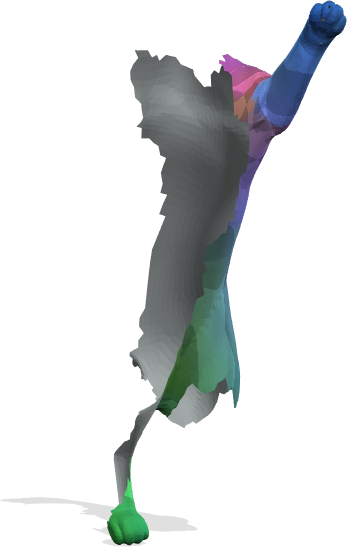}&
        \hspace{\hspaceCols}
        \includegraphics[height=\heightT, width=\widthT]{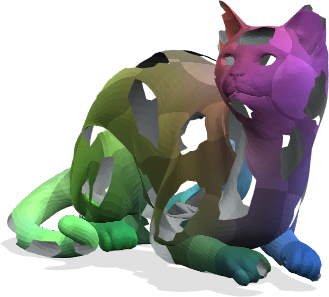}&
        \hspace{\hspaceCols}
        \includegraphics[height=\heightT, width=\widthT]{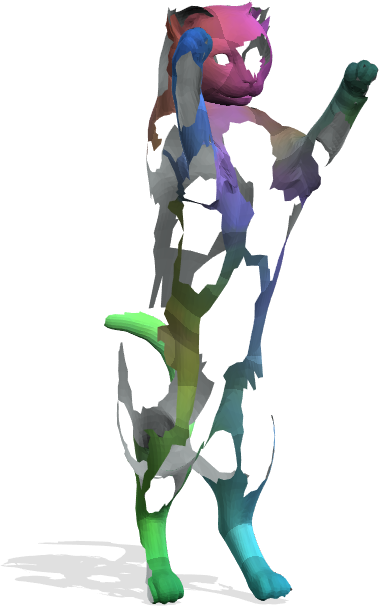}&
        \hspace{\hspaceCols}
        \includegraphics[height=\heightT, width=\widthT]{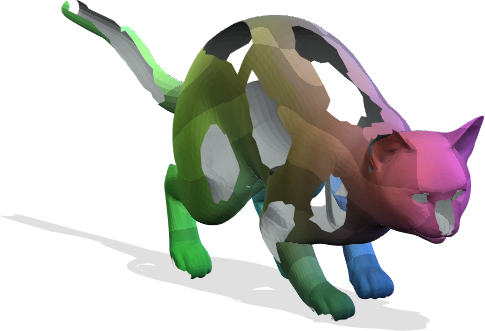}
        \\
        &
        \hspace{\hspaceCols}
        \includegraphics[height=\heightT, width=\widthT]{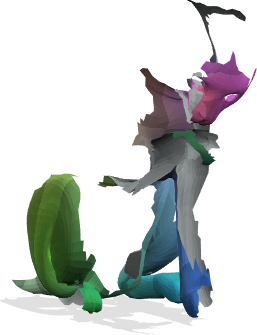}&
        \hspace{\hspaceCols}
        \includegraphics[height=\heightT, width=\widthT]{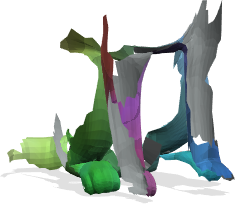}&
        \hspace{\hspaceCols}
        \includegraphics[height=\heightT, width=\widthT]{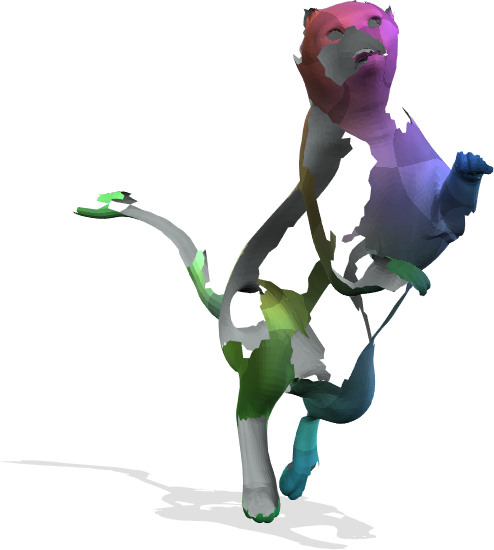}&
        \hspace{\hspaceCols}
        \includegraphics[height=\heightT, width=\widthT]{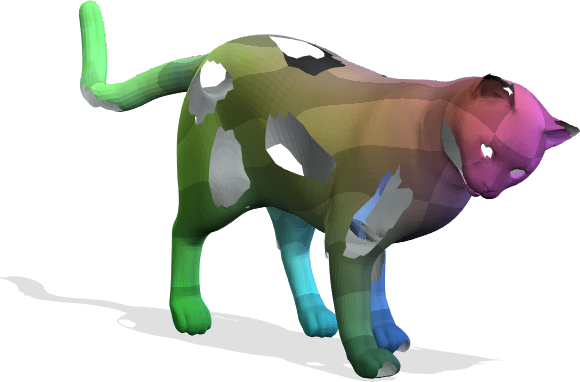} \\

        \hspace{\hspaceCols}
        \includegraphics[height=\heightT, width=\widthT]{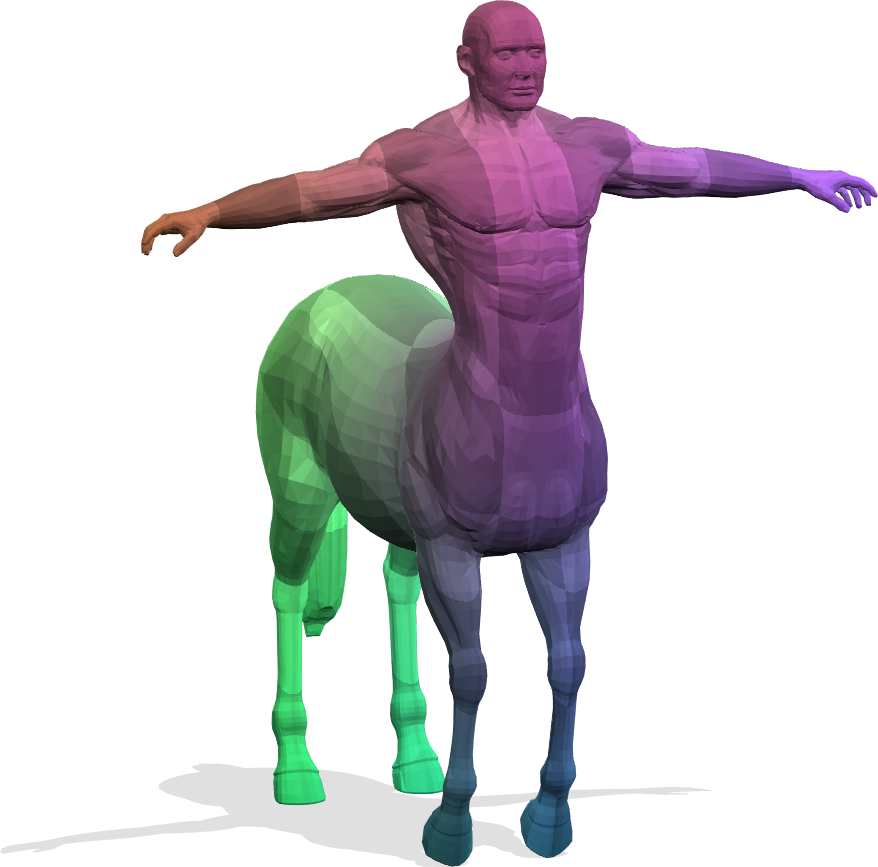}&
        \hspace{\hspaceCols}
        \includegraphics[height=\heightT, width=\widthT]{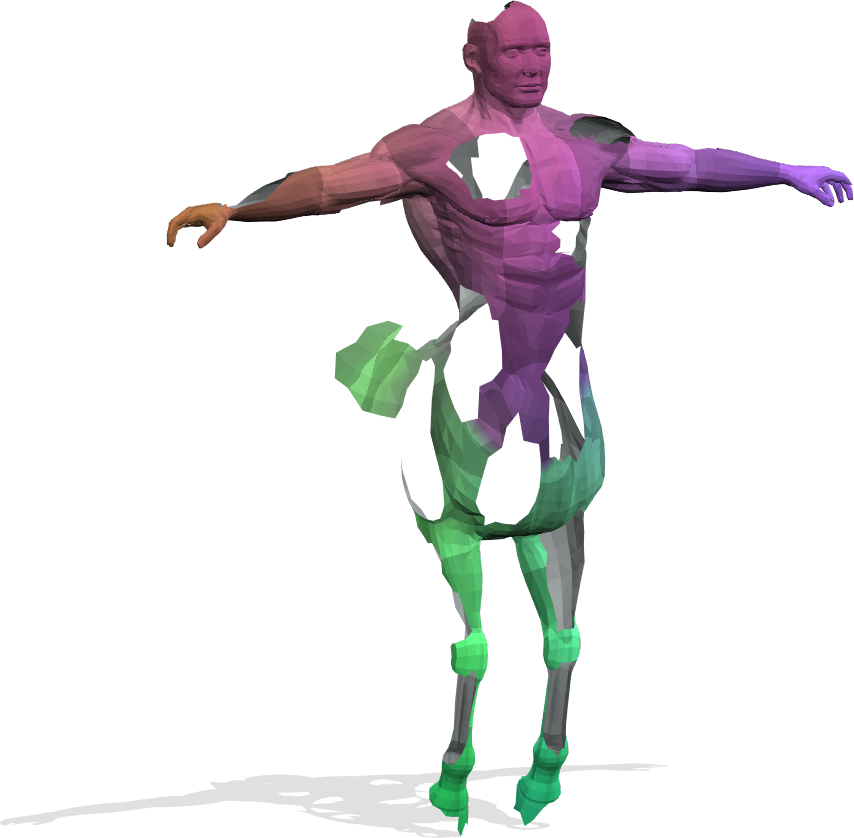}&
        \hspace{\hspaceCols}
        \includegraphics[height=\heightT, width=\widthT]{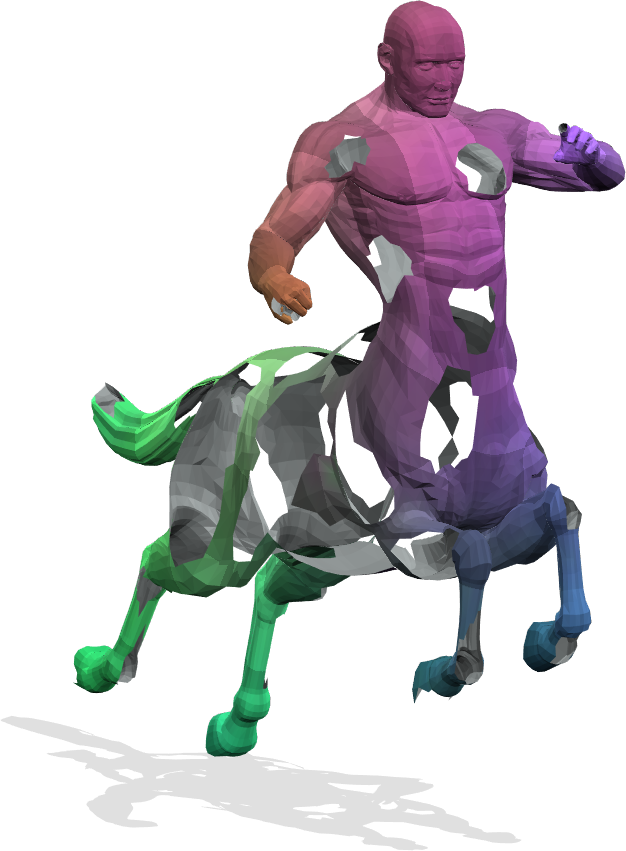}&
        \hspace{\hspaceCols}
        \includegraphics[height=\heightT, width=\widthT]{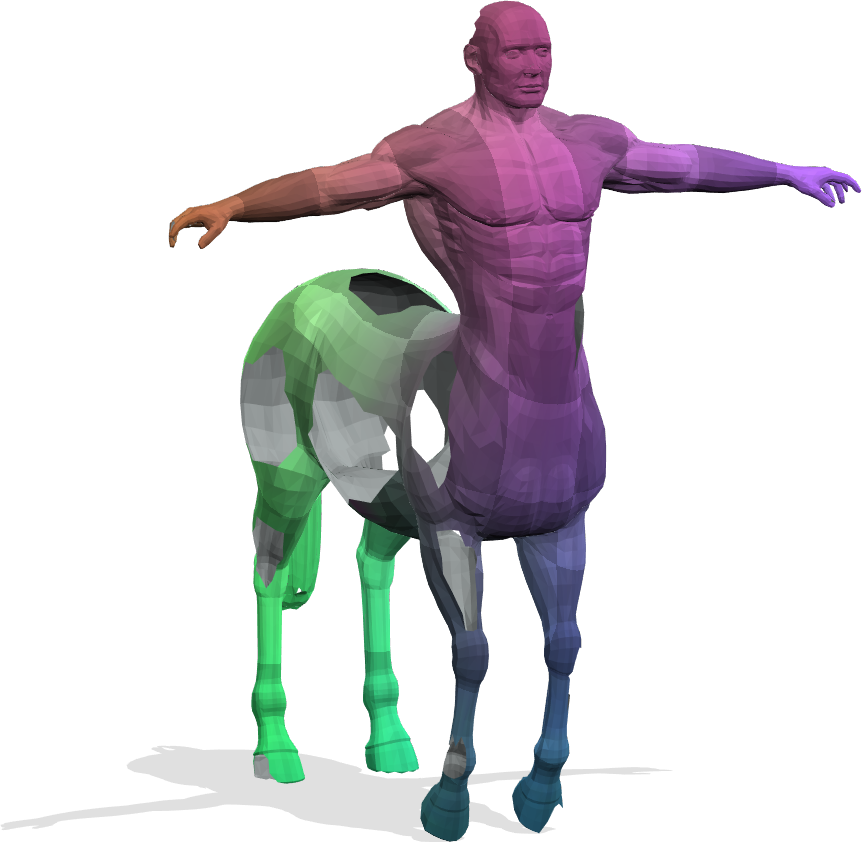}&
        \hspace{\hspaceCols}
        \includegraphics[height=\heightT, width=\widthT]{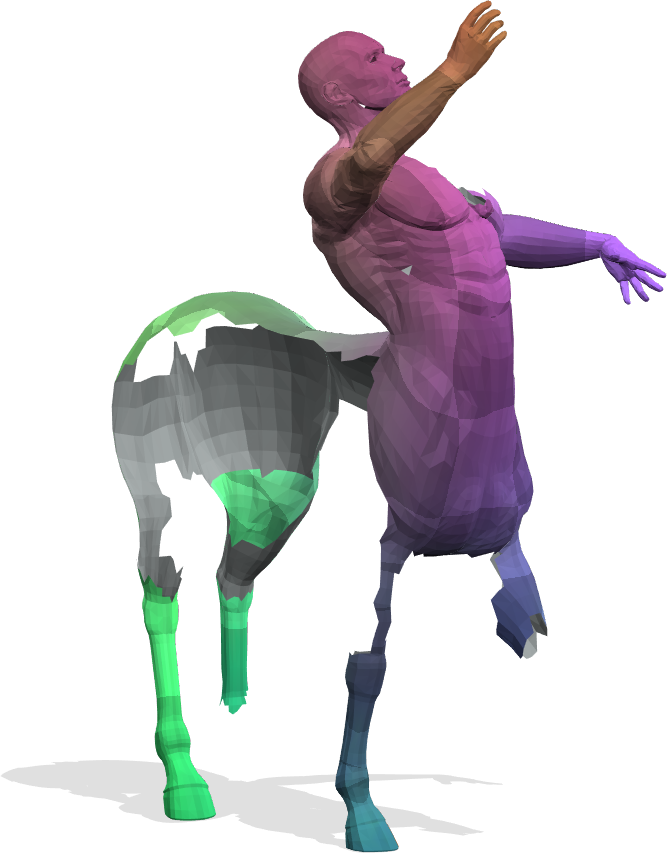}
        \\
        &
        \hspace{\hspaceCols}
        \includegraphics[height=\heightT, width=\widthT]{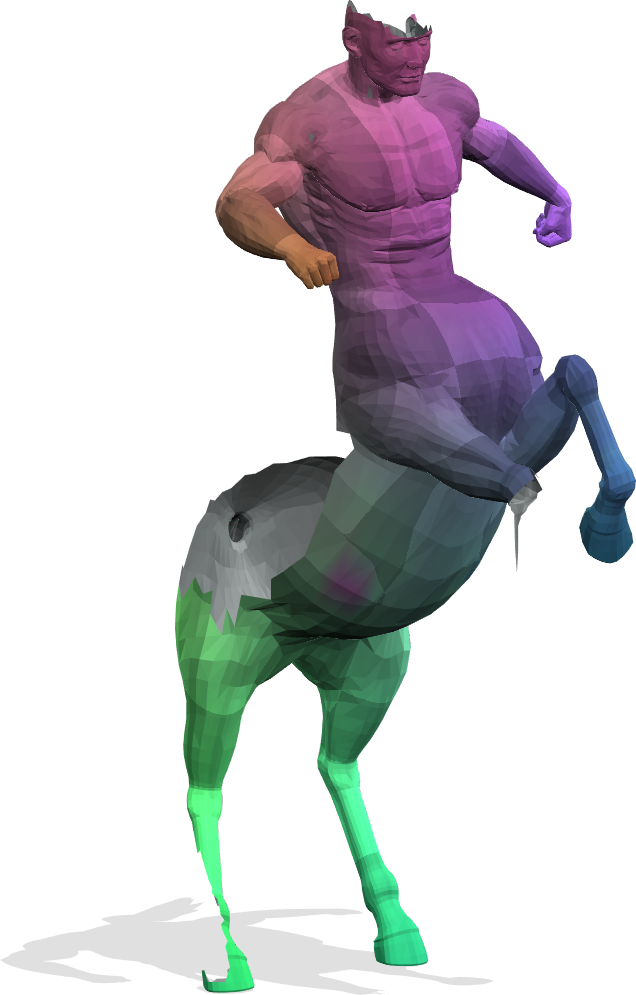}&
        \hspace{\hspaceCols}
        \includegraphics[height=\heightT, width=\widthT]{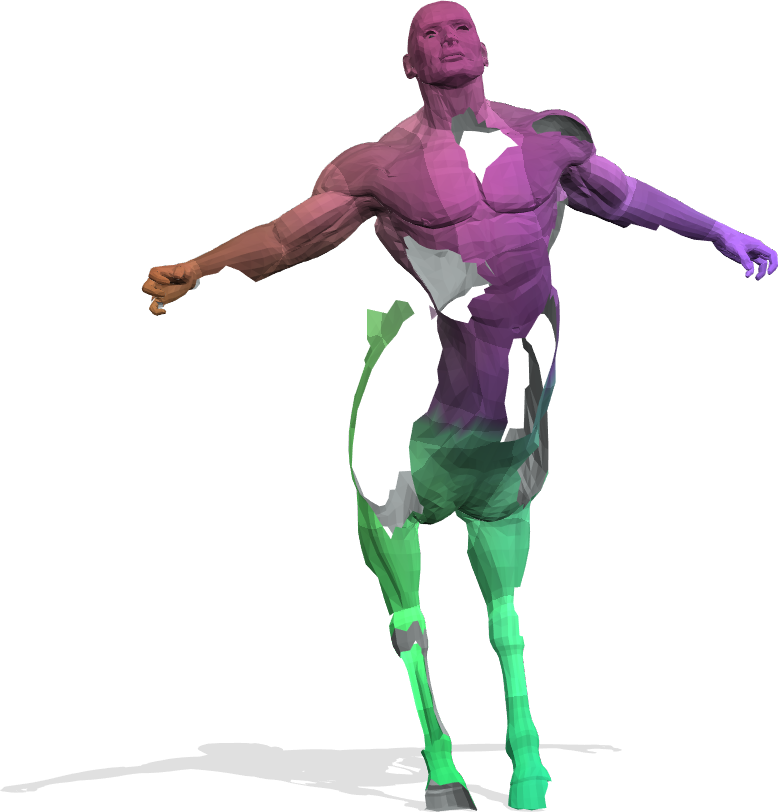}&
        \hspace{\hspaceCols}
        \includegraphics[height=\heightT, width=\widthT]{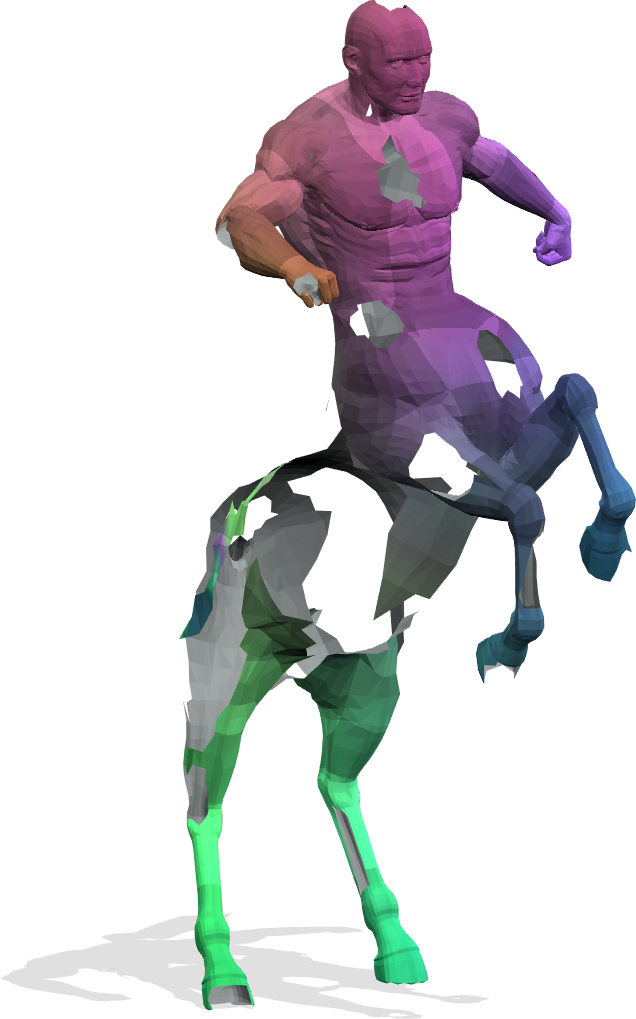}&
        \hspace{\hspaceCols}
        \includegraphics[height=\heightT, width=\widthT]{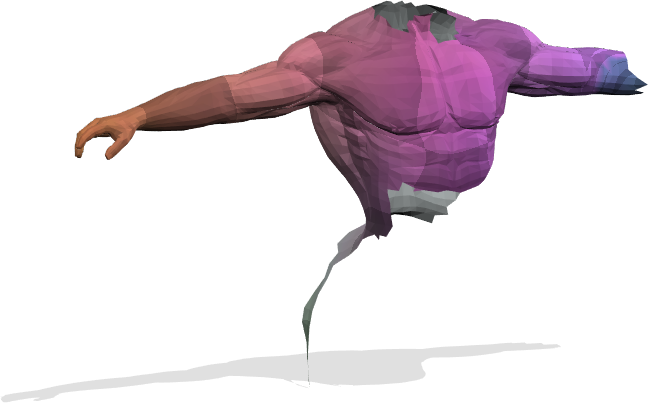} \\
        \hspace{\hspaceCols}
        \includegraphics[height=\heightT, width=\widthT]{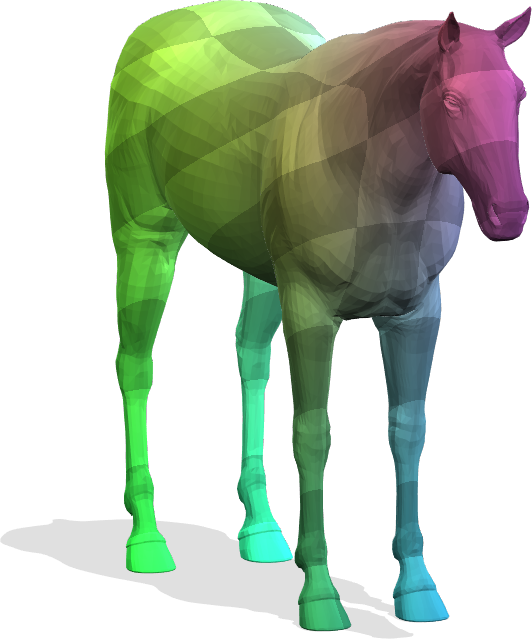}&
        \hspace{\hspaceCols}
        \includegraphics[height=\heightT, width=\widthT]{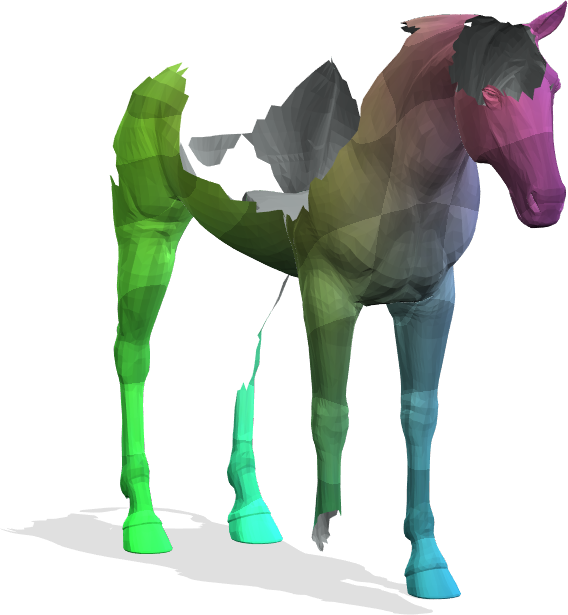}&
        \hspace{\hspaceCols}
        \includegraphics[height=\heightT, width=\widthT]{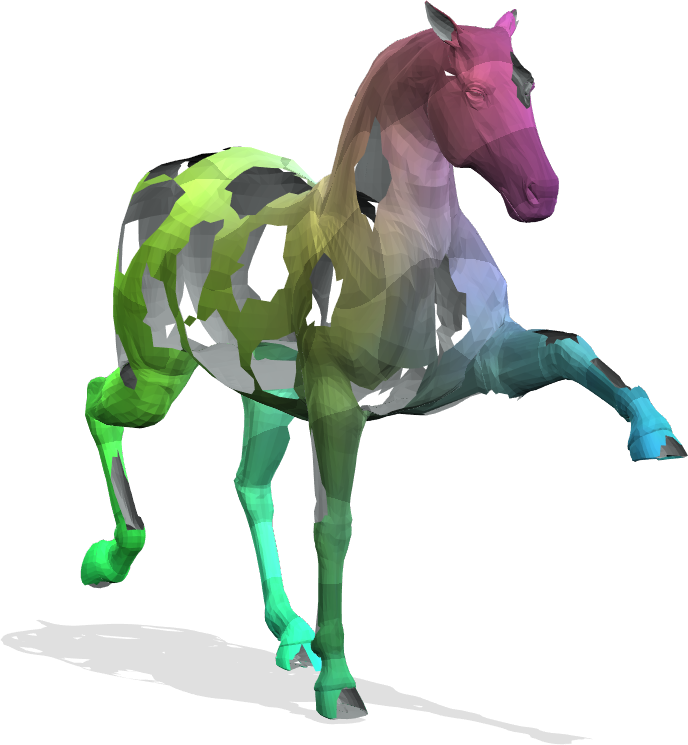}&
        \hspace{\hspaceCols}
        \includegraphics[height=\heightT, width=\widthT]{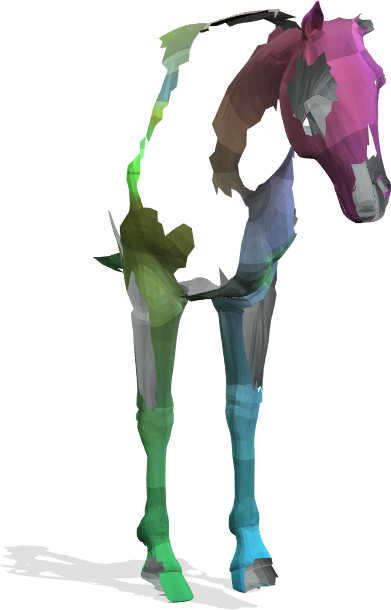}&
        \hspace{\hspaceCols}
        \includegraphics[height=\heightT, width=\widthT]{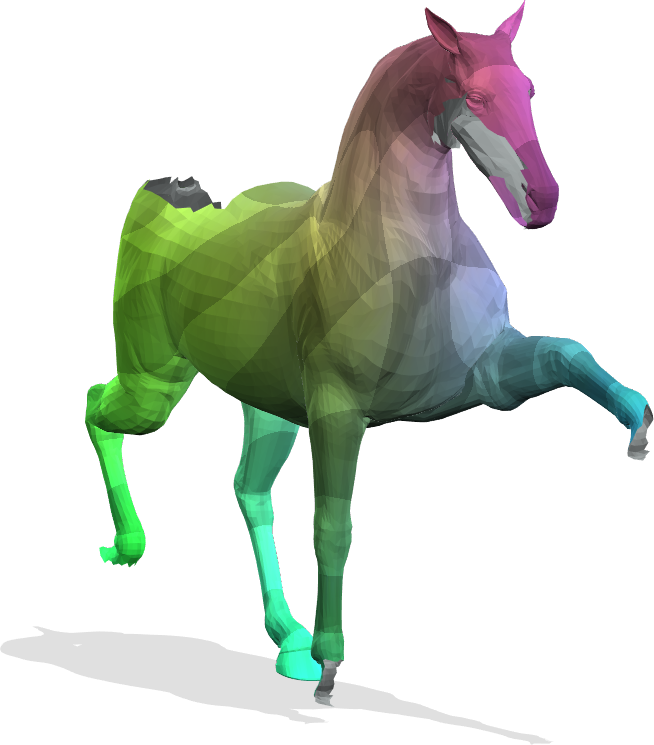} \\
        &
        \hspace{\hspaceCols}
        \includegraphics[height=\heightT, width=\widthT]{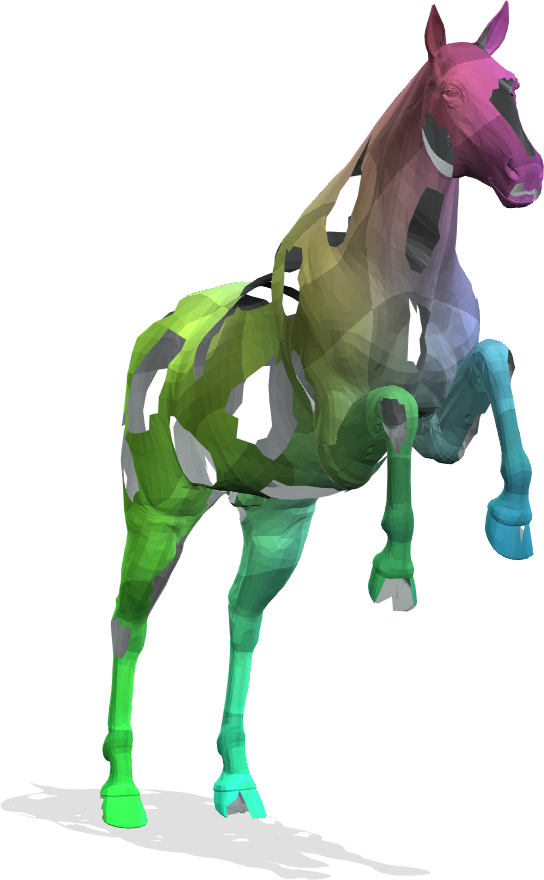}&
        \hspace{\hspaceCols}
        \includegraphics[height=\heightT, width=\widthT]{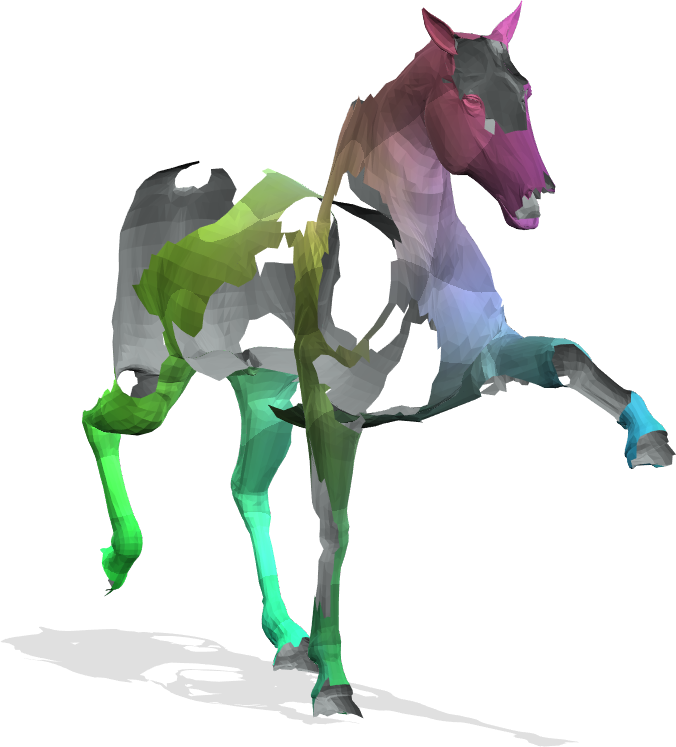}&
        \hspace{\hspaceCols}
        \includegraphics[height=\heightT, width=\widthT]{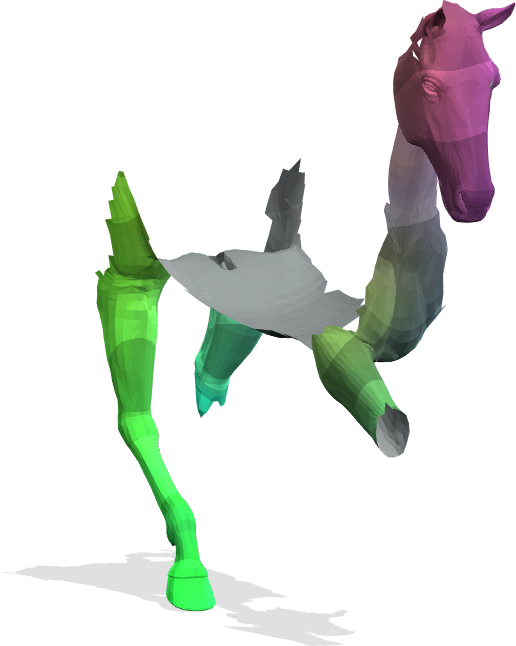}&
        \hspace{\hspaceCols}
        \includegraphics[height=\heightT, width=\widthT]{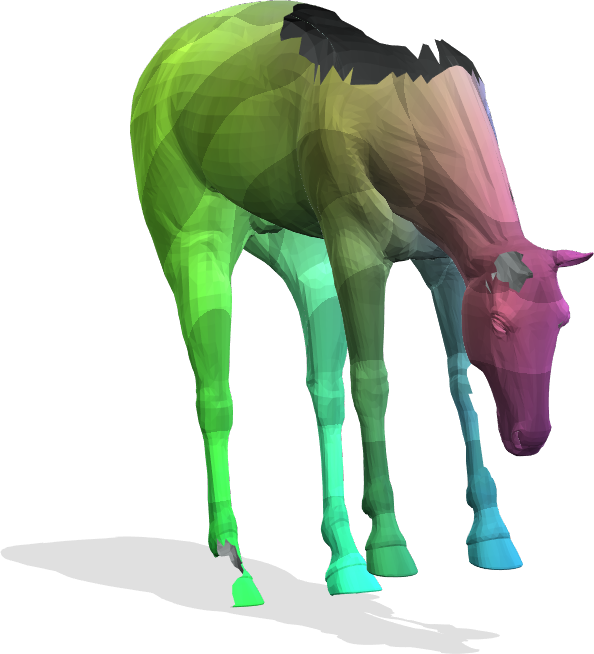}
        \\
        &
        \hspace{\hspaceCols}
        \includegraphics[height=\heightT, width=\widthT]{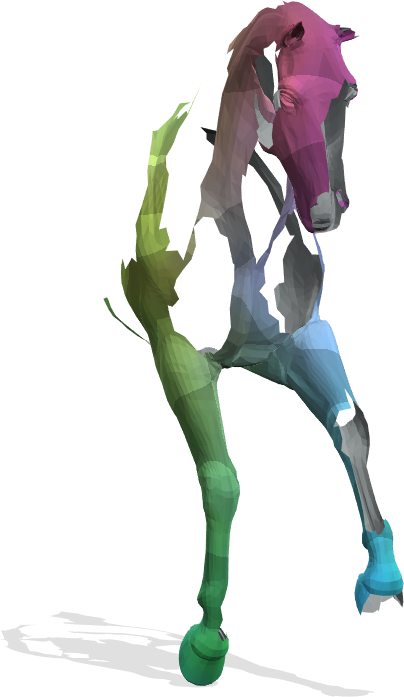}&
        \hspace{\hspaceCols}
        \includegraphics[height=\heightT, width=\widthT]{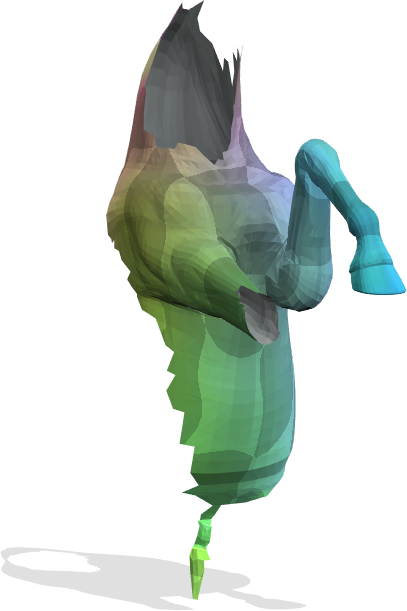}&
        \hspace{\hspaceCols}
        \includegraphics[height=\heightT, width=\widthT]{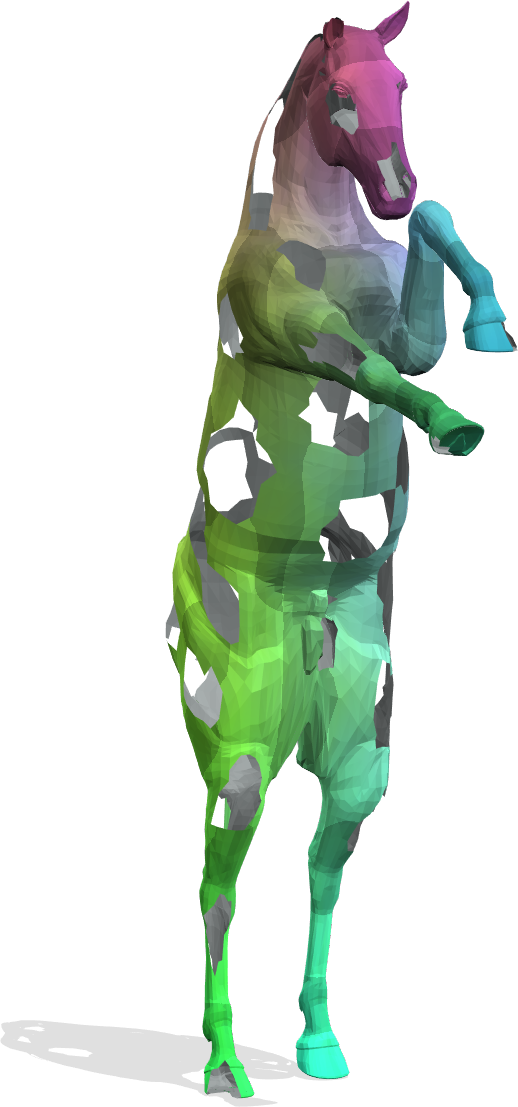}&
        \hspace{\hspaceCols}
        \includegraphics[height=\heightT, width=\widthT]{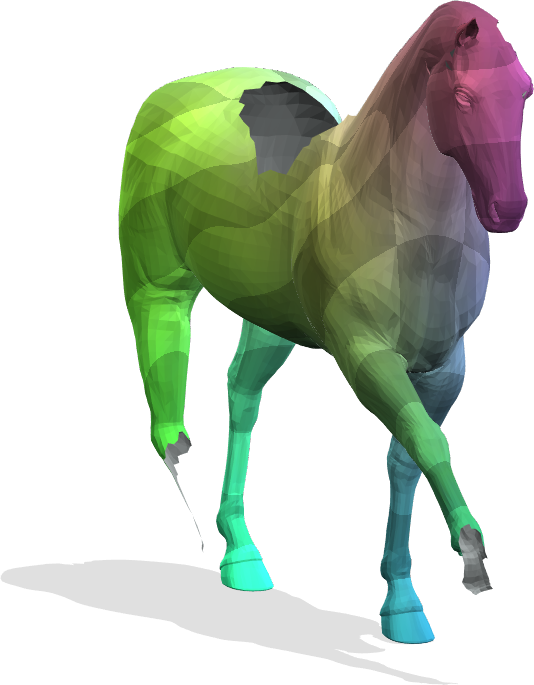} \\

        \hspace{\hspaceCols}
        \includegraphics[height=\heightT, width=\widthT]{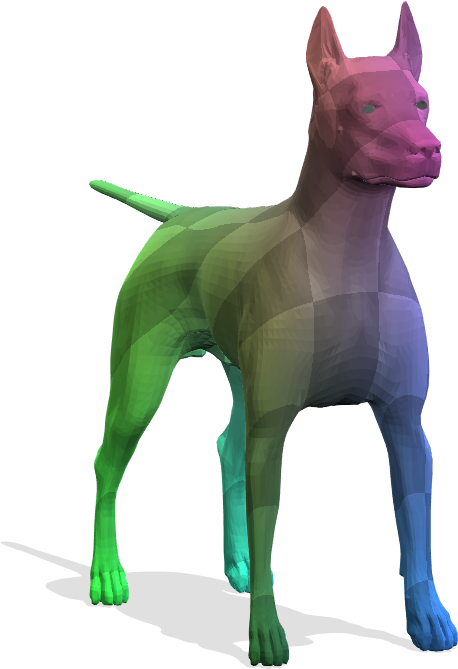}&
        \hspace{\hspaceCols}
        \includegraphics[height=\heightT, width=\widthT]{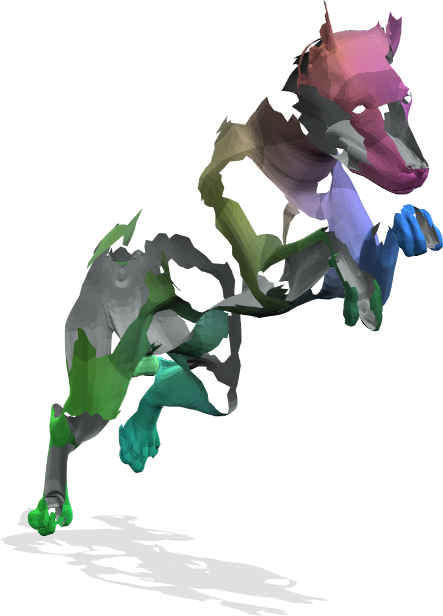}&
        \hspace{\hspaceCols}
        \includegraphics[height=\heightT, width=\widthT]{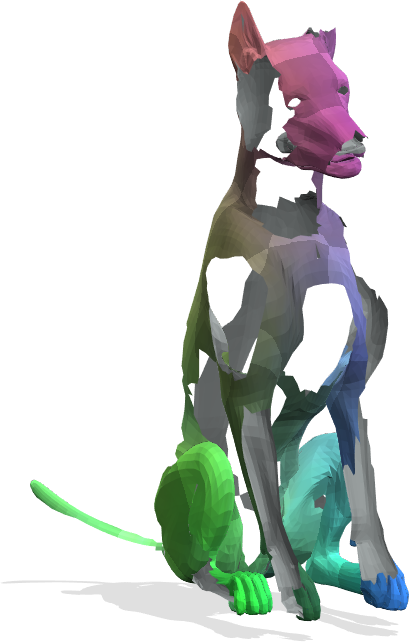}&
        \hspace{\hspaceCols}
        \includegraphics[height=\heightT, width=\widthT]{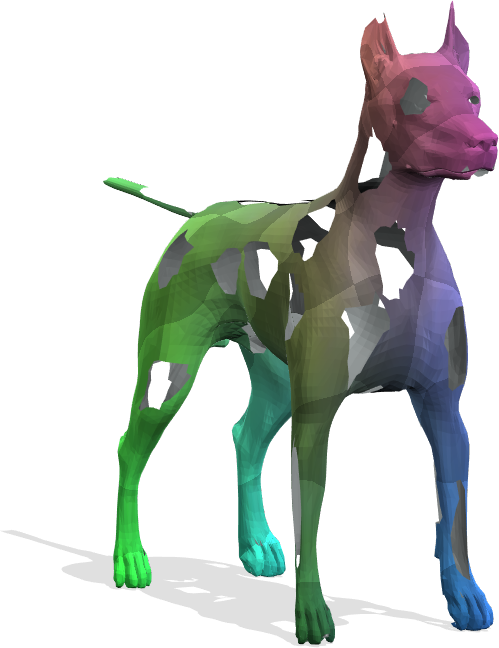}&
        \hspace{\hspaceCols}
        \includegraphics[height=\heightT, width=\widthT]{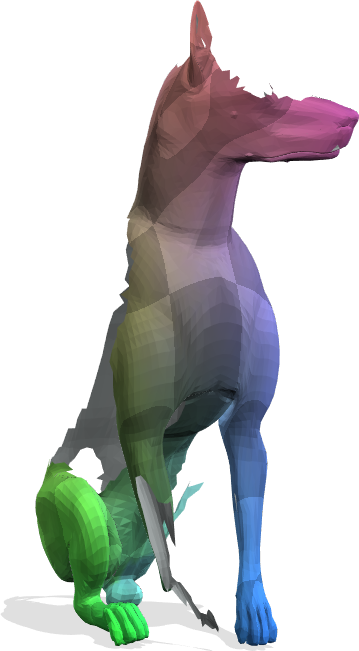} \\
        &
        \hspace{\hspaceCols}
        \includegraphics[height=\heightT, width=\widthT]{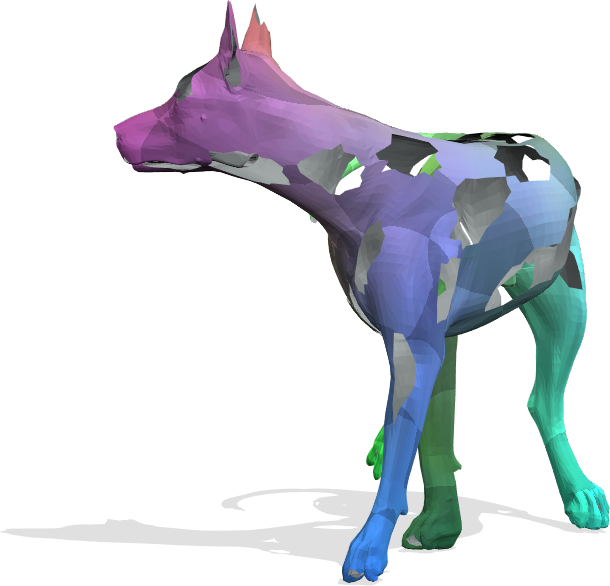}&
        \hspace{\hspaceCols}
        \includegraphics[height=\heightT, width=\widthT]{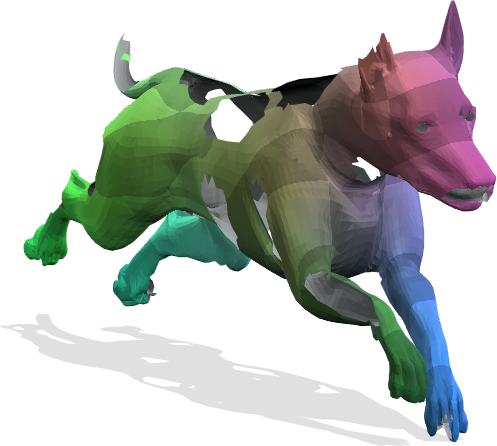}&
        \hspace{\hspaceCols}
        \includegraphics[height=\heightT, width=\widthT]{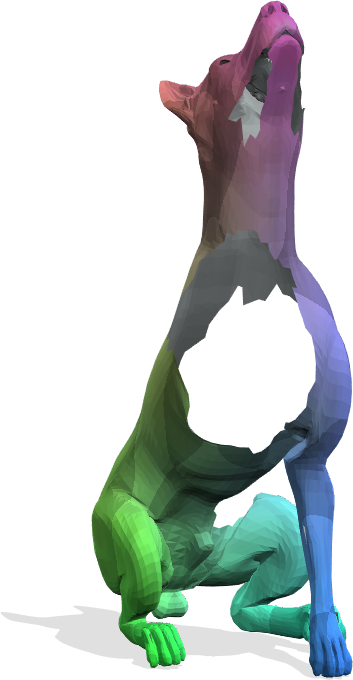}&
        \hspace{\hspaceCols}
        \includegraphics[height=\heightT, width=\widthT]{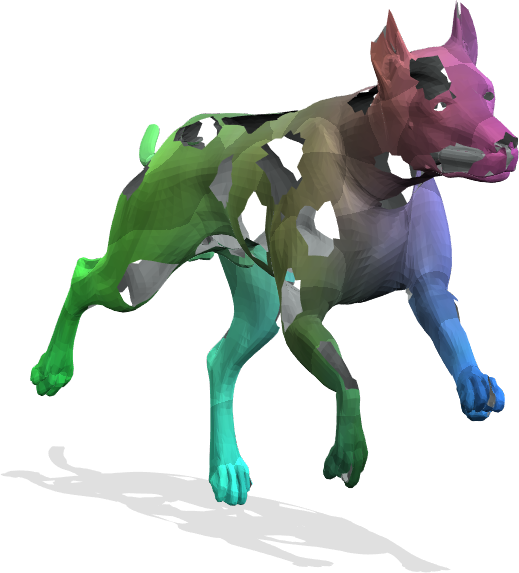}
        \\
        &
        \hspace{\hspaceCols}
        \includegraphics[height=\heightT, width=\widthT]{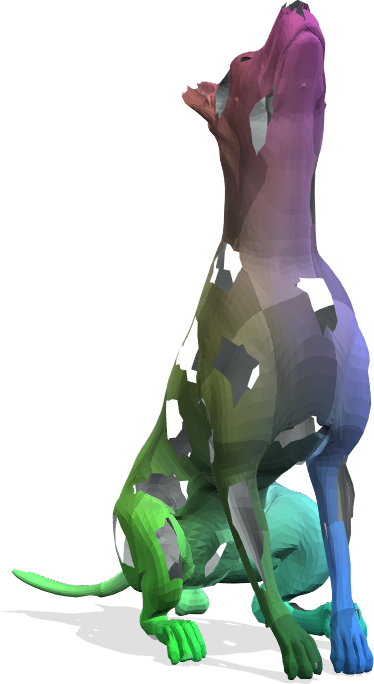}&
        \hspace{\hspaceCols}
        \includegraphics[height=\heightT, width=\widthT]{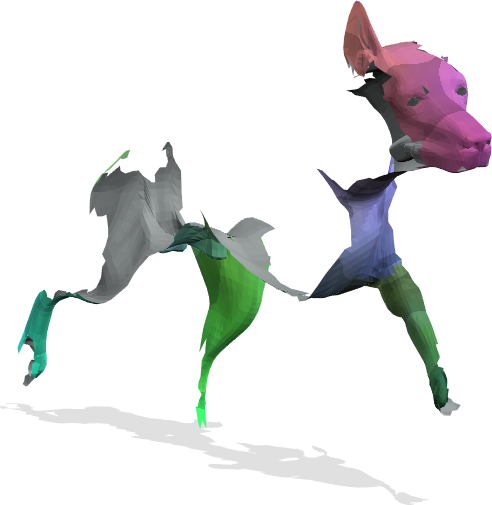}&
        \hspace{\hspaceCols}
        \includegraphics[height=\heightT, width=\widthT]{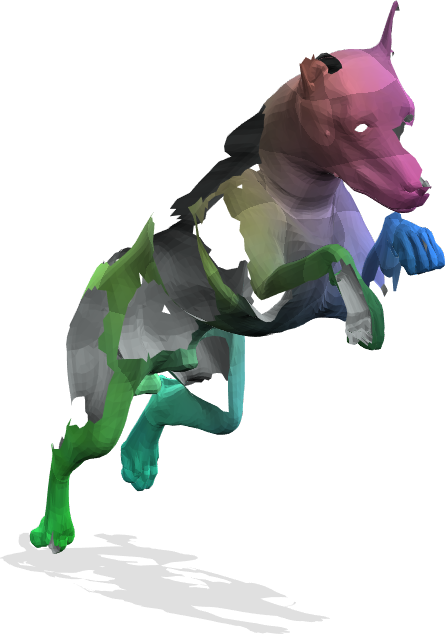}&
        \hspace{\hspaceCols}
        \includegraphics[height=\heightT, width=\widthT]{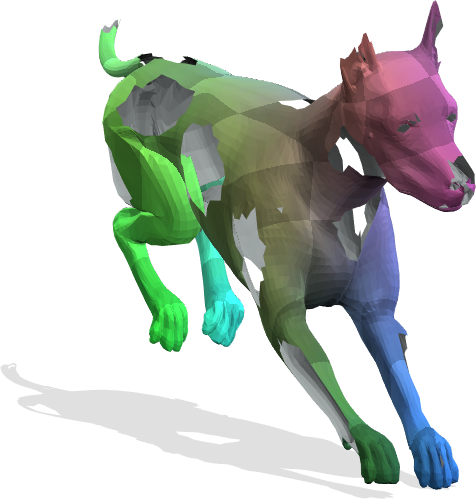} \\

        \hspace{\hspaceCols}
        \includegraphics[height=\heightT, width=\widthT]{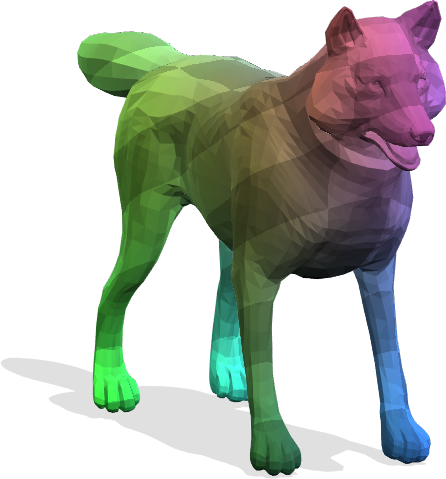}&
        \hspace{\hspaceCols}
        \includegraphics[height=\heightT, width=\widthT]{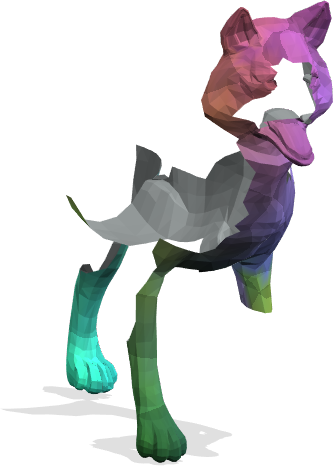}&
        \hspace{\hspaceCols}
        \includegraphics[height=\heightT, width=\widthT]{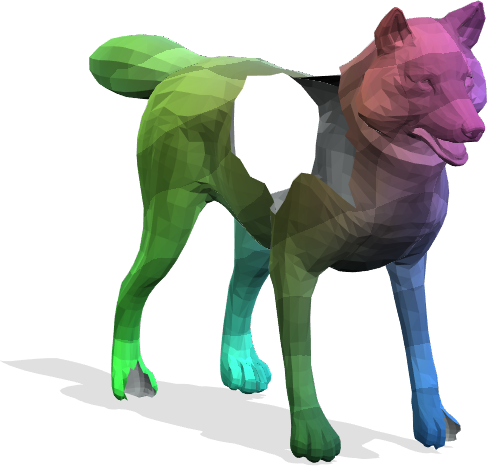}&
        \hspace{\hspaceCols}
        \includegraphics[height=\heightT, width=\widthT]{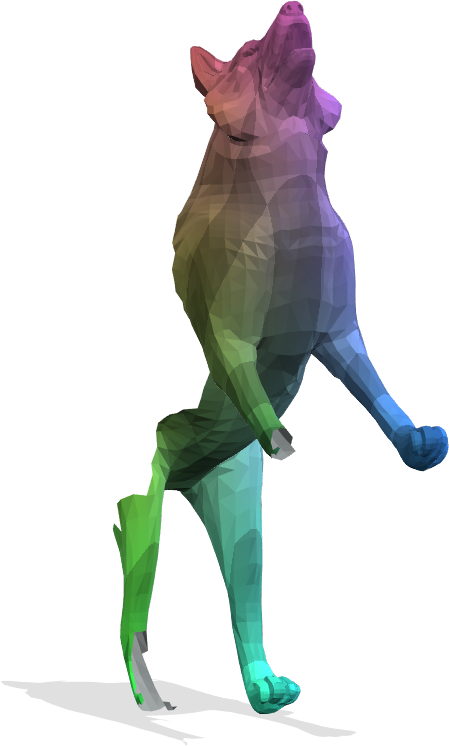}&
        \hspace{\hspaceCols}
        \includegraphics[height=\heightT, width=\widthT]{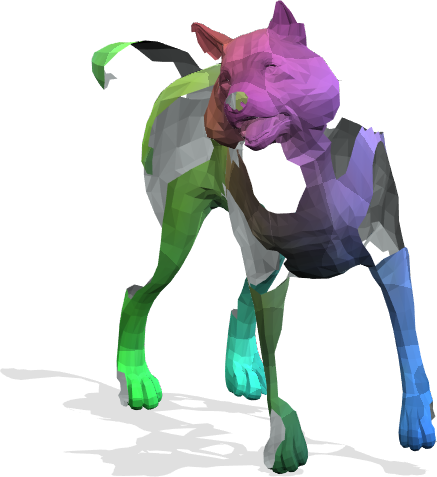} \\
        
    \end{tabular}
    \caption{\textbf{Qualitative results of our method on the SHREC'16 HOLES dataset.} For each shape category, the top-left shape is the reference shape to be matched by other shapes. Our method obtains accurate correspondences for partial shapes with multiple missing parts.}
    \label{fig:holes_qualitative}
\end{figure}

\end{document}